\documentclass[11pt]{article} 
\usepackage[dvipsnames]{xcolor}
\usepackage{url}
\usepackage{smile}
\usepackage{graphicx,wrapfig} 
\usepackage{algorithm}
\usepackage{algorithmic}
\usepackage{todonotes}
\usepackage{epstopdf}
\usepackage{wrapfig}
\usepackage[colorlinks, linkcolor=blue, anchorcolor=blue, citecolor=blue,hypertexnames=false]{hyperref}
\usepackage[margin=1in]{geometry}
\usepackage[normalem]{ulem}
\usepackage[export]{adjustbox}
\usepackage{mathtools, cuted}
\usepackage[numbers]{natbib}
\usepackage{enumitem}
\usepackage{makecell}

\usepackage{pifont}
\allowdisplaybreaks
\usepackage{lineno}

\usepackage{kpfonts}
\DeclareMathAlphabet{\mathsf}{OT1}{cmss}{m}{n}

\SetMathAlphabet{\mathsf}{bold}{OT1}{cmss}{bx}{n}

\newcommand{\spspan}{\mathrm{span}}

\newcommand{\ReLU}{\mathrm{ReLU}}

\newcommand{\Pdim}{\mathrm{Pdim}}

\newtheorem*{theorem*}{Theorem}
\newtheorem{setting}{Setting}

\title{\huge \bf Deep Nonparametric Estimation of Operators between Infinite Dimensional Spaces \thanks{on going work.}}

\author{Hao Liu, Haizhao Yang, Minshuo Chen, Tuo Zhao, and Wenjing Liao \thanks{Hao Liu is affiliated with the Math department of Hong Kong Baptist University; Haizhao Yang is affiliated with the Math department of Purdue University; Minshuo Chen and Tuo Zhao are affiliated with the ISYE department at Georgia Tech; Wenjing Liao is affiliated with the Math department at Georgia Tech; Haizhao Yang and Wenjing Liao are co-corresponding authors. Email: \text{haoliu@hkbu.edu.hk, haizhao@purdue.edu, $\{$mchen393, tzhao80, wliao60$\}$@gatech.edu}.}}

\newcommand{\commentout}[1]{}

\begin{document}

\maketitle

\begin{abstract}
Learning operators between infinitely dimensional spaces is an important learning task arising in wide applications in machine learning, imaging science, mathematical modeling and simulations, etc. This paper studies the nonparametric estimation of Lipschitz operators using deep neural networks. Non-asymptotic upper bounds are derived for the generalization error of the empirical risk minimizer over a properly chosen network class.   Under the assumption that the target operator exhibits a low dimensional structure, our error bounds decay as the training sample size increases, with an attractive fast rate depending on the intrinsic dimension in our estimation. Our assumptions cover most scenarios in real applications and our results give rise to fast rates by exploiting low dimensional structures of data in operator estimation. We also investigate the influence of network structures (e.g., network width, depth, and sparsity) on the generalization error of the neural network estimator and propose a general suggestion on the choice of network structures to maximize the learning efficiency quantitatively.
\end{abstract}

\section{Introduction}

Learning nonlinear operators from a Hilbert space to another via nonparametric estimation has been an important topic with broad applications. For example, in reduced-order modeling, a data-driven approach desires to map a full model trajectory to a reduced model trajectory or vice versa \citep{PEHERSTORFER2016196}. In solving parametric partial differential equations (PDEs), it is desired to learn a map from the parametric function space to the PDE solution space \citep{khoo2021solving,li2020fourier,deeponet}. In forward and inverse scattering problems \citep{switchnet,wei2019physics}, it is interesting to learn an operator mapping the observed data function space to the parametric function space that models the underlying PDE. In density functional theory, it is desired to learn a nonlinear operator mapping a potential function to a density function \citep{MNN}. In phase retrieval \citep{phase}, an operator from the observed data function space to the reconstructed image function space is learned. Other image processing problems, e.g., image super-resolution \citep{resolution}, image denoising \citep{Tian_2020}, image inpainting \citep{QIN2021102028}, are similar to the deep learning-based phase retrieval, where  an operator from a function space to another function space is learned.

As a powerful tool of nonparametric estimation, deep learning \citep{goodfellow2016deep} has made astonishing breakthroughs in various applications, including computer vision \citep{krizhevsky2012imagenet}, natural language processing \citep{graves2013speech}, speech recognition \citep{hinton2012deep}, healthcare \citep{miotto2017deep}, as well as nonlinear operator learning \citep{khoo2021solving,ZHU2018415,MNN,FAN20191,switchnet,chen1995universal,deeponet,lanthaler2021error,bhattacharya2020model,li2020fourier,nelsen2020random}. A typical method for operator learning is to first discretize the function spaces and represent each function by a vector using sampling. Then deep neural networks are applied to learn the map between these vector spaces \citep{khoo2021solving,ZHU2018415,MNN,FAN20191,switchnet}. Such methods are mesh dependent: if a different discretization scheme is used, the network needs to be trained again.
Though empirical successes have been demonstrated in learning nonlinear operators by this approach in many applications, it is computationally expensive to train these algorithms and the training procedure has to be repeated when the dimension of vector spaces is changed. Another approach based on the theory of approximating operators by neural networks \citep{chen1995universal} can alleviate this issue to a certain extent by avoiding the discretization of the output Hilbert space of the operator. This approach was first proposed in \citep{chen1995universal} with two-layer neural networks and recently revisited with deeper neural networks in \citep{deeponet} with successful applications \citep{doi:10.1063/5.0041203,CAI2021110296}. However, the methods in \citep{chen1995universal,deeponet,doi:10.1063/5.0041203,CAI2021110296} are still mesh-dependent due to the requirement of a fixed number of sample points for the input function of the operator. More recently, a discretization-invariant (mesh-independent) operator learning method 
was proposed for problems with a sparsity structure in \citep{li2020neural,bhattacharya2020model,li2020fourier,nelsen2020random} by taking the advantage of graph kernel networks, principal component analysis (PCA), 
and kernel integral operators, etc. With discretization-invariant approaches, the training procedure does not need to be performed again when the discretization scheme changes.

Although operator learning via deep learning-based nonparametric estimation 
has been successful in many applications, its statistical learning theory is still in its infancy, especially when the operator is from an infinite dimensional space to another. The successes of deep neural networks are largely due to their universal approximation power \citep{cybenko1989approximation,hornik1991approximation}, 
showing the existence of a neural network with a proper size fulfilling the approximation task for certain function classes. Quantitative function approximation theories, provably better than traditional tools, have been extensively studied with various network architectures and activation functions, e.g., for continuous functions \citep{yarotsky2017error,shijun2,shijun4,shijun5,shijun6,yarotsky:2021:02}, for functions with certain smoothness \citep{yarotsky18a,yarotsky:2019:06,shijun3,suzuki2018adaptivity}, and for functions with integral representations \citep{barron1993,Weinan2019,flow,siegel2021sharp}. In theory, deep neural networks can approximation certain high dimensional functions with a fast rate that is independent of the input dimension \citep{barron1993,Weinan2019,flow,siegel2021sharp,shijun4,shijun5,yarotsky:2019:06,shijun7,chen1995universal,chen2019nonparametric,chen2020doubly,hao2021icml,Jiao2021DeepNR,cloninger2020relu,shaham2018provable,schmidt2019deep,du2021discovery,nakada2020adaptive}. However, in the context of operator approximation, deep learning theory is very limited. Probably the first result is the universal approximation theorem for operators in \citep{chen1995universal}. More recently, quantitative approximation results for operators between infinite dimensional spaces were given in \citep{bhattacharya2020model,lanthaler2021error,kovachki2021universal} based on the function approximation theory in \citep{yarotsky2017error}.  Note that the function approximation results in \citep{yarotsky2017error} does not give the flexibility to choose arbitrary width and depth of neural networks. In this paper, we provide a new operator approximation theory based on nearly optimal function approximation results where the width and depth of the network can be chosen flexibly. In comparison with \citep{lanthaler2021error}, the flexibility of choosing arbitrary width and depth provides an explicit guideline to balance the approximation error and the statistical variance to achieve a better generalization error in operator learning.

We also establish a novel statistical theory for deep nonparametric estimation of Lipschitz operators between infinite dimensional Hilbert spaces. The core question to be answered is: how the generalization error scales when the number of training samples increases and whether the scaling is dimension-independent without the curse of dimensionality. In literature, the statistical theory for function regression via neural networks has been a popular research topic \citep{hamers2006nonasymptotic,kohler2005adaptive,Arthur18,10.1214/18-AOS1747,schmidt2020nonparametric,Yuan1,chen2019nonparametric,kohler2020estimation,JMLR:v21:20-002,Farrell_2021,hao2021icml,Jiao2021DeepNR}. These works have proved that deep nonparametric regression can achieve the optimal minimax rate of regression established in \citep{10.1214/aos/1176345969,gyorfi2002distribution}. When the target function has low complexity or the function domain is a low dimensional set, deep neural networks can achieve a fast rate depending on the intrinsic dimension \citep{chen2019efficient,chen2019nonparametric,chen2020doubly,hao2021icml,shijun2,Jiao2021DeepNR,cloninger2020relu,shaham2018provable,schmidt2019deep,du2021discovery,nakada2020adaptive}. In more sophisticated cases when a mathematical modeling problem is transferred to a special regression problem, e.g., solving high dimensional PDEs and identifying the governing equation of spatial-temporal data, the generalization analysis of deep learning has been proposed in \citep{DBLP:journals/corr/abs-1809-03062,shin2020convergence,Luo2020,mishra2020estimates,lu2021priori,lu2021priori2,duan2021convergence,gu2021stationary}. All these results focus on the regression problem when the target function is a mapping from a finite dimensional space to a finite dimensional space. Therefore, these results cannot be applied to mappings from an infinite dimensional space to another. To our best knowledge, the only work on the generalization error analysis of deep operator learning in Hilbert spaces is \citep{lanthaler2021error} for the algorithm in \citep{deeponet}, which is not completely discretization-invariant.  The generalization error in \citep{lanthaler2021error} is a posteriori depending on the properties of neural networks fitting the target operator. Recently, the posterior rates on learning linear operators by Bayesian inversion have been studied in \citep{de2021convergence}.

In this paper, we establish a priori generalization error for a discretization-invariant operator learning algorithm for operators between Hilbert spaces. As we shall see later, our theory can be applied to operator learning from a finite dimensional vector space to another as a special case. Therefore, the theoretical result in this paper can facilitate the understanding of many operator learning algorithms by neural networks in the literature. 
Our contributions are summarized as follows:
\begin{enumerate}
	\item We derive an upper bound on the generalization error for a general framework of learning operators between infinite dimensional spaces by deep neural networks. The framework considered here first encodes the input and output space into finite-dimensional spaces by some encoders and decoders. Then a transformation between the dimension reduced spaces is learned using deep neural networks. Our upper bound is derived for two network architectures: one has constraints on the number of nonzero weight parameters and parameter magnitude; The other network architecture does not have such constraints and allows one to flexibly choose the depth and width. Our upper bound consists of two parts: the error of learning the transformation by deep neural networks, and the dimension reduction error with encoders and decoders. 
	\item Our analysis is general and can be applied for a wide range of popular choices of encoders and decoders in the numerical implementation, such as those derived from Legendre polynomials, trigonometric bases, and principal component analysis. The generalization error is given for each of these examples.
	\item We discuss two scenarios to further exploit the additional low-dimensional structures of data in operator estimation motivated by practical considerations and classical numerical methods. The first scenario is when encoded vectors in the input space are on a low-dimensional manifold. 
	In this scenario, we show that the generalization error converges as the training sample increases with a fast rate depending on the intrinsic dimension of the manifold. The second scenario is when the operator itself has low complexity: the composition of the operator with a certain encoder and decoder is a multi-index model. In this scenario, we show that the convergence rate of the generalization error depends on the intrinsic dimension of the composed operator.
\end{enumerate}

We organize this paper as follows. In Section \ref{sec.framework}, we introduce our notations and the learning framework considered in this paper. Our main results with general encoders and decoders are presented in Section \ref{sec.mainresults}. We discuss the applications of our main results to specific encoders and decoders derived from certain function basis and PCA in Section \ref{sec.deter} and \ref{sec.PCA}, respectively. To 
further exploit additional low-dimensional structures of data, we discuss the application of our results to two scenarios 
in Section \ref{sec.lowD}. The proofs of all results are given in Section \ref{sec.proof}. We conclude this paper in Section \ref{sec.conclusion}.

\section{A general framework}\label{sec.framework}
\subsection{Preliminaries}
We first briefly introduce some definitions and notations on a Hilbert space, encoders, decoders, and feedforward neural networks used in this paper. 
A Hilbert space is a Banach space equipped with an inner product. It is separable if it admits a countable orthonormal basis. Let $\cH$ be a separable Hilbert space. An encoder for $\cH$ is an operator $E_{\cH}: \cH\rightarrow \RR^d$, where $d$ is a positive integer representing the encoding dimension. The associated decoder is an operator  $D_{\cH}:\RR^d \rightarrow \cH$. The composition $\Pi_{\cH}=D_{\cH}\circ E_{\cH}: \cH\rightarrow \cH$ is a projection. For any $u\in \cH$, we define the projection error as $
\|\Pi_{\cH}(u)-u\|_{\cH}.$


In this paper, we consider  the ReLU Feedforward Neural Network (FNN) in the form of
\begin{align}
	f(\xb)=W_L\cdot\ReLU\left( W_{L-1}\cdots \ReLU(W_1\xb+\bbb_1)+ \cdots +\bbb_{L-1}\right)+\bbb_L,
	\label{eq.FNN.f}
\end{align}
where $W_l$'s are weight matrices, $\bbb_l$'s are biases, and $\ReLU(a)=\max\{a,0\}$ is the rectified linear unit activation (ReLU) applied element-wise.

We consider two classes of network architectures whose inputs are in a compact domain of a vector space and whose outputs are vectors in $\RR^d$. The dimension of the input and output spaces are to be specified later. The first class is defined as
\begin{align}
	\cF_{\rm NN}(d,L,p,K,\kappa,M)=\{\Gamma=&[f_1, f_2,...,f_{d}]^{\top}: \mbox{ for each }k=1,...,d,\nonumber\\
	&f_k(\xb) \mbox{ is in the form of (\ref{eq.FNN.f}) with } L \mbox{ layers, width bounded by } p, \nonumber\\
	& \|f_k\|_{\infty}\leq M, \ \|W_l\|_{\infty,\infty}\leq \kappa, \ \|\bbb_l\|_{\infty}\leq \kappa,\  \sum_{l=1}^L \|W_l\|_0+\|\bbb_l\|_0\leq K   \},
	\label{eq.FNN}
\end{align}
where
$
\|f\|_{\infty}=\sup_{\xb} |f(\xb)|,\ \|W\|_{\infty,\infty}=\max_{i,j} |W_{i,j}|,\ \|\bbb\|_{\infty}=\max_i |b_i|
$
for any function $f$, matrix $W$, and vector $\bbb$ with $\|\cdot\|_0$ denoting the number of nonzero elements of its argument. 
%
The function class given by this first network architecture has an upper bound on all weight parameters (the magnitude of all weight parameters are upper bounded by $\kappa$) and a cardinality constraint (the total number of nonzero parameters are no more than $K$). Each element of the output is upper bounded by $M$. This constraint on the output is often enforced by clipping the output in the testing procedure. Such a clipping can be realized with a two-layer network, which is fixed during training. This clipping step is common in nonparametric regression \citep{gyorfi2002distribution}.

In the second class of network architecture, we drop the magnitude and cardinality constraints for practical concerns on training. The second network architecture is parameterized by $L,p,M$ only:
\begin{align}
	\cF_{\rm NN}(d,L,p,M)=\{\Gamma=&[f_1, f_2,...,f_{d}]^{\top}: \mbox{ for each }k=1,...,d_{\cY},\nonumber\\
	&f_k(\xb) \mbox{ is in the form of (\ref{eq.FNN.f}) with } L \mbox{ layers, width bounded by } p, \nonumber\\
	& \|f_k\|_{\infty}\leq M\}.
	\label{eq.FNN.dense}
\end{align}
All theoretical results in this paper can be applied to both network architectures.

\paragraph{Notations:} We use bold lowercase letters to denote vectors, and normal font letters to denote scalars. The notation $\mathbf{0}$ represents a zero vector. For a $d$ dimensional vector $\kb=[
	k_1, \cdots , k_d]^{\top}$, we denote $|\kb|=\sum_{i=1}^d k_i$. The vector norms are defined as $\|\kb\|_{\infty}=\max_i |k_i|$ and $\|\kb\|_2=\sqrt{\sum_{i=1}^d k_i^2}$. For any scalar $s$, we denote $\lceil s \rceil$ as the smallest integer that is no less than $s$. We use $\NN$ to denote the set of positive integers and $\NN_0=\NN \cup \{0\}$. For a function $f: \Omega \rightarrow \RR$ in a Hilbert space $\cH$, we define the function norms as $\|f\|_{\infty}=\sup_{\xb\in\Omega} |f(\xb)|$ and $\|f\|_{\cH}=\sqrt{\langle f, f\rangle_{\cH}}$, where $\langle\cdot,\cdot\rangle_{\cH}$ denotes the inner product of $\cH$. For an operator $A: \cH\rightarrow \cH$, we denote its operator norm by $\|A\|_{\rm op}$ and its Hilbert-Schmidt norm by $\|A\|_{\rm HS}$. More notations used in this paper is summarized in Table \ref{tab.notation}.
	

\subsection{Problem setup and a learning framework}

Let $\cX$ and $\cY$ be two separable Hilbert spaces and $\Psi:\cX\rightarrow\cY$ be an unknown operator. Our goal is to learn the operator $\Psi$ from a finite number of samples  $\cS=\{u_i,v_i\}_{i=1}^{2n}$ in the following setting.

\begin{setting}\label{setting}
	Let $\cX,\cY$ be two separable Hilbert spaces and $\gamma$ be a probability measure on $\cX$. Let $\cS=\{u_i,v_i\}_{i=1}^{2n}$ be the given data where $u_i$'s are i.i.d. samples from $\gamma$ and the $v_i$'s are generated according to model:
	\begin{align}
	v_i = \Psi(u_i)+\widetilde{\epsilon_i},
	\label{eq.v}
\end{align}
where the $\widetilde{\epsilon}_i$'s are i.i.d. samples from a probability measure $\mu$ on $\cY$, independently of $u_i$'s. We denote the probability measure of $v$ by $\zeta$.
\end{setting}
The pushforward measure of $\gamma$ under $\Psi$ is denoted by $\Psi_{\#}\gamma$, such that for any $\Omega \subset\cY$,
$$
\Psi_{\#}\gamma(\Omega)=\gamma\left(\{u: \Psi(u)\in\Omega\}\right).
$$


Without additional assumptions, 
the estimation error of $\Psi$ based on a finite number of samples may not converge to zero since $\Psi$ is an operator between infinite-dimensional spaces. 
In this paper, we exploit the low-dimensional structures in this estimation problem arising from practical applications, and prove a nonparametric estimation error of $\Psi$ by deep neural networks. 

Our learning framework follows the idea of model reduction \citep{bhattacharya2020model}. It consists of encoding and decoding in both the $\cX$ and $\cY$ spaces, and deep learning of a transformation between the encoded vectors for the elements in $\cX$ and $\cY$. 
We first encode the elements in $\cX$ and $\cY$ to finite dimensional vectors by an encoding operator. For fixed positive integers $d_\cX$ and $d_\cY$, let $E_{\cX}:\cX\rightarrow \RR^{d_\cX}$ and $ D_{\cX}: \RR^{d_{\cX}}\rightarrow \cX$ be the encoder and decoder of $\cX$, and $E_{\cY}:\cY\rightarrow \RR^{d_{\cY}}$ and $ D_{\cY}: \RR^{d_{\cY}}\rightarrow \cY$ be the encoder and decoder of $\cY$ such that
$$D_{\cX} \circ E_{\cX} \approx I \ \text{ and } \ D_{\cY}\circ E_{\cY} \approx I.$$
The empirical counterparts of encoders and decoders are denoted by $E_{\cX}^n, D_{\cX}^n, E_{\cY}^n$ and $D_{\cY}^n$, and we call them empirical encoders and decoders.


The simplest encoder in a function space is the discretization operator. When $\cX$ is a function space containing functions defined on a compact subset of $\RR^D$, we can discretize the domain with a fixed grid, and take the encoder as the sampling operator on this grid. However, the discretization operator may not reveal the low-dimensional structures in the functions of interest, and therefore may not effectively reduce the dimension.

A popular choice of encoders in applications is the basis encoder, such as the Fourier transform with trigonometric basis, or PCA with data-driven basis, etc.
Given an orthonormal basis of $\cX$ and a positive integer $d_{\cX}$, the basis encoder maps an element in $\cX$ to $d_{\cX}$ coefficients associated with a fixed set of $d_{\cX}$ bases. For any coefficient vector $\ab\in \RR^{d_{\cX}}$, the decoder $D_{\cX}(\ab)$ gives rise to a linear combination of these $d_{\cX}$  bases weighted by $\ab$. See Section \ref{sec.deter} for the details about the basis encoder. The trigonometric basis and orthogonal polynomials are commonly used bases in applications. These bases are a priori given, independently of the training data.  In this case, the basis encoder can be viewed as a deterministic operator, which is given independently of the training data. The empirical encoder and decoder are the same as the oracle encoder and decoder:
$
E_{\cX}^n=E_{\cX}$ and $D_{\cX}^n=D_{\cX}.
$

PCA \citep{pearson1901liii,hotelling1933analysis,hotelling1992relations} is an effective dimension reduction technique, when $u_i$'s exhibit a low-dimensional linear structure.
The PCA encoder encodes an element in $\cX$ to the $d_{\cX}$ coefficients associated with the top $d_{\cX}$ eigenbasis of a trace operator. The PCA decoder gives a linear combination of the eigenbasis weighted by the given coefficient vector.
In practice, one needs to estimate this trace operator from the training data and obtain an empirical estimation of $E_{\cX}$ and $D_{\cX}$, which are denoted by $E_{\cX}^n$ and $D_{\cX}^n$, respectively.
The PCA encoder is data-driven, and we expect $
E_{\cX}^n \approx E_{\cX}, \ D_{\cX}^n \approx D_{\cX}
$
when the sample size $n$ is sufficiently large.
The encoding and decoding operator in $\cY$ can be defined analogously.

The operator $D_{\cX}\circ E_{\cX}$ is  the projection operator associated with the encoder $E_{\cX}$ and decoder $D_{\cX}$.
We have the following projections and their empirical counterparts:
\begin{align*}
	&\Pi_{\cX,d_{\cX}}=D_{\cX}\circ E_{\cX},\quad \Pi_{\cX,d_{\cX}}^n=D_{\cX}^n\circ E_{\cX}^n,\\
	&\Pi_{\cY,d_{\cY}}=D_{\cY}\circ E_{\cY},\quad \Pi_{\cY,d_{\cY}}^n=D_{\cY}^n\circ E_{\cY}^n.
\end{align*}

After the empirical encoders $E_{\cX}^n,E_{\cY}^n$ and decoders $D_{\cX}^n,D_{\cY}^n$ are computed, our objective is to learn a transformation $\Gamma: \RR^{d_{\cX}}\rightarrow \RR^{d_{\cY}}$ such that
\begin{align}
	D_{\cY}^n \circ \Gamma \circ E_{\cX}^n \approx \Psi.
	\label{eq.Gamma}
\end{align}

We learn $\Gamma$ using a two-stage algorithm. Given the training data $\cS=\{u_i,v_i\}_{i=1}^{2n}$, 
 we split the data into two subsets $\cS_1=\{u_i,v_i\}_{i=1}^{n}$ and $\cS_2=\{u_i,v_i\}_{i=n+1}^{2n}$ \footnote{The data can be  split unevenly as well.}, where $\cS_1$ is used to compute the encoders and decoders and $\cS_2$ is used to learn the transformation $\Gamma$ between the encoded vectors.
Our two-stage algorithm follows
\begin{description}
	\item[Stage 1:] Compute the empirical encoders and decoders $E_{\cX}^n, D_{\cX}^n, E_{\cY}^n, D_{\cY}^n$ based on $\cS_1$. In the case of deterministic encoders, we skip Stage 1 and let $
E_{\cX}^n=E_{\cX}, \ D_{\cX}^n=D_{\cX}, \ E_{\cY}^n=E_{\cY}, \ D_{\cY}^n=D_{\cY}.
$

	\item[Stage 2:] Learn $\Gamma$ with $\cS_2$ by solving the following optimization problem
	\begin{align}
		\Gamma_{\rm NN}\in \argmin_{\Gamma\in \cF_{\rm NN}} \frac{1}{n}\sum_{i=n+1}^{2n}\| \Gamma\circ E_{\cX}^n(u_i) -E_{\cY}^n(v_i)\|_2^2
		\label{eq.PsiNN}
	\end{align}
for some $\cF_{\rm NN}$ class with a proper choice of parameters.
\end{description}
Our estimator of $\Psi$ is given as
$$\Psi_{\rm NN} :=
D_{\cY}^n\circ \Gamma_{\rm NN} \circ E_{\cX}^n,
$$
and the mean squared generalization error is defined as
\begin{align}
		\EE_{\cS}\EE_{u\sim \gamma}\| \Psi_{\rm NN}(u)- \Psi(u)\|_{\cY}^2.
		\label{eq.error}
\end{align}

\begin{table}[t!]\footnotesize
	\centering
	\begin{tabular}{c|l||c|l}
		\hline\hline
		Notation & Description & Notation & Description \\
		\hline
		$\cX$ &Input space & $\cY$  & Output space\\
		\hline
		$\Psi: \cX \rightarrow \cY$ & An unknown operator & $ \cS=\{u_i,v_i\}_{i=1}^{2n}$ & Given data set\\
		\hline
		$\gamma$ & A probability measure on $\cX$ & $\Psi_{\#}\gamma$  & \makecell[l]{Push forward measure of $\gamma$\\ under $\Psi$} \\
		\hline
		$\mu$ &  The probability measure of noise $\widetilde{\varepsilon}$ & $\zeta$ & The probability measure of $v=\Psi(u)+\widetilde{\varepsilon}$\\
		\hline
		$E_{\cX},D_{\cX}$ & Encoder and decoder of $\cX$ & $E_{\cY},D_{\cY}$ & Encoder and decoder of $\cY$\\
		\hline
		$E_{\cX}^n,D_{\cX}^n$ & \makecell[l]{Empirical estimations of  $E_{\cX},D_{\cX}$ \\ from noisy data} & $E_{\cY}^n,D_{\cY}^n$ & \makecell[l]{Empirical estimations of  $E_{\cY},D_{\cY}$ \\ from noisy data}\\
		\hline
		$d_{\cX}$ & Encoding dimension of $\cX$ & $d_{\cY}$ & Encoding dimension of $\cY$	\\
		\hline
		$\Pi_{\cX, d_{\cX}}$ & Projection $D_{\cX}\circ E_{\cX}$ & $\Pi_{\cY, d_{\cY}}$ & Projection $D_{\cY}\circ E_{\cY}$ \\
		\hline
		$\Pi_{\cX, d_{\cX}}^n$ & Empirical projection $D_{\cX}^n\circ E_{\cX}^n$& $\Pi_{\cY, d_{\cY}}^n$ & Empirical projection $D_{\cY}^n\circ E_{\cY}^n$\\
		\hline
		$\|\Pi_{\cX, d_{\cX}}(u)-u\|_{\cX}$ & Encoding error for $u$ in $\cX$ & $\|\Pi_{\cY, d_{\cY}}(v)-v\|_{\cY}$ & Encoding error for $v$ in $\cY$\\
		\hline
		$\cF_{\rm NN}$ & Neural network class &  $\Gamma_{\rm NN}$ & Neural network estimator in (\ref{eq.PsiNN})\\
		\hline\hline
	\end{tabular}
	\caption{Notations used in this paper. }\label{tab.notation}
\end{table}
\section{Main results}\label{sec.mainresults}
The main results of this paper provide statistical guarantees on the mean squared generalization error for the estimation of Lipchitz operators.


\subsection{Assumptions}
We first make some assumptions on the measure $\gamma$ and the operator $\Psi$.
\begin{assumption}[Compactly supported measure]
\label{assum.rho}
The probability distribution $\gamma$ is supported on a compact set $\Omega_{\cX}\subset \cX$. There exists $R_{\cX}>0$ such that, for any $u\in \Omega_{\cX}$, we have
	\begin{align}
		\|u\|_{\cX}\leq R_{\cX}.
	\end{align}
	\end{assumption}

\begin{assumption}[Lipschitz operator]
\label{assum.Psi.L}
There exists $L_{\Psi}>0$ such that
		$$
		\|\Psi(u_1)-\Psi(u_2)\|_{\cY}\leq L_{\Psi}\|u_1-u_2\|_{\cX}, \ \text{ for any } u_1,u_2\in \Omega_{\cX}.
		$$	
\end{assumption}
Assumption \ref{assum.rho} and \ref{assum.Psi.L} assume that $\gamma$ is compactly supported and $\Psi$ is Lipschitz continuous. We denote the image of $\Omega_{\cX}$ under the transformation $\Psi$ as
$$
	\Omega_{\cY}=\{v\in \cY: v=\Psi(u) \mbox{ for some } u\in\Omega_{\cX}\}.
$$
Assumption \ref{assum.rho} and \ref{assum.Psi.L} imply that $\Omega_{\cY}$ is bounded: there exists a constant $R_{\cY}=L_{\Psi}R_{\cX}>0$ such that for any $v\in \Omega_{\cY}$, we have $\|v\|_{\cY}\leq R_{\cY}$.

We next make some natural assumptions on the empirical encoders and decoders:
\begin{assumption}[Lipchitz encoders and decoders]\label{assum.ED.lip}
	The empirical encoders and decoders $E^n_{\cX},D^n_{\cX},E^n_{\cY},D^n_{\cY}$ satisfy:
	\begin{align*}
		&E^n_{\cX}(0_{\cX})=\mathbf{0},\ D^n_{\cX}(\mathbf{0})=0_{\cX},\ E^n_{\cY}(0_{\cY})=\mathbf{0},\ D^n_{\cY}(\mathbf{0})=0_{\cY},
	\end{align*}
where $\mathbf{0}$ denotes the zero vector, $0_{\cX}$ is the zero function in $\cX$ and $0_{\cY}$ is the zero function in $\cY$.

	They are also Lipschitz: there exist $L_{E_{\cX}^n},L_{D_{\cX}^n},L_{E_{\cY}^n},L_{D_{\cY}^n} >0$ such that, for any $u_1,u_2\in \cX$ and any $\ab_1,\ab_2\in \RR^{d_{\cX}}$, we have
	\begin{align*}
		 \|E_{\cX}^n(u_1)-E_{\cX}^n(u_2)\|_{2}\leq L_{E_{\cX}^n}\|u_1-u_2\|_{\cX}, \quad \|D_{\cX}^n(\ab_1)-D_{\cX}^n(\ab_2)\|_{\cX}\leq L_{D_{\cX}^n}\|\ab_1-\ab_2\|_2,
	\end{align*}
and for any $v_1,v_2\in \cY$ and any $\ab_1,\ab_2\in \RR^{d_{\cY}}$, we have
	\begin{align*}
		 \|E_{\cY}^n(v_1)-E_{\cY}^n(v_2)\|_{2}\leq L_{E_{\cY}^n}\|v_1-v_2\|_{\cY}, \quad \|D_{\cY}^n(\ab_1)-D_{\cY}^n(\ab_2)\|_{\cY}\leq L_{D_{\cY}^n}\|\ab_1-\ab_2\|_2.
	\end{align*}
\end{assumption}
\begin{remark} 
Assumption \ref{assum.ED.lip} is made on empirical encoders and decoders. The basis encoders and PCA encoders, which are most commonly used, satisfy Assumption \ref{assum.ED.lip} with the Lipchitz constants $L_{E_{\cX}^n}=L_{D_{\cX}^n}=L_{E_{\cY}^n}=L_{D_{\cY}^n}=1$,  independently of the training data (see Lemma \ref{lem.EncoDecolip} and Lemma \ref{lem.LipschitzFG}). 
\end{remark}

	\commentout{
\begin{enumerate}[(i)]
	\item Zero element
	\begin{align*}
		&E_{\cX}(0)=\mathbf{0},\ D_{\cX}(\mathbf{0})=0,\ E_{\cY}(0)=\mathbf{0},\ D_{\cY}(\mathbf{0})=0,\\
		&E^n_{\cX}(0)=\mathbf{0},\ D^n_{\cX}(\mathbf{0})=0,\ E^n_{\cY}(0)=\mathbf{0},\ D^n_{\cY}(\mathbf{0})=0,
	\end{align*}
	\item There exist $L_{E_{\cX}},L_{D_{\cX}} >0$ such that for any $u_1,u_2\in \cX$ and any $\ab_1,\ab_2\in \RR^{d_{\cX}}$, we have
	\begin{align*}
		&\|E_{\cX}(u_1)-E_{\cX}(u_2)\|_{2}\leq L_{E_{\cX}}\|u_1-u_2\|_{\cX}, \quad \|E_{\cX}^n(u_1)-E_{\cX}^n(u_2)\|_{2}\leq L_{E_{\cX}}\|u_1-u_2\|_{\cX},\\
		&\|D_{\cX}(\ab_1)-D_{\cX}(\ab_2)\|_{\cX}\leq L_{D_{\cX}}\|\ab_1-\ab_2\|_2, \quad \|D_{\cX}^n(\ab_1)-D_{\cX}^n(\ab_2)\|_{\cX}\leq L_{D_{\cX}}\|\ab_1-\ab_2\|_2.
	\end{align*}
	
	Assume there exist $L_{E_{\cY}},L_{D_{\cY}} >0$ such that for any $v_1,v_2\in \cY$ and  any $\ab_1,\ab_2\in \RR^{d_{\cY}}$, we have
	\begin{align*}
		&\|E_{\cY}(v_1)-E_{\cY}(v_2)\|_{2}\leq L_{E_{\cY}}\|v_1-v_2\|_{\cY}, \quad \|E_{\cY}^n(v_1)-E_{\cY}^n(v_2)\|_{2}\leq L_{E_{\cY}}\|u_1-u_2\|_{\cY},\\
		&\|D_{\cY}(\ab_1)-D_{\cY}(\ab_2)\|_{\cY}\leq L_{D_{\cY}}\|\ab_1-\ab_2\|_2, \quad \|D_{\cY}^n(\ab_1)-D_{\cY}^n(\ab_2)\|_{\cY}\leq L_{D_{\cY}}\|\ab_1-\ab_2\|_2.
	\end{align*}
\end{enumerate}
}

Assumption \ref{assum.ED.lip} implies that $E^n_{\cX}(u)$ and $E^n_{\cY}(v)$ are bounded for any $u\in \Omega_{\cX}$ and $v\in\Omega_{\cY}$.
For any $u\in \Omega_{\cX}$, we have
$	\|E^n_{\cX }(u)\|_{2}	\leq \|E^n_{\cX }(u)-E^n_{\cX}(0)\|_{2}+\|E^n_{\cX}(0)\|_{2}	\leq  L_{E^n_{\cX}}R_{\cX}.$
Similarly, for any $v\in\Omega_{\cY}$, we have
$
 \|E_{\cY }^n(v)\|_{2}\leq L_{E_{\cY}^n}R_{\cY}.
$

\begin{assumption}[Noise]\label{assum.noise}
	The random noise $\widetilde{\epsilon}$ satisfies 	\begin{enumerate}[label=(\roman*)]
		\item $\widetilde{\epsilon}$ is independent of $u$.
		\item $\EE[\widetilde{\epsilon}]=0$.
		\item  There exists $\widetilde{\sigma}>0$ such that $\|\widetilde{\epsilon}\|_{\cY}\leq \widetilde{\sigma}$.
	\end{enumerate}
\end{assumption}

Assumption \ref{assum.noise}(i)-(iii) are natural assumptions on noise. Assumption \ref{assum.noise}(i) is about the independence of the input and the noise, which is commonly used in nonparametric regression.
Assumption \ref{assum.noise}(iii) together with Assumption \ref{assum.ED.lip} imply that the perturbation of the encoded vectors are bounded: $\|E^n_{\cY}(\Psi(u)+\widetilde{\epsilon})-E^n_{\cY}(\Psi(u))\|_{\infty}\leq L_{E^n_{\cY}}\widetilde{\sigma}$.  We denote $\sigma=L_{E^n_{\cY}}\widetilde{\sigma}$ such that $$\|E^n_{\cY}(\Psi(u)+\widetilde{\epsilon})-E^n_{\cY}(\Psi(u))\|_{\infty}\le \sigma \text{ for any } u \text{ and } \widetilde{\epsilon}.$$ 

\begin{assumption}[Noise and encoder]\label{assum.noiseEncoder}
	For any noise satisfying Assumption \ref{assum.noise} and any given $\cS_1$, the conditional expectation satisfies 
	$$\EE_{\widetilde{\epsilon}}\left[E^n_{\cY}(\Psi(u)+\widetilde{\epsilon})-E^n_{\cY}(\Psi(u))|\cS_1\right]=\mathbf{0}, \text{ for any}\ u\in \Omega_{\cX},$$
	where   
	$E^n_{\cY}$ is the empirical encoder computed with $\cS_1$.
\end{assumption}

Assumption \ref{assum.noiseEncoder} requires that, if we condition on $\cS_1$ based on which the empirical encoder $E^n_{\cY}$ is computed, the perturbation on the encoded vector resulted from noise  has zero expectation. 
Assumption \ref{assum.noiseEncoder} is guaranteed for all linear encoders as long as Assumption \ref{assum.noise}(ii) holds: 
$$\EE_{\widetilde{\epsilon}}\left[E^n_{\cY}(\Psi(u)+\widetilde{\epsilon})-E^n_{\cY}(\Psi(u))|\cS_1\right]=\EE_{\widetilde{\epsilon}}\left[E^n_{\cY}(\widetilde{\epsilon})|\cS_1\right] = \mathbf{0}.$$
Basis encoders, including the PCA encoder, are linear encoders, so they all satisfy Assumption \ref{assum.noiseEncoder}.

\subsection{Generalization error with general encoders and decoders}
\label{sec.mainresults.error}

Our main result is an upper bound of the generalization error in (\ref{eq.error}) with general encoders and decoders. Our results can be applied to both network architectures defined in (\ref{eq.FNN}) and (\ref{eq.FNN.dense}). 
Our first theorem gives an upper bound of the generalization error with the network architecture defined in (\ref{eq.FNN}).


\begin{theorem}\label{thm.general}
	In Setting \ref{setting},
	suppose Assumption \ref{assum.rho} -- \ref{assum.noiseEncoder} hold. Let $\Gamma_{\rm NN}$ be the minimizer of (\ref{eq.PsiNN}) with the network architecture $\cF(d_{\cY},L,p,K,\kappa,M)$ in (\ref{eq.FNN}),  where
	\begin{equation}
		\begin{aligned}
			&L=O(\log n+\log d_{\cY}), \ p=O\left(d_{\cY}^{-\frac{d_{\cX}}{2+d_{\cX}}}n^{\frac{d_{\cX}}{2+d_{\cX}}}\right),\ K=O\left(d_{\cY}^{-\frac{d_{\cX}}{2+d_{\cX}}}n^{\frac{d_{\cX}}{2+d_{\cX}}}\log n\right), \\
			& M=\sqrt{d_{\cY}}L_{E^n_{\cY}}R_{\cY},\ \kappa=\max\left\{1,\sqrt{d_{\cY}}L_{E^n_{\cY}}R_{\cY}, \sqrt{d_{\cX}}L_{E^n_{\cX}}R_{\cX},L_{E^n_{\cY}}L_{D^n_{\cX}}L_{\Psi} \right\}.
		\end{aligned}\label{eq.NN.parameter.n}
	\end{equation}	
	  Then we have
	  \begin{align}
	  	&\EE_{\cS}\EE_{u\sim \gamma}\|D_{\cY}^n\circ \Gamma_{\rm NN} \circ E_{\cX}^n(u)- \Psi(u)\|_{\cY}^2 \nonumber \\ 
	  	\leq &\ C_1(\sigma^2+R_{\cY}^2)d_{\cY}^{\frac{4+d_{\cX}}{2+d_{\cX}}}n^{-\frac{2}{2+d_{\cX}}}\log^3 n +C_2(\sigma^2+R_{\cY}^2)d_{\cY}^2(\log d_{\cY})n^{-1}\nonumber \\
	  	& + C_3\EE_{\cS}\EE_{u\sim \gamma}\| \Pi_{\cX,d_{\cX}}^n(u)- u\|_{\cX}^2+2\EE_{\cS}\EE_{w\sim \Psi_{\#}\gamma}\|\Pi_{\cY,d_{\cY}}^n(w)- w\|_{\cY}^2,
	  	\label{eq.error.general}
	  \end{align}
where $C_1,C_2$ are constants depending on $d_{\cX},R_{\cX}, R_{\cY},L_{E^n_{\cX}}, L_{E^n_{\cY}},L_{D^n_{\cX}},L_{D^n_{\cY}},L_{\Psi}$ and $C_3=16L_{D^n_{\cY}}^2L_{E^n_{\cY}}^2L_{\Psi}^2$.
\end{theorem}

Our second theorem gives an upper bound of the generalization error with the network architecture defined in (\ref{eq.FNN.dense}).
\begin{theorem}\label{thm.general.dense}
	In Setting \ref{setting},
	suppose Assumption \ref{assum.rho} -- \ref{assum.noiseEncoder} hold. Let $\Gamma_{\rm NN}$ be the minimizer of (\ref{eq.PsiNN}) with the network architecture $\cF(d_{\cY},L,p,M)$ in (\ref{eq.FNN.dense}) with
	\begin{equation}
		\begin{aligned}
			&L=O(\widetilde{L}\log \widetilde{L}), \ p=O\left(\widetilde{p} \log \widetilde{p}\right), M=\sqrt{d_{\cY}}L_{E^n_{\cY}}R_{\cY},
		\end{aligned}\label{eq.NN.parameter.n.dense}
	\end{equation}
	where $\widetilde{L},\widetilde{p}>0$ are positive integers satisfying 
	\begin{equation}
	\widetilde{L}\widetilde{p}=\left\lceil d_{\cY}^{-\frac{d_{\cX}}{4+2d_{\cX}}}n^{\frac{d_{\cX}}{4+2d_{\cX}}}\right\rceil.
	\label{eq.FNN.densecondition}
	\end{equation}
	Then we have
	\begin{align}
		&\EE_{\cS}\EE_{u\sim \gamma}\|D_{\cY}^n\circ \Gamma_{\rm NN} \circ E_{\cX}^n(u)- \Psi(u)\|_{\cY}^2 \nonumber\\
		\leq&\ C_{4}(\sigma^2+R_{\cY}^2)d_{\cY}^{\frac{4+d_{\cX}}{2+d_{\cX}}}n^{-\frac{2}{2+d_{\cX}}}\log^6 n  +  C_3\EE_{\cS}\EE_{u\sim \gamma}\| \Pi_{\cX,d_{\cX}}^n(u)- u\|_{\cX}^2+2\EE_{\cS}\EE_{w\sim \Psi_{\#}\gamma}\|\Pi_{\cY,d_{\cY}}^n(w)- w\|_{\cY}^2,
	\label{eq.error.general.dense}
	\end{align}
	where $C_{4}$ is a constant depending on $d_{\cX},R_{\cX}, R_{\cY},L_{E^n_{\cX}},L_{E^n_{\cY}},L_{D^n_{\cX}},L_{\Psi}$ and $C_3=16L_{D^n_{\cY}}^2 L_{E^n_{\cY}}^2L_{\Psi}^2$ is the same one in Theorem \ref{thm.general}.
\end{theorem}

Theorem \ref{thm.general} is proved in Section \ref{thm.general.proof} and Theorem \ref{thm.general.dense} is proved in Section \ref{thm.general.dense.proof}. In Theorem \ref{thm.general.dense}, we have chosen the optimal $\widetilde{L}\widetilde{p}$ to balance the bias and variance term. For readers who are interested in the generalization error with arbitrary network depth $L$ and width $p$, please see our proof in Section \ref{thm.general.dense.proof}. The constants in both theorems only depend on the settings of the problem, and the choices of encoders and decoders. They do not depend on properties of $\Gamma_{\rm NN}$.
For both network architectures, the upper bound in  \eqref{eq.error.general} and (\ref{eq.error.general.dense}) consists of a network estimation error and the projection errors in the $\cX$ and $\cY$ space.
\begin{itemize}
	\item The first two terms in \eqref{eq.error.general} and the first term in (\ref{eq.error.general.dense}) represent the {\bf network estimation error} for the transformation $\Gamma: \RR^{d_{\cX}} \rightarrow \RR^{d_{\cY}}$ which maps the encoded vector $E^n_{\cX}(u)$ for $u$ in $\cX$ to the encoded vector $E^n_{\cY}(\Phi(u))$ for $\Phi(u)$ in $\cY$. This error decays exponentially as the sample size $n$ increases with an exponent depending on the dimension $d_{\cX}$ of the encoded space. The dimension $d_{\cX}$ appears in the exponent and $d_{\cY}$ appears as a constant factor. This is because that the transformation $\Gamma$ has $d_{\cY}$ outputs and each output is a function from $\RR^{\cX}$ to $\RR$. Therefore the rate is only cursed by the input dimension $d_{\cX}$.
	
	\item The last two terms in \eqref{eq.error.general} and (\ref{eq.error.general.dense}) are {\bf projection errors} in the ${\cX}$ and ${\cY}$ space, respectively. If the measure $\gamma$ is concentrated near a $d_{\cX}$-dimensional subspace in $\cX$, both projection errors can be made small if the encoder and decoder are properly chosen as the projection onto this $d_{\cX}$-dimensional subspace (see Section \ref{sec.lowD}). 
	
\end{itemize}


We next compare the difference between the network architectures in Theorem \ref{thm.general} and Theorem \ref{thm.general.dense}.
Denote the network architecture in Theorem \ref{thm.general} and Theorem \ref{thm.general.dense} by $\cF_1$ and $\cF_2$, respectively.
The architecture $\cF_1$ has the depth and width scaling properly with respect to each other, and an upper bound on all weight parameters and a cardinality constraint. The cardinality constraint is nonconvex and therefore not practical for training this neural network.
 The architecture $\cF_2$ has more flexibility in the choice of  depth and width as long as  \eqref{eq.FNN.densecondition} is satisfied. The cardinality is removed for practical concerns. When we set $\widetilde{L}=O(\log n), \widetilde{p}=O(n^{\frac{d_{\cX}}{4+2d_{\cX}}}\log^{-1}n)$ in $\cF_2$,  both networks have a depth of $O(\log n)$, while the width of $\cF_1$ is the square of that of $\cF_2$, i.e., $\cF_1$ is wider than $\cF_2$.  The comparison between $\cF_1$ and $\cF_2$ is summarized in Table \ref{tab.NetworkCompare}.

\begin{table}[ht!]\footnotesize
	\centering
	\begin{tabular}{c||c|c}
		\hline\hline
		& $\cF_1$ in (\ref{eq.FNN}) & $\cF_2$ in (\ref{eq.FNN.dense})\\
		\hline \hline
		\multicolumn{3}{c}{General comparison}\\
		\hline
		\makecell[c]{Network architecture \\ with a given $n$} &Fixed $L$ and $p$ depending on $n$& \makecell[c]{One has the flexibility to choose $L$ and $p$\\ as long as \eqref{eq.FNN.densecondition} depending on $n$ is satisfied }\\
		\hline
		Constraints on cardinality & Yes & No\\
		\hline
		\makecell[c]{Constraints on the\\ magnitude of weight parameters} & Yes & No\\
		\hline \hline
		\multicolumn{3}{c}{Set $\widetilde{L}=O(\log n), \widetilde{p}=O(n^{\frac{d_{\cX}}{4+2d_{\cX}}}\log^{-1}n)$ in $\cF_2$}\\
		\hline \hline
		Depth $L$ & $O(\log n)$ & $O(\log n)$\\
		\hline 
		Width $p$ & $O\left(d_{\cY}^{-\frac{d_{\cX}}{2+d_{\cX}}}n^{\frac{d_{\cX}}{2+d_{\cX}}}\right)$ & $O\left(d_{\cY}^{-\frac{d_{\cX}}{4+2d_{\cX}}}n^{\frac{d_{\cX}}{4+2d_{\cX}}}\right)$\\
		\hline \hline
	\end{tabular}
\caption{\label{tab.NetworkCompare} Comparison of the network architectures in Theorem \ref{thm.general} and \ref{thm.general.dense}.}
\end{table}

In the rest of this paper, we focus on the network architecture in Theorem \ref{thm.general.dense} and discuss its applications in various scenarios. Theorem \ref{thm.general} can also be applied in each case with a similar upper bound.


\section{Generalization error with basis encoders and decoders} \label{sec.deter}
In this section, we discuss the application of Theorem \ref{thm.general.dense} when the encoder is chosen to be a deterministic basis encoder with a given orthonormal basis of the Hilbert space. Popular choices of orthonormal bases include orthogonal polynomials (e.g., Legendre polynomials \citep{szeg1939orthogonal,chkifa2015breaking,cohen2015approximation}) and trigonometric functions \citep{orszag1971accurate,chen1998applications,li2016characterizing}.  

\subsection{Basis encoders and decoders}
Let $\cH$ be a separable Hilbert space equipped with an inner product $\langle\cdot,\cdot \rangle_{\cH}$, and $\{\phi_k\}_{k=1}^{\infty}$ be an orthonormal basis of $\cH$ such that
$\langle \phi_{k_1},\phi_{k_2}\rangle_{\cH}=0 \mbox{ whenever } k_1\neq k_2$ and $\|\phi_k\|_{\cH}=1 \mbox{ for any } k.$
For any $u\in\cH$, we have
\begin{align}
	u=\sum_{k=1}^{\infty} \langle u,\phi_k\rangle_{\cH} \phi_k.
	\label{eq.basis.decomposition}
\end{align}
For a fixed positive integer $d$ representing the encoding dimension, we define the encoder of $\cH$ as
\begin{align}
	E_{\cH,d}(u)=\left[\langle u,\phi_1\rangle_{\cH},...,\langle u,\phi_d\rangle_{\cH} \right]^{\top} \in \RR^d, \ \text{ for any } u \in \cH,
	\label{eq.encoder.deter}
\end{align}
which gives rise to the coefficients associated with a fixed set of $d$ basis functions in the decomposition (\ref{eq.basis.decomposition}).
The decoder $D_{\cH,d}$ is defined as
\begin{align}
	D_{\cH,d}(\ab)=\sum_{k=1}^d a_k\phi_k, \text{ for any } \ab\in\RR^d.
	\label{eq.decoder.deter}
\end{align}

The basis encoder and decoder naturally satisfy the Lipchitz property with a Lipschitz constant 1 (see a proof of Lemma \ref{lem.EncoDecolip} in Appendix \ref{lem.EncoDecolip.proof}).
\begin{lemma}\label{lem.EncoDecolip}
	The encoder $E_{\cH,d}$ and decoder $D_{\cH,d}$ defined in (\ref{eq.encoder.deter}) and (\ref{eq.decoder.deter}) satisfy
	\begin{align}
		\|E_{\cH,d}(u)-E_{\cH,d}(\widetilde{u})\|_2&\leq \|u-\widetilde{u}\|_{\cH}, \label{eq.deter.encoder.lip}\\
		\|D_{\cH,d}(\ab)- D_{\cH,d}(\widetilde{\ab})\|_{\cH} &= \|\ab-\widetilde{\ab}\|_2,
		\label{eq.deter.decoder.lip}
	\end{align}
for any $u,\widetilde{u}\in \cH$ and $\ab,\widetilde{\ab}\in \RR^d$.
\end{lemma}

\begin{remark}
	All encoders in the form of (\ref{eq.encoder.deter}) are linear operators and therefore satisfy Assumption \ref{assum.noiseEncoder} as long as Assumption \ref{assum.noise}(ii) holds.
\end{remark}

\subsection{Generalization error with basis encoders}

We next consider the generalization error when the elements in $\cX$ and $\cY$ are encoded by basis encoders with the encoding dimension $d_{\cX}$ and $d_{\cY}$, respectively. 
Substituting the Lipschitz constants of all encoders and decoders by 1 in Theorem \ref{thm.general.dense}, we obtain the following corollary:
\begin{corollary}\label{coro.basis.dense}
	In Setting \ref{setting},
	suppose Assumption \ref{assum.rho} -- \ref{assum.noise} hold. Let $\Gamma_{\rm NN}$ be the minimizer of (\ref{eq.PsiNN}) with the network architecture $\cF(d_{\cY},L,p,M)$ in (\ref{eq.FNN.dense}) with
	\begin{equation}
		\begin{aligned}
			&L=O(\widetilde{L}\log \widetilde{L}), \ p=O\left(\widetilde{p} \log \widetilde{p}\right), M=\sqrt{d_{\cY}}R_{\cY},
		\end{aligned}
	\label{eq.NN.parameter.n.dense.deter}
	\end{equation}
	where $\widetilde{L},\widetilde{p}>0$ are positive integers satisfying \eqref{eq.FNN.densecondition}.
	Then we have
	\begin{align}
		&\EE_{\cS}\EE_{u\sim \gamma}\|D_{\cY}^n\circ \Gamma_{\rm NN} \circ E_{\cX}^n(u)- \Psi(u)\|_{\cY}^2 \nonumber\\
		\leq&\ C_{4}(\sigma^2+R_{\cY}^2)d_{\cY}^{\frac{4+d_{\cX}}{2+d_{\cX}}}n^{-\frac{2}{2+d_{\cX}}}\log^6 n  +  16 L_{\Psi}^2\EE_{\cS}\EE_{u\sim \gamma}\| \Pi_{\cX,d_{\cX}}^n(u)- u\|_{\cX}^2+2\EE_{\cS}\EE_{w\sim \Psi_{\#}\gamma}\|\Pi_{\cY,d_{\cY}}^n(w)- w\|_{\cY}^2,
		\label{eq.error.general.basis.dense}
	\end{align}
	where $C_{4}$ is a constant depending on $d_{\cX},R_{\cX}, R_{\cY},L_{\Psi}$.
\end{corollary}

Popular choices of orthonormal bases are orthogonal polynomials and trigonometric functions. We next provide an upper bound on the generalization error when Legendre polynomials or trigonometric functions are used for encoding and decoding.
In the rest of this section, we assume 
$\cX=\cY= L^2([-1,1]^D)$ with the inner product 
\begin{align}
	\langle u_1,u_2\rangle=\int_{[-1,1]^D} u_1(\xb)\overline{u_2(\xb)}d\xb,
	\label{eqinnerproduct}
\end{align}
where $\overline{u_2(\xb)}$ denotes the complex conjugate of $u_2(\xb)$.

%

\subsection{ Legendre polynomials} On the interval $[-1,1]$, one-dimensional Legendre polynomials $\{\widetilde{P}_k\}_{k=0}^{\infty}$ are defined recursively as
\begin{align*}
	\begin{cases}
		\widetilde{P}_0(x)=1,\\
		\widetilde{P}_1(x)=x,\\
		\widetilde{P}_{k+1}(x)=\frac{1}{k+1}\left[(2k+1)x\widetilde{P}_{k}(x)-k\widetilde{P}_{k-1}(x)\right].
	\end{cases}
\end{align*}
The Legendre polynomials satisfy
\begin{align*}
	\int_{-1}^1 \widetilde{P}_k(x)\widetilde{P}_l(x)dx=\frac{2}{2k+1}\delta_{kl},
\end{align*}
where $\delta_{kl}$ is the Kronecker delta which equals to 1 if $k=l$ and equals to 0 otherwise. We define the normalized Legendre polynomials as 
\begin{align}
	P_k(x)=\sqrt{\frac{2k+1}{2}}\widetilde{P}_k(x).
\end{align}
In the Hilbert space $L^2([-1,1]^D)$, the $D$-variate normalized Legendre polynomials are defined as 
\begin{align}
	\phi^{\rm L}_{\kb}=\prod_{j=1}^{D} P_{k_j}(x_j),
\end{align}
where $\kb=[
	k_1\ \cdots \ k_D ]^{\top}$. The orthonormal basis of Legendre polynomials in $L^2([-1,1]^D)$ is $\{\phi^{\rm L}_{\kb}\}_{\kb\in\NN_0^D}$. 
	
The encoder with Legendre polynomials can be naturally defined as the expansion coefficients associated with low-order polynomials. 
Specifically, when $\cX=L^2([0,1]^D)$, we fix a positive integer $r_{\cX}$ representing the highest degree of the polynomials in each dimension and consider the following set of low-order polynomials 
$$
\Phi^{{\rm L},r_{\cX}}:=\{\phi^{\rm L}_\kb: \|\kb\|_{\infty}\leq r_{\cX}\}.
$$
The encoder $E_{\cX}$ and decoder $D_{\cX}$ can be defined according to (\ref{eq.encoder.deter}) and (\ref{eq.decoder.deter}) using the basis functions in $\Phi^{{\rm L},r_{\cX}}$. In the space $\cY=L^2([0,1]^D)$, the encoder $E_{\cY}$ and decoder $D_{\cY}$ can be defined similarly with basis functions in $\Phi^{{\rm L},r_{\cY}}$ for some positive integer $r_{\cY}$.

When Legendre polynomials are used for encoding, the encoding error is guaranteed for regular functions, such as H\"{o}lder functions.


\begin{definition}[H\"{o}lder space]
	Let $k\geq 0$ be an integer and $0<\alpha\leq 1$. A function $f: [-1,1]^D\rightarrow \RR$ belongs to the H\"{o}lder space $\cC^{k,\alpha}([-1,1]^D)$ if
	\begin{align*}
		\|f\|_{\cC^{k,\alpha}}:= \max_{|\kb|\leq k}\sup\limits_{ \xb \in [-1,1]^D} |\partial^{\kb}f (\xb)| + \max\limits_{|\kb|=k}
		\sup\limits_{ \xb_1\neq \xb_2 \in [-1,1]^D}\frac{|\partial^{\kb}f(\xb_1)- \partial^{\kb}f (\xb_2)|}{\|\xb_1-\xb_2\|_2^{\alpha}} < \infty,
	\end{align*}
where $\partial^{\kb}f=\frac{\partial^{|\kb|} f}{\partial x_1^{k_1} \partial x_2^{k_2} \cdots \partial x_D^{k_D}}$.
\end{definition}

For a given $k$ and $\alpha$, any functions in $\cC^{k,\alpha}([-1,1]^D)$ has continuous partial derivatives up to order $k$. In particular, $\cC^{0,1}([-1,1]^D)$ consists of all Lipschitz functions defined on $[-1,1]^D$.

We assume that the probability measure $\gamma$ in $\cX$ and the pushforward measure $\Psi_{\#}\gamma$ in $\cY$ are supported on subsets of the H\"{o}lder space.

\begin{assumption}[H\"{o}lder input and output]
\label{assum.holder.leg}
	Let $\cX=\cY=L^2([-1,1]^D)$ with the inner product \eqref{eqinnerproduct}. For some integer $k>0$ and $0<\alpha\leq 1$, the support of the probability measure $\gamma$ and the pushforward measure $\Psi_{\#}\gamma$ satisfies 
	$$\Omega_{\cX}\subset \cC^{k,\alpha}([-1,1]^D), \quad \Omega_{\cY}\subset \cC^{k,\alpha}([-1,1]^D).$$ There exist $C_{\cH,\cX}>0$ and $C_{\cH,\cY}>0$ such that, for any $u\in \Omega_{\cX}$ and $v\in \Omega_{\cY}$ 
	\begin{align*}
		\|u\|_{\cC^{k,\alpha}}<C_{\cH,\cX},\quad \|v\|_{\cC^{k,\alpha}}<C_{\cH,\cY}.
	\end{align*}
\end{assumption}

When Legendre polynomials are used to encode H\"{o}lder functions, the generalization error for the operator is given as below: 
\begin{corollary}\label{coro.leg}
	In Setting \ref{setting},
	suppose Assumption \ref{assum.rho}--\ref{assum.holder.leg} hold. Denote $s=k+\alpha$.
Fix positive integers 	$d_{\cX}$ and $d_{\cY}$ such that $d_{\cX}^{1/D}$ and $d_{\cY}^{1/D}$ are integers. Suppose the encoders and decoders are chosen as in (\ref{eq.encoder.deter}) and (\ref{eq.decoder.deter}) with basis functions $\Phi^{{\rm L},d_{\cX}^{1/D}}$ and $\Phi^{{\rm L},d_{\cY}^{1/D}}$ in $\cX$ and $\cY$, respectively. 
	Let $\Gamma_{\rm NN}$ be the minimizer of (\ref{eq.PsiNN}) with the network architecture $\cF(d_{\cY},L,p,M)$ in (\ref{eq.FNN.dense}) where $L,p,M$ are set as in (\ref{eq.NN.parameter.n.dense.deter}). We have
	\begin{align*}
		&\EE_{\cS}\EE_{u\sim \gamma}\|D_{\cY}^n\circ \Gamma_{\rm NN} \circ E_{\cX}^n(u)- \Psi(u)\|_{\cY}^2 
		\leq C_{4}(\sigma^2+R_{\cY}^2)d_{\cY}^{\frac{4+d_{\cX}}{2+d_{\cX}}}n^{-\frac{2}{2+d_{\cX}}}\log^6 n + C_5L_{\Psi}^2d_{\cX}^{-\frac{2s}{D}}+C_6d_{\cY}^{-\frac{2s}{D}}.
	\end{align*}
	where $C_{4}$ depends on $d_{\cX},R_{\cX},R_{\cY},L_{\Psi}$, and $C_5,C_6$ depend on $D,C_{\cH,\cX},C_{\cH,\cY}, L_{\Psi}$.
\end{corollary}

Corollary \ref{coro.leg} is proved in Section \ref{coro.leg.proof}.
 In Corollary \ref{coro.leg}, the last two terms represent the projection errors in $\cX$ and $\cY$, respectively. 
When $D$ is large, both terms decay slowly as $d_{\cX}$ and $d_{\cY}$ increase. These two error terms remain the same if we choose the encoders given by finite element bases in traditional numerical PDE methods.  
For example, we consider learning a PDE solver where the operator $\Psi$ represents a map from the initial condition to the PDE solution at a certain time. Assumption \ref{assum.holder.leg} assumes that the initial condition and the PDE solution are H\"older functions.  Suppose we discretize the domain and represent the solution by finite element basis such that the diameter of all finite elements is no larger than $h$ for some $0<h<1$. Let $W^{k,2}([-1,1]^D)$ denote the Sobolev space. We say a set of basis functions are $k$--order if they are in $W^{k,2}([-1,1]^D)$. If the finite element method with $k$--th order basis functions is used to approximate the PDE solution, under appropriate assumptions and for any positive integer $k$, the squared approximation error is $O(h^{2k})$ \citep[Corollary 1.109]{ern2004theory}.
In this case, the total number of basis functions is $O(h^{-D})$. Taking such a finite element approximation as our encoder for $\Omega_{\cX}$, we have $d_{\cX}=O(h^{-D})$ and the resulting squared projection error is of $O(d_{\cX}^{-\frac{2k}{D}} )$. In particular, if sparse grids \citep{bungartz2004sparse} are used to construct basis functions, the approximation errors for the encoder and decoder can be further reduced. 

In the setting of Corollary \ref{coro.leg}, we only assume the global smoothness of input and output functions. The global approximation encoder by Legendre polynomials (or trigonometric functions in the following subsection) leads to a slow rate of convergence: In Corollary \ref{coro.leg}, if we choose $d_{\cX}=(\log n)^{\frac{1}{2}}$ and when $n\geq \exp\left(\max\left\{100,\left(\frac{7}{2}+\frac{s}{2D}\right)^{6}\right\}\right)$, the squared generalization error decays in the order of $(\log n)^{-\frac{s}{D}}$ (see a derivation in Appendix \ref{sec.n0.proof}). 


However, in practice, when we solve PDEs, the initial conditions and PDE solutions often exhibit low-dimensional structures. For example, the initial conditions and PDE solutions often lie on a low-dimensional subspace or manifold, or the solver itself has low complexity (see Section \ref{sec.lowD} and \citep{haasdonk2017reduced,rozza2014fundamentals} for details). Therefore, one can use a few bases (small $d_{\cX}$ and $d_{\cY}$) to achieve a small projection error, leading to a fast rate of convergence in the generalization error.


\subsection{Trigonometric functions} 
Trigonometric functions and the Fourier transform have been widely used in various applications where the computation is converted from the spacial domain to the frequency domain. Let $\{T_k(x)\}_{k=1}^{\infty}$ be  one-dimensional trigonometric functions defined on $[-1,1]$ such that
\begin{align}
	\begin{cases}
		T_1=1/2,\\
		T_{2k}=\sin(k\pi x) \ \mbox{ for } k>1,\\
		T_{2k+1}=\cos(k\pi x) \ \mbox{ for } k>1.
	\end{cases}
\label{eq.trig}
\end{align}
In the Hilbert space $L^2([-1,1]^D)$,  the trigonometric basis is given as $\{\phi_{T,\kb}\}_{\kb\in \NN^D}$ with
\begin{align}
\phi^{\rm T}_{\kb}(\xb)=\prod_{j=1}^D T_{k_j}(x_j).
\label{eq.trig.basis}
\end{align}
When $\cX=L^2([0,1]^D)$, we fix a positive integer $r_{\cX}$ and define the set of low-frequency basis 
\begin{align*}
	\Phi^{{\rm T},r_{\cX}}=\{\phi^{\rm T}_{\kb}: \|\kb\|_{\infty}\leq r_{\cX}\}.
\end{align*}
We set the encoder $E_{\cX}$ and decoder $D_{\cX}$ in $\cX$ according to (\ref{eq.encoder.deter}) and (\ref{eq.decoder.deter}) using the basis functions in $\Phi^{{\rm T},r_{\cX}}$. Similarly, we set the encoder $E_{\cY}$ and decoder $D_{\cY}$ in $\cY$ using  the basis functions in $\Phi^{{\rm T},r_{\cY}}$ for some positive integer $r_{\cY}$.

Let $\cP$ be the set of periodic functions on $[-1,1]^D$.
We assume that the input and output functions are periodic H\"{o}lder functions.
\begin{assumption}\label{assum.holder.trig}
Let $\cX=\cY=L^2([-1,1]^D)$ with the inner product \eqref{eqinnerproduct}. For some integer $k>0$ and $0<s\leq 1$, the support of the probability measure $\gamma$ and the pushforward measure $\Psi_{\#}\gamma$ satisfies 
	$$\Omega_{\cX}\subset \cP\cap \cC^{k,\alpha}([-1,1]^D), \quad \Omega_{\cY}\subset \cP\cap \cC^{k,\alpha}([-1,1]^D).$$ There exist $C_{\cH_P,\cX}>0$ and $C_{\cH_P,\cY}>0$ such that for any $u\in \Omega_{\cX}$ and $v\in \Omega_{\cY}$
	\begin{align*}
		\|u\|_{\cC^{k,\alpha}}<C_{\cH_P,\cX},\quad \|v\|_{\cC^{k,\alpha}}<C_{\cH_P,\cY}.
	\end{align*}
\end{assumption}

When trigonometric functions are used to encode periodic H\"{o}lder functions, the generalization error for the operator is given as below: 
\begin{corollary}\label{coro.trig}
	Consider Setting \ref{setting}.
	Suppose Assumption \ref{assum.rho}--\ref{assum.noiseEncoder} and \ref{assum.holder.trig} hold. Denote $s=k+\alpha$. Fix positive integers 	$d_{\cX}$ and $d_{\cY}$ such that $d_{\cX}^{1/D}$ and $d_{\cY}^{1/D}$ are integers. 
	 Suppose the encoders and decoders are chosen as in (\ref{eq.encoder.deter}) and (\ref{eq.decoder.deter}) with basis functions $\Phi^{{\rm T},d_{\cX}^{1/D}}$ and $\Phi^{{\rm T}, d_{\cY}^{1/D}}$ for $\cX$ and $\cY$, respectively.
	Let $\Gamma_{\rm NN}$ be the minimizer of (\ref{eq.PsiNN}) with the network architecture $\cF(d_{\cY},L,p,M)$ in (\ref{eq.FNN.dense}) where $L,p,M$ are set as in (\ref{eq.NN.parameter.n.dense.deter}). We have
	\begin{align*}
		&\EE_{\cS}\EE_{u\sim \gamma}\|D_{\cY}^n\circ \Gamma_{\rm NN} \circ E_{\cX}^n(u)- \Psi(u)\|_{\cY}^2 
		\leq C_{4}(\sigma^2+R_{\cY}^2)d_{\cY}^{\frac{4+d_{\cX}}{2+d_{\cX}}}n^{-\frac{2}{2+d_{\cX}}}\log^6 n + C_7L_{\Psi}^2d_{\cX}^{-\frac{2s}{D}}+C_8d_{\cY}^{-\frac{2s}{D}}.
	\end{align*}
	where $C_{4}$ depends on $d_{\cX},R_{\cX},R_{\cY},L_{\Psi}$, and $C_7,C_8$ depend on $D,C_{\cH_P,\cX},C_{\cH_P,\cY},L_{\Psi}$.
\end{corollary}

Corollary \ref{coro.trig} is proved in Section \ref{coro.trig.proof}. The generalization error with trigonometric basis encoder in Corollary \ref{coro.trig}  is similar to the error with Legendre polynomials in Corollary \ref{coro.leg}. If only the global smoothness of input and output functions is assumed, the generalization error decays at a low rate. A faster rate can be achieved if we exploit the low-dimensional structures of the input and output functions.  

\section{Generalization error for PCA encoders and decoders}\label{sec.PCA}

When the given data are concentrated near a low-dimensional subspace, PCA is an effective tool for dimension reduction. In this section, we consider the PCA encoder, where the orthonormal basis is estimated from the training data.
\subsection{PCA encoders and decoders}\label{sec.PCA.intro}
Let $\rho$ be a probability measure on a separable Hilbert space $\cH$. Define the covariance operator with respect to $\rho$ as
\begin{align}
	G_{\rho}=\EE_{u \sim\rho} [u\otimes u],
	\label{eq.C}
\end{align}
where $\otimes$ denotes the outer product $
	(f\otimes g)(h)=\langle g,h\rangle_{\cH} f
$
for any $f,g,h\in\cH$, and $\langle \cdot,\cdot\rangle_{\cH}$ denotes the inner product in $\cH$. 
Let $\{\lambda_k\}_{k=1}^{\infty}$ be the eigenvalues of $G_{\rho}$ in a non-increasing order, and $\phi_k$ be the eigenfunction associated with $\lambda_k$.
For any $u\in \cH$, we have
$$
u=\sum_{j=1}^{\infty} \langle u,\phi_{j}\rangle_{\cH}\phi_{j}.
$$

For a fixed positive integer $d$, the eigenfunctions $\{\phi_k\}_{k=1}^d$ associated with the top $d$ eigenvalues are called the first $d$ principal components. Fixing $d$, we define the encoder operator $E_{\cH,d}:\cH\rightarrow \RR^d$ as
\begin{align}
	E_{\cH,d}(u)=\left[ \langle u,\phi_{1}\rangle, \langle u,\phi_{2}\rangle, ..., \langle u,\phi_{d}\rangle\right]^{\top}, \text{ for any  } u\in \cH,
	\label{eq.PCA.encoder}
\end{align}
which gives rise to the coefficients of $u$ associated with the first $d$ principal components.
The decoder $D_{\cH,d}: \RR^d\rightarrow \cH$ is defined as
\begin{align}
	D_{\cH,d}(\ab)=\sum_{j=1}^d a_j\phi_{j} , \text{ for any  } \ab=[a_1,...,a_d]^{\top}\in \RR^d.
	\label{eq.PCA.decoder}
\end{align}


Given $n$ i.i.d samples $\{u_i\}_{i=1}^n$ from $\rho$, the empirical covariance operator is 
\begin{align}
	G_{\rho}^n=\frac{1}{n}\sum_{i=1}^n u_i\otimes u_i.
	\label{eq.Cn}
\end{align}
Let $\{\lambda_k^n\}_{k=1}^{\infty}$ be the eigenvalues of $G^n_{\rho}$ in a non-increasing order, and $\phi^n_k$ be the eigenfunction associated with $\lambda_k^n$.
We define the empirical encoder $E_{\cH,d}^n: \cH \rightarrow \RR^d$ as 
\begin{align}
	E_{\cH,d}^n(u)=\left[ \langle u,\phi_{1}^n\rangle, \langle u,\phi_{2}^n\rangle, ..., \langle u,\phi_{d}^n\rangle\right]^{\top} \text{ for any  } u\in \cH.
	\label{eq.PCA.encoder.n}
\end{align}
The empirical decoder is
\begin{align}
	D_{\cH,d}^n(\ab)=\sum_{j=1}^d a_j{ \phi_{j}^n} \text{ for any  } \ab\in \RR^d.
	\label{eq.PCA.decoder.n}
\end{align}

The PCA encoders and decoders $E_{\cH,d},D_{\cH,d},E_{\cH,d}^n,D_{\cH,d}^n$ are Lipchitz operators with a Lipchitz constant $1$.
\begin{lemma}\label{lem.LipschitzFG}
	Let $\cH$ be a separable Hilbert space and $\rho$ be a probability measure on $\cH$. For any integer $d>0$, let $E_{\cH,d}$ and $D_{\cH,d}$ be the PCA encoder and decoder and $E_{\cH,d}^n$ and $D_{\cH,d}^n$ be their empirical counterparts. Then we have
	\begin{align*}
		 \|E_{\cH,d}^n(u)-E_{\cH,d}^n(\widetilde{u})\|_2 &\leq \|u-\widetilde{u}\|_{\cH},\ \text{for any } u,\widetilde{u}\in \cH,
	\\
		 \|D_{\cH,d}^n(\ab)-D_{\cH,d}^n(\widetilde{\ab})\|_{\cH} &= \|\ab-\widetilde{\ab}\|_2, \ \text{for any } \ab,\widetilde{\ab}\in \RR^d.
	\end{align*}
\end{lemma}
Lemma \ref{lem.LipschitzFG} can be proved in the same way as Lemma \ref{lem.EncoDecolip}. The proof is omitted here. 

\subsection{Generalization error with PCA encoders and decoders}

In this subsection, we choose PCA encoders and decoders for $\cX$ and $\cY$. 
For the $\cX$ space, we define the covariance operator and its empirical counterpart as 
$$
G_{\gamma}=\EE_{u\sim \gamma} u\otimes u \quad \mbox{ and } \quad G^n_{\gamma}=\frac{1}{n}\sum_{i=1}^n u_i\otimes u_i.
$$
Let $\{\phi_{\gamma,k}\}_{k=1}^{d_{\cX}}$ and $\{\phi^n_{\gamma,k}\}_{k=1}^{d_{\cX}}$ be the first $d_{\cX}$ principle components of $G_{\gamma}$ and $G_{\gamma}^n$, respectively. The PCA encoder and its empirical counterpart are given as 
\begin{align}
	E_{\cX}(u)=\left[ \langle u,\phi_{\gamma,1}\rangle, \langle u,\phi_{\gamma,2}\rangle, ..., \langle u,\phi_{\gamma,d_{\cX}}\rangle\right]^{\top}, \ D_{\cX}(\ab)=\sum_{j=1}^d a_j\phi_{\gamma,j}, \label{eq.PCAX}\\
	E_{\cX}^n(u)=\left[ \langle u,\phi^n_{\gamma,1}\rangle, \langle u,\phi^n_{\gamma,2}\rangle, ..., \langle u,\phi^n_{\gamma,d_{\cX}}\rangle\right]^{\top}, \ D^n_{\cX}(\ab)=\sum_{j=1}^d a_j\phi^n_{\gamma,j} \label{eq.PCAX.n}
\end{align}
for any $u\in \cX$ and $\ab\in \RR^{d_{\cX}}$.

For the $\cY$ space, the ideal covariance operator in the noiseless case is defined based on the pushforward measure $\Psi_{\#}\gamma$. In the noisy case, the samples $\{v_i\}_{i=1}^n$ are random copies of $\Psi(u)+\widetilde{\epsilon}$. Denote the probability measure of $v$ by $\zeta$. The ideal and empirical covariance operators are defined as
$$
G_{\Psi_{\#}\gamma}=\EE_{w\sim \Psi_{\#}\gamma} w\otimes w \quad \mbox{ and } \quad G^n_{\zeta}=\frac{1}{n}\sum_{i=1}^n v_i\otimes v_i.
$$
Notice that $G^n_{\zeta}$ is the empirical counterpart of $G_{\zeta}$, which is different from $G_{\Psi_{\#}\gamma}$ in the noisy case.

Let  $\{\phi_{\Psi_{\#}\gamma,k}\}_{k=1}^{d_{\cY}}$ and $\{\phi^n_{\zeta,k}\}_{k=1}^{d_{\cY}}$ be the first $d_{\cY}$ principle components of $G_{\Psi_{\#}\gamma}$ and $G_{\zeta}^n$, respectively. We choose  the PCA encoder:
\begin{align}
	&E_{\cY}(w)=\left[ \langle w,\phi_{\Psi_{\#}\gamma,1}\rangle, \langle w,\phi_{\Psi_{\#}\gamma,2}\rangle, ..., \langle w,\phi_{\Psi_{\#}\gamma,d_{\cX}}\rangle\right]^{\top}, \ D_{\cX}(\ab)=\sum_{j=1}^d a_j\phi_{\Psi_{\#}\gamma,j}, \label{eq.PCAY}\\
	&E_{\cY}^n(w)=\left[ \langle w,\phi^n_{\zeta,1}\rangle, \langle u,\phi^n_{\zeta,2}\rangle, ..., \langle u,\phi^n_{\zeta,d_{\cX}}\rangle\right]^{\top}, \ D^n_{\cY}(\ab)=\sum_{j=1}^d a_j\phi^n_{\zeta,j} \label{eq.PCAY.n}
\end{align}
for any $w\in \cY$ and $\ab\in \RR^{d_{\cY}}$.

The following theorem gives a bound on the generalization error of operator estimation with PCA encoders:
\begin{theorem}\label{thm.pca}
	In Setting \ref{setting}, suppose Assumption \ref{assum.rho}--\ref{assum.Psi.L} and \ref{assum.noise} hold. Consider the PCA encoders and decoders defined in (\ref{eq.PCAX})--(\ref{eq.PCAY.n}). 
	Let $\{\lambda_k\}_{k=1}^{\infty}$ be the eigenvalues of the covariance operator $G_{\Psi_{\#}\gamma}$ in nonincreasing order.
	 Let $\Gamma_{\rm NN}$ be the minimizer of (\ref{eq.PsiNN}) with the network architecture $\cF(d_{\cY},L,p,M)$ in (\ref{eq.FNN.dense}), where $L,p,M$ are set as in (\ref{eq.NN.parameter.n.dense.deter}). We have
	\begin{align*}
		&\EE_{\cS}\EE_{u\sim \gamma}\|D_{\cY}^n\circ \Gamma_{\rm NN} \circ E_{\cX}^n(u)- \Psi(u)\|_{\cY}^2 \\
		\leq& C_{4}(\sigma^2+R_{\cY}^2)d_{\cY}^{\frac{4+d_{\cX}}{2+d_{\cX}}}n^{-\frac{2}{2+d_{\cX}}}\log^6 n +8\left(4R^2_{\cX}L_{\Psi}^2\sqrt{d_{\cX}}+(R_{\cY}+\widetilde{\sigma})^2\sqrt{d_{\cY}}\right)n^{-\frac{1}{2}} +16\widetilde{\sigma}^2\left(\frac{\widetilde{\sigma}}{\lambda_{d_{\cY}}-\lambda_{d_{\cY+1}}}\right)^2(R_{\cY} +\widetilde{\sigma})^2\nonumber\\ 
		&+ 20\widetilde{\sigma}^2 + 16L_{\Psi}^2\EE_{u\sim \gamma}\| \Pi_{\cX,d_{\cX}}(u)-u\|_{2}^2+16\EE_{w\sim \Psi_{\#}\gamma}\|\Pi_{\cY,d_{\cY}}(w)- w\|_{\cY}^2
	\end{align*}
where $C_{4}$ is a constant depending on $d_{\cX},R_{\cX}, R_{\cY},L_{\Psi}$.
\end{theorem}
Theorem \ref{thm.pca} is proved in Section \ref{thm.pca.proof}. PCA is effective when the input and output samples are concentrated near  low-dimensional subspaces. In this case, an orthonormal basis of the subspace is estimated from the samples. Since the PCA encoder and decoder are data-driven, we expect the corresponding projection errors are smaller than those by Legendre polynomials or trigonometric functions. 

In the generalization error in Theorem \ref{thm.pca}, the error $16\widetilde{\sigma}^2\left(\frac{\widetilde{\sigma}}{\lambda_{d_{\cY}}-\lambda_{d_{\cY+1}}}\right)^2(R_{\cY} +\widetilde{\sigma})^2+ 20\widetilde{\sigma}^2$ does not decay as $n$ increases. 
This is because PCA extracts the principal components from noisy data but does not denoise the data set without additional assumptions on noise. If the noise does not perturb the space spanned by the first $d_{\cY}$ principal eigenfunctions of $G_{\Psi_{\#}\gamma}$, the constant terms can be dropped as the following corollary.
\begin{corollary}\label{coro.PCA.noPerturb}
	Under the conditions of Theorem \ref{thm.pca}, if the eigenspace spanned by the first $d_{\cY}$ principal eigenfunctions of $G_{\mu}$ coincides with that of $G_{\Psi_{\#}\gamma}$, then we have
	\begin{align*}
		&\EE_{\cS}\EE_{u\sim \gamma}\|D_{\cY}^n\circ \Gamma_{\rm NN} \circ E_{\cX}^n(u)- \Psi(u)\|_{\cY}^2 \\
		\leq& C_{4}(\sigma^2+R_{\cY}^2)d_{\cY}^{\frac{4+d_{\cX}}{2+d_{\cX}}}n^{-\frac{2}{2+d_{\cX}}}\log^6 n +8\left(4R^2_{\cX}L_{\Psi}^2\sqrt{d_{\cX}}+(R_{\cY}+\widetilde{\sigma})^2\sqrt{d_{\cY}}\right)n^{-\frac{1}{2}}\nonumber\\ 
		&+ 16L_{\Psi}^2\EE_{u\sim \gamma}\| \Pi_{\cX,d_{\cX}}(u)-u\|_{2}^2+16\EE_{w\sim \Psi_{\#}\gamma}\|\Pi_{\cY,d_{\cY}}(w)- w\|_{\cY}^2.
	\end{align*}
\end{corollary}
Corollary \ref{coro.PCA.noPerturb} is proved in Section \ref{coro.PCA.noPerturb.proof}.


\section{Exploit additional low-dimensional structures}\label{sec.lowD}

Section \ref{sec.deter} and Section \ref{sec.PCA} are suitable for the case where the input and output samples are concentrated near a low-dimensional subspace. While in practice, the low-dimensional subspace is not a priori known. In order to capture such a subspace, we need to choose a large encoding dimension so that the low-dimensional subspace is enclosed by the encoded space, which guarantees a small projection error. 
However, the network estimation error (see Section \ref{sec.mainresults.error} for the definition) has an exponential dependence on $d_{\cX}$. The error decays slowly when $d_{\cX}$ is large. 

Additionally, the given data may be located on a low-dimensional manifold enclosed by the encoded space, or the operator $\Psi$ may have low complexity. In this section, we will exploit such additional low-dimensional structures. We will show that, even though $d_{\cX}$ and $d_{\cY}$ are chosen to be large in order to guarantee small projection errors, the exponent in the network estimation error only depends on the intrinsic dimension of the additional low-dimensional structures of data, instead of $d_{\cX}$. Specifically, we consider two scenarios : (1) when the collection of encoded vectors $E_{\cX}(\Omega_{\cX})$ is on a low-dimensional manifold and (2) when the operator $\Psi$ only depends on a few directions in the encoded space. 

\subsection{When encoded vectors lie on a low-dimensional manifold}

We first consider the case when the given data exhibit a  nonlinear low-dimensional structure: For a given encoder $E_{\cX}: \cX \rightarrow \RR^{d_{\cX}}$, the encoded vectors $\{E_{\cX}(u): u \text{ is randomly sampled from } \gamma\}$ lie on a $d_0$-dimensional manifold with $d_0 \ll d_{\cX}$. This scenario is observed in many applications. For example, the solutions of most PDEs are in an infinite-dimensional function space. After uniform discretization, the solutions are encoded to vectors in a very high dimensional space. For many PDEs, it is commonly observed that the solutions actually lie on a low-dimensional manifold enclosed by the discretized high-dimensional space. Therefore the solution manifold can be well-approximated using much fewer bases than those used in the discretization. This observation leads to the success of the reduced basis method \citep{haasdonk2017reduced,rozza2014fundamentals}. Another concrete example is described as follows:

 \begin{wrapfigure}{r}{0.5\textwidth}
 \centering
\includegraphics[width=0.4\textwidth]{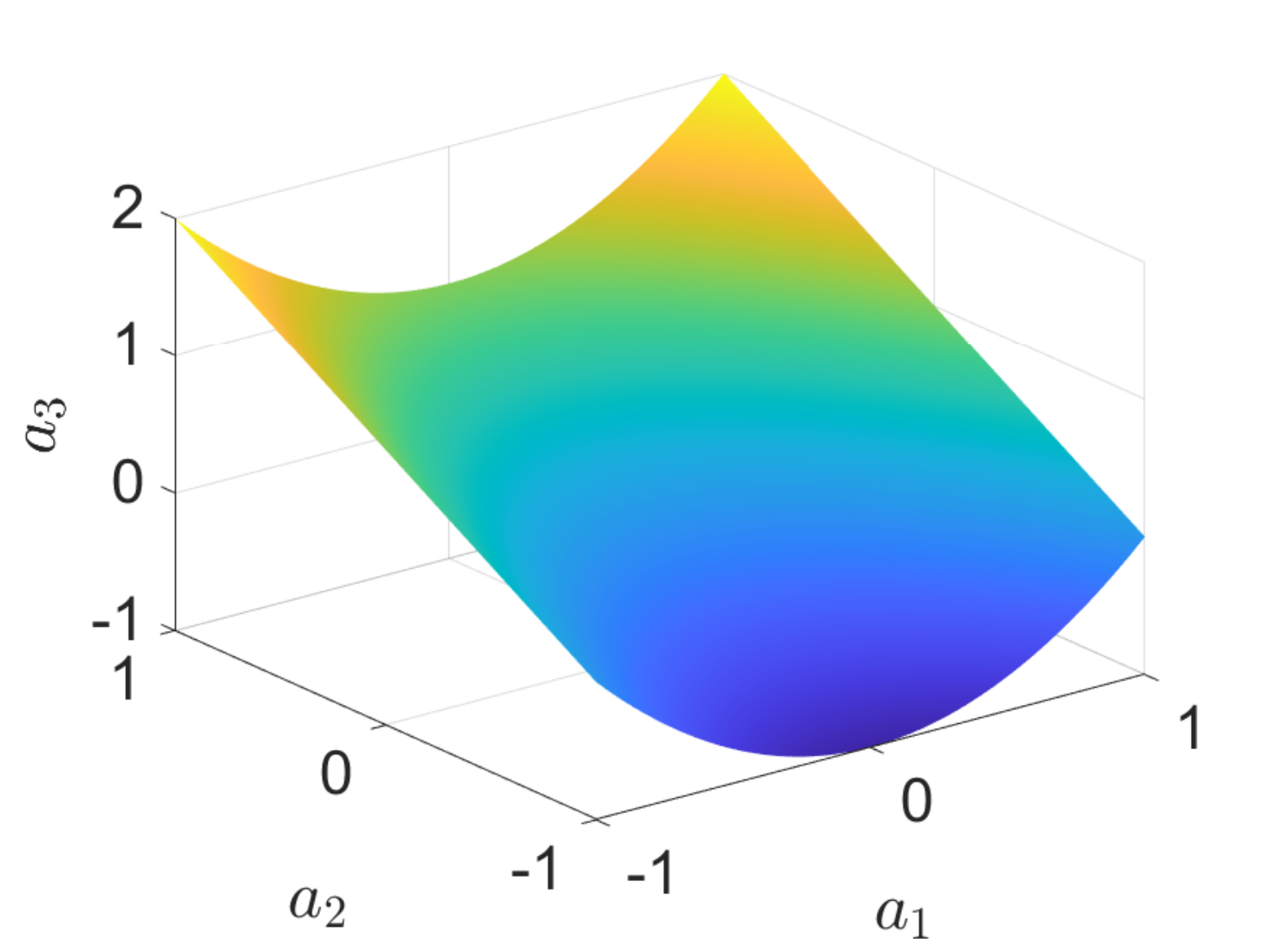}  
\caption{An illustration of Example \ref{example.1} with $d_{\cX}=3, d_0=2$ and $g_3=a_1^2+a_2$. }
			\label{fig.ex.manifold}
			\end{wrapfigure}

\begin{example}\label{example.1}
	Let $\cX=L^2([-1,1])$ and $d_0, d_{\cX}$ be positive integers such that $d_0<d_{\cX}$. Let $\left\{T_k\right\}_{k=1}^{\infty}$ be the trigonometric functions defined in (\ref{eq.trig}) and $\left\{g_k\right\}_{k=d_0+1}^{d_{\cX}}$ be some real valued functions. Suppose the probability measure $\gamma$ is supported on 
	$$
	\Omega_{\cX}=\left\{u: u=\sum_{k=1}^{d_{\cX}} a_kT_k \mbox{ with } a_k\in \RR \mbox{ for } k=1,\ldots, d_0, \mbox{ and } a_k=g_k(a_1,...,a_{d_0}) \mbox{ for } k=d_0+1,\ldots,d_{\cX} \right\}.
	$$
	The support set $\Omega_{\cX}$ has an intrinsic dimension $d_0$. If we choose the basis encoder $E_{\cX}: \cX \rightarrow \RR^{d_{\cX}}$ using the trigonometric functions $\{T_k\}_{k=1}^{d_{\cX}}$, then the encoded vectors $\{E_{\cX}(u): u \text{ is randomly sampled from} \gamma\}$ lie on a $d_0$-dimensional manifold embedded in $\RR^{d_{\cX}}$. Figure \ref{fig.ex.manifold} shows this manifold when $d_{\cX}=3, d_0=2$ and $g_3=a_1^2+a_2$.
	\end{example}

This nonlinear low-dimensional structure of data can be described as follows:

\begin{assumption}\label{assum.M}  Let $d_0, d_{\cX}$ be positive integers such that $d_0<d_{\cX}$. In Setting \ref{setting}, there exists an encoder $E_{\cX}: \cX \rightarrow \RR^{d_{\cX}}$ such that the encoded vectors $\{E_{\cX}(u): u \text{ is randomly sampled from } \gamma\}$ is on a $d_0$-dimensional compact smooth Riemannian manifold $\cM$ isometrically embedded in $\RR^{d_{\cX}}$. The reach of $\cM$ \citep{federer1959curvature,niyogi2008finding} is $\tau>0$.	 
\end{assumption}

Under Assumption \ref{assum.M} and Setting \ref{setting}, the output $\Psi(u)$ is perturbed by noise, while the input $u$ is clean and its encoded vector is located on $\cM$. Such a setting is common in practice when a series of experiments is conducted to simulate a scientific phenomenon. In experiments, one designs the inputs and takes measurements of the outputs. Usually, the inputs are generated according to some physical laws that lead to low-dimensional structures. Due to the limitations of sensors and equipment, the measured outputs are perturbed by noise.


Approximation and statistical estimation theories of deep neural networks for functions on a low-dimensional manifold have been studied in \citep{chen2019efficient,chen2019nonparametric,chen2020doubly,hao2021icml,shijun2,Jiao2021DeepNR,cloninger2020relu,shaham2018provable,schmidt2019deep,du2021discovery,nakada2020adaptive}. 
In this subsection, we show that deep neural networks can automatically adapt to nonlinear low-dimensional structures of data, and give rise to a sample complexity depending on the intrinsic dimension  $d_0$. 
The following theorem gives a generalization error in this scenario.

\begin{theorem}\label{thm.general.M}
	In Setting \ref{setting},
	suppose Assumption \ref{assum.rho}--\ref{assum.noiseEncoder} and \ref{assum.M} hold, and the encoder $E_{\cX}$ in Assumption \ref{assum.M} is given. 
	Let $\Gamma_{\rm NN}$ be the minimizer of (\ref{eq.PsiNN}) with the network architecture $\cF(d_{\cY},L,p,M)$ in (\ref{eq.FNN.dense}) with
	\begin{equation}
		\begin{aligned}
			&L=O(\widetilde{L}\log \widetilde{L}), \ p=O\left(d_{\cX}\widetilde{p} \log \widetilde{p}\right), M=\sqrt{d_{\cY}}L_{E^n_{\cY}}R_{\cY},
		\end{aligned}\label{eq.NN.parameter.n.M.dense}
	\end{equation}
	where $\widetilde{L},\widetilde{p}>0$ are positive integers satisfying 
	\begin{equation}
		\widetilde{L}\widetilde{p}=\left\lceil d_{\cY}^{-\frac{d_0}{4+2d_0}}n^{\frac{d_0}{4+2d_0}}\right\rceil.
		\label{eq.FNN.densecondition.M}
	\end{equation}
	Then we have
	\begin{align}
		&\EE_{\cS}\EE_{u\sim \gamma}\|D_{\cY}^n\circ \Gamma_{\rm NN} \circ E^n_{\cX}(u)- \Psi(u)\|_{\cY}^2 \nonumber\\
		\leq&\ C_{5}(\sigma^2+R_{\cY}^2)d_{\cY}^{\frac{4+d_0}{2+d_0}}d_{\cX}^2n^{-\frac{2}{2+d_0}}\log^6 n   +  C_3\EE_{\cS}\EE_{u\sim \gamma}\| \Pi_{\cX,d_{\cX}}(u)- u\|_{\cX}^2+2\EE_{\cS}\EE_{w\sim \Psi_{\#}\gamma}\|\Pi_{\cY,d_{\cY}}^n(w)- w\|_{\cY}^2,
		\label{eq.error.general.M.dense}
	\end{align}
	where $C_5$ dpends on $d_0,\log d_{\cX}, R_{\cX},  R_{\cY},L_{E^n_{\cX}},L_{E^n_{\cY}},L_{D_{\cX}},L_{D^n_{\cY}},L_{\Psi},\tau$, the surface area of $\cM$, and $C_3=16L_{D^n_{\cY}}^2L_{E^n_{\cY}}^2L_{\Psi}^2$.

\end{theorem}
Theorem \ref{thm.general.M} is proved in Section \ref{thm.general.M.proof}. The convergence rate in Theorem \ref{thm.general.M} has an exponential dependence on $d_0$, instead of $d_{\cX}$. Theorem \ref{thm.general.M} shows that when the encoded vectors are located on a low-dimensional manifold, deep neural networks are adaptive to such nonlinear geometric structures of data. 

\subsection{When the operator $\Psi$ has low complexity }
In our framework, learning $\Psi$ is converted to learning the transformation $\Gamma: \RR^{d_{\cX}}\rightarrow\RR^{d_{\cY}}$, as defined in (\ref{eq.Gamma}). 
The second scenario we consider in this subsection is that, even though the $u_i$'s and $v_i$'s are in infinite-dimensional spaces, the operator $\Psi$ has low complexity: its corresponding transformation $\Gamma$ can be approximated by some low-dimensional functions that only depend on few directions in $\RR^{d_{\cX}}$.
For example, consider solving a linear PDE with constant coefficients by the Fourier spectral method. In this case, the operator $\Psi$ is the PDE solver that maps initial conditions to solutions at certain time. By taking the Fourier transform on both sides of the PDE, solving the PDEs is converted to solving a series of independent ODEs, each of which controls the evolution of a Fourier coefficient of the solution \citep[Chapter 2]{shen2011spectral}. The operator $\Psi$  can be fully characterized by a system of one-dimensional ODEs. We next adapt this setting to our framework in order to learn $\Psi$. We use trigonometric functions as our encoders and decoders: the initial conditions and solutions are approximated by the first $d_{\cX}=d_{\cY}$ terms of their Fourier series expansion. Then learning $\Psi$ reduces to learning $d_{\cY}$ one-dimensional functions, each of which corresponds to an ODE of a Fourier coefficient,  instead of learning $d_{\cY}$ $d_{\cX}$-dimensional functions. 

In this subsection, we show that we can get a faster rate by exploiting the low complexity of $\Psi$. We first make an assumption on $\Psi$:
\begin{assumption}\label{assum.Psi.lowCom}
	Let $0<d_0\leq d_{\cX}$ be integers. Assume there exist $E_{\cX},D_{\cX},E_{\cY},D_{\cY}$ such that for any $u\in \Omega_{\cX}$, we have
	\begin{align}
		\Pi_{\cY, d_{\cY}}\circ\Psi(u)=D_{\cY}\circ \gb\circ E_{\cX}(u)
	\end{align}
	with $\gb: \RR^{d_{\cX}} \rightarrow \RR^{d_{\cX}}$ in the form: \begin{equation}\gb(\ab)=\begin{bmatrix}
		g_1(V_1^{\top} \ab) & \cdots & g_{d_{\cY}}(V_{d_{\cY}}^{\top} \ab)
	\end{bmatrix}^{\top},
	\end{equation}
	for some unknown matrix $V_k\in \RR^{d_{\cX}\times d_0}$, and some unknown real valued function $g_k: \RR^{d_0}\rightarrow \RR$ where $k=1,...,d_{\cY}$.
\end{assumption}


In statistics, the functions $g_k$'s in Assumption \ref{assum.Psi.lowCom} are known as single-index models for $d_0=1$, and are known as multi-index models for $d_0>1$. 
For any given $u\sim \gamma$, we decompose $\Psi(u)$ into two parts: the first part is its projection to the set of encoded vectors $E_{\cY}(\Omega_{\cY})$; the second part is the rest orthogonal to the first part. Assumption \ref{assum.Psi.lowCom} assumes that the operator mapping $u$ to the first part follows a multi-index model. When $d_{\cX}$ is large enough, the second part has a small magnitude and is included in the projection error.
In the following example, we give a simple illustration when the second part vanishes.

\begin{example}\label{example.2}
	Let $\cX=L^2([-1,1]),\ \Omega_{\cX}\subset \cX$ be a compact set in $\cX$ and $0<d_0<d_{\cX}$ be integers. Let $\left\{T_k\right\}_{k=1}^{\infty}$ be trigonometric functions defined in (\ref{eq.trig}). Any $u\in \Omega_{\cX}$ can be written as $u=\sum_{k=1}^{\infty} a_kT_k$ for some $a_k$'s. Denote $\ab_u=\begin{bmatrix}
		a_1 & \cdots & a_{d_{\cX}}
	\end{bmatrix}^{\top}.$ Suppose the operator we want to learn has the following form
	\begin{align}
		\Psi(u)=\sum_{k=1}^{d_{\cY}} g_k(V_k^{\top}\ab_u)T_k,
		\label{eq.example.lowComplex}
	\end{align}
	with $V_k\in \RR^{d_{\cX}\times d_0 }$ and $g_k: \RR^{d_0}\rightarrow \RR$ for $k=1,...,d_{\cY}$. We set $E_{\cX},D_{\cX}$ as the basis encoder and decoder using the basis functions $\{T_k\}_{k=1}^{d_{\cX}}$, and  $E_{\cY},D_{\cY}$ as encoder and decoder derived using basis $\{T_k\}_{k=1}^{d_{\cY}}$. In this example, $\Pi_{\cY,d_{\cY}}\circ\Psi(u)=\Psi(u)$ for any $u\sim \gamma$. Then learning $\Psi$ reduces to learning the $g_k$'s and the $V_k$'s. An illustration of the estimator is shown in Figure \ref{fig.ex.lowCom}. In neural networks, the $V_k's$ can be realized by a single layer. Therefore, our major task is to learn good  approximations of the $g_k's$. Note that each $g_k$ is a $d_0$-dimensional function. By exploiting such low complexity of the operator, we can convert the learning task from learning $d_{\cY}$ $d_{\cX}$-dimensional functions to learning $d_{\cY}$ $d_0$-dimensional functions.  
	\begin{figure}[ht!]
		\centering
		\includegraphics[width=0.8\textwidth]{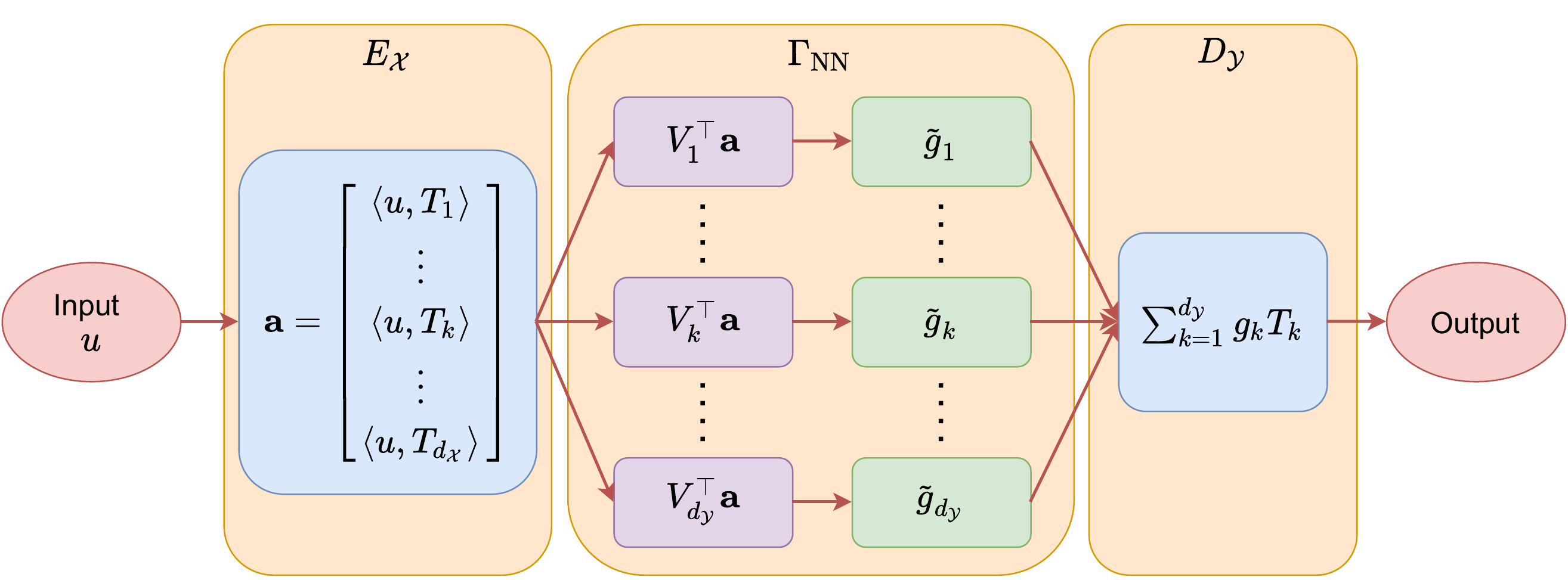}
		\caption{An illustration of Example \ref{example.2}, where the $\widetilde{g}_k$'s represent network approximations of the $g_k$'s in (\ref{eq.example.lowComplex}). 
		}
		\label{fig.ex.lowCom}
	\end{figure}
	
	
\end{example}


 With Assumption \ref{assum.Psi.lowCom}, the following theorem gives a faster rate on the generalization error:
\begin{theorem}\label{thm.general.lowCom}
	In Setting \ref{setting},
	suppose Assumption \ref{assum.rho}--\ref{assum.noiseEncoder} and \ref{assum.Psi.lowCom} hold. Assume that the encoders and decoders $E_{\cX},D_{\cX},E_{\cY}, D_{\cY}$ in Assumption \ref{assum.Psi.lowCom} are given. 
	Let $\Gamma_{\rm NN}$ be the minimizer of (\ref{eq.PsiNN}) with the network architecture $\cF(d_{\cY},L,p,M)$ in (\ref{eq.FNN.dense}), where
	\begin{equation}
		\begin{aligned}
			&L=O(\widetilde{L}\log \widetilde{L}), \ p=O\left(\widetilde{p} \log \widetilde{p}\right), M=\sqrt{d_{\cY}}L_{E^n_{\cY}}R_{\cY}
		\end{aligned}\label{eq.NN.parameter.n.dense.lowCom}
	\end{equation}
	and $\widetilde{L},\widetilde{p}>0$ are integers and satisfy (\ref{eq.FNN.densecondition.M}). 
	
	We have
	\begin{align}
		&\EE_{\cS}\EE_{u\sim \gamma}\|D^n_{\cY}\circ \Gamma_{\rm NN} \circ E^n_{\cX}(u)- \Psi(u)\|_{\cY}^2 \nonumber\\
		\leq& C_{6}(\sigma^2+R_{\cY}^2)d_{\cY}^{\frac{4+d_0}{2+d_0}}\max\left\{n^{-\frac{2}{2+d_0}},d_{\cX}n^{-\frac{4+d_0}{4+2d_0}}\right\}\log^6 n \nonumber\\
		& +  C_3\EE_{\cS}\EE_{u\sim \gamma}\| \Pi_{\cX,d_{\cX}}(u)- u\|_{\cX}^2+2\EE_{\cS}\EE_{w\sim \Psi_{\#}\gamma}\|\Pi_{\cY,d_{\cY}}(w)- w\|_{\cY}^2.
		\label{eq.error.general.dense.lowCom}
	\end{align}
	where $C_{6}$ depends on $d_0,\log d_{\cX}, R_{\cX},  R_{\cY},L_{E^n_{\cX}},L_{E^n_{\cY}},L_{D^n_{\cX}},L_{D^n_{\cY}},L_{\Psi}$, and $C_3=16L_{D^n_{\cY}}^2L_{E^n_{\cY}}^2L_{\Psi}^2$.

\end{theorem}
Theorem \ref{thm.general.lowCom} is proved in Section \ref{thm.general.lowCom.proof}. In Assumption \ref{assum.Psi.lowCom}, each $V_k$ is a linear transformation that can be realized by a singly layer. In our network construction, the first layer is used to learn these transformations and the rest is used to learn the functions $g_k$'s.

\section{Proof of main results}\label{sec.proof}
\subsection{Preliminaries}
In this section, we define several quantities that will be used in the proof. We first define two types of covering number of function classes. The first type is independent of data and will be used to prove Theorem \ref{thm.general}.
\begin{definition}[Cover]\label{def.cover}
	Let $\cF$ be a class of functions. A set of functions $\cS$ is a $\delta$-cover of $\cF$ with respect to a norm $\|\cdot \|$ if for any $f\in F$, one has
	$$
	\inf_{f^*\in \cS}\|f-f^*\|\leq \delta.
	$$
\end{definition}
\begin{definition}[Covering number, Definition 2.1.5 of \citep{van1996weak}]\label{def.covernumber}
	Let $F$ be a class of functions. For any $\delta>0$, the covering number of $\cF$ is defined as
	$$
	\cN(\delta,\cF,\|\cdot\|)=\min\{ |\cS_f| : \text{$\cS_f$ is a $\delta$-cover of $\cF$} \},
	$$
	where $|\cS_f|$ denotes the cardinality of $\cS_f$.
\end{definition}

Definition \ref{def.cover} and \ref{def.covernumber} depend on the norm $\|\cdot\|$. In the following, we choose $\|\cdot\|$ as a sample dependent norm and define the so-called uniform covering number. We first define the cover with respect to samples:
\begin{definition}[Cover with respect to samples]\label{def.conver.sample}
	Let $\cF$ be a class of functions from $\RR^{d_1}$ to $\RR^{d_2}$. Given a set of samples $X=\{\xb_k\}_{k=1}^m\subset \RR^{d_1}$, for any $\delta>0$, a function set $\cS_f(X)$ is a $\delta$-cover of $F$ with respect to $X$ if for any $f\in \cF$, there exists $f^*\in \cS_f(X)$ such that
	$$
	\|f(\xb_k)-f^*(\xb_k)\|_{\infty}\leq \delta, \quad \forall 1\leq k\leq m.
	$$
\end{definition}
Definition \ref{def.conver.sample} is a special case of Definition \ref{def.cover} in which the norm $\|\cdot\|$ is chosen as the $\ell^{\infty}$ norm of the collection of its argument's values over samples $X$. Based on Definition \ref{def.conver.sample}, we define the uniform covering number as follows:
\begin{definition}[Uniform covering number, Section 10.2 of \citep{anthony1999neural}]\label{def.covernumber.n}
	Let $\cF$ be a class of functions from $\RR^d$ to $\RR$. For any set of samples $X=\{\xb_k\}_{k=1}^m\subset \RR^d$, denote
	$$
	\cF|_{X}=\left\{\left(f(\xb_1),...,f(\xb_m)\right): f\in \cF\right\}.
	$$
	For any $\delta>0$, the uniform covering number of $\cF$ with $m$ samples is defined as
	\begin{align}
		\cN(\delta,\cF,m)=\max_{X\subset \RR^d,|X|=m}\min_{\cS_f(X)}\{|\cS_f(X)|:\cS_f(X) \mbox{ is a $\delta$-cover of $\cF$ with respect to $X$}\}.
	\end{align}
\end{definition}
This covering number is used to prove Theorem \ref{thm.general.dense}.
\subsection{Proof of Theorem \ref{thm.general}}
\label{thm.general.proof}

To prove Theorem \ref{thm.general}, we first decompose the squared $L^2$ error $\EE_{\cS}\EE_{u\sim \gamma}\|D_{\cY}^n\circ \Gamma_{\rm NN} \circ E_{\cX}^n(u)- \Psi(u)\|_{\cY}^2$ into a network estimation error and a projection error. The network estimation error can be further decomposed into a bias term and a variance term. The bias term heavily depends on the approximation error of the network class (\ref{eq.FNN}). The variance term is upper bounded in terms of the covering number of the network class. 

\begin{proof}[Proof of Theorem \ref{thm.general}]
We first decompose the squared $L^2$ error as
\begin{align}
	&\EE_{\cS}\EE_{u\sim \gamma}\left[\left\|D_{\cY}^n\circ \Gamma_{\rm NN} \circ E_{\cX}^n(u)- \Psi(u)\right\|_{\cY}^2\right] \nonumber\\
	\leq&  \underbrace{2\EE_{\cS}\EE_{u\sim \gamma}\left[\left\|D_{\cY}^n\circ \Gamma_{\rm NN} \circ E_{\cX}^n(u)-D_{\cY}^n\circ E_{\cY}^n\circ \Psi(u)\right\|_{\cY}^2\right]}_{\rm I}+\underbrace{2\EE_{\cS}\EE_{u\sim \gamma}\left[\left\|D_{\cY}^n\circ E_{\cY}^n\circ \Psi(u)- \Psi(u)\right\|_{\cY}^2\right]}_{\rm II}.
	\label{eq.err.decom}
\end{align}
Here I is the network estimation error in the $\cY$ space, II is the empirical projection error, which can be rewritten as
\begin{align}
	{\rm II}=2\EE_{\cS}\EE_{w\sim \Psi_{\#}\gamma} \left[\left\|\Pi_{\cY,d_{\cY}}^n(w)- w\right\|_{\cY}^2\right].
	\label{eq.II}
\end{align}

In the remaining of this subsection, we derive an upper bound of I. Note that I can be bounded as
\begin{align}
	{\rm I}=& 2\EE_{\cS}\EE_{u\sim \gamma}\left[\left\|D_{\cY}^n\circ \Gamma_{\rm NN} \circ E_{\cX}^n(u)-D_{\cY}^n\circ E_{\cY}^n\circ \Psi(u)\right\|_{\cY}^2\right] \nonumber\\
	\leq & 2L_{D^n_{\cY}}^2\EE_{\cS}\EE_{u\sim \gamma}\left[\left\| \Gamma_{\rm NN} \circ E_{\cX}^n(u)- E_{\cY}^n\circ \Psi(u)\right\|_{2}^2\right] .
	\label{eq.I}
\end{align}
If the training samples in $\cS_1$ are fixed, we have the following conditioned on $\cS_1$:
\begin{align}
	&\EE_{\cS_2}\EE_{u\sim \gamma}\left[\left\| \Gamma_{\rm NN} \circ E_{\cX}^n(u)- E_{\cY}^n\circ\Psi(u)\right\|_{2}^2\right]\nonumber\\
	=& \underbrace{2\EE_{\cS_2}\left[\frac{1}{n}\sum_{i=n+1}^{2n}\left\| \Gamma_{\rm NN} \circ E_{\cX}^n(u_i)-E_{\cY}^n\circ\Psi(u_i)\right\|_{2}^2\right]}_{\rm T_1}\nonumber\\
	&+\underbrace{\EE_{\cS_2}\EE_{u\sim \gamma}\left[\left\| \Gamma_{\rm NN} \circ E_{\cX}^n(u)- E_{\cY}^n\circ\Psi(u)\right\|_{2}^2\right]
	- \EE_{\cS_2}\left[\frac{2}{n}\sum_{i=n+1}^{2n}\left\| \Gamma_{\rm NN} \circ E_{\cX}^n(u_i)- E_{\cY}^n\circ\Psi(u_i)\right\|_{2}^2\right]}_{\rm T_2}.
\label{eq.I.integral1}
\end{align}
In the decomposition of \eqref{eq.I.integral1}, the term ${\rm T_1}$ consists of the bias of using neural network to approximate the transformation $\Gamma$ and the projection error of $\Pi_{\cX,d_{\cX}}^n$ in the $\cX$ space. The term ${\rm T_2}$ captures the variance. We next derive bounds for ${\rm T_1}$ and ${\rm T_2}$ respectively.

\paragraph{Upper bound of ${\rm T_1}$.}
The term ${\rm T_1}$ is the expected mean squared error of the learned transformation $\Gamma_{\rm NN}$ with respect to $\cS_2$. We will derive an upper bound using the network approximation error and network architecture's covering number. The network approximation error is the bias. We use network architecture's covering number to bound the stochastic error. 

Define the transformation $\Gamma_{d}^n: \RR^{d_{\cX}}\rightarrow \RR^{d_{\cY}}$
\begin{align}
	\Gamma_d^n=E_{\cY}^n\circ \Psi\circ D_{\cX}^n,
	\label{eq.pid}
\end{align}
which maps the encoded vector  $E^n_{\cX}(u)$ in $\cX$ to the encoded vector  $E^n_{\cY}(v)$ in $\cY$.
The transformation $\Gamma_d^n$ is the target transformation to be estimated by $\Gamma_{\rm NN}$. 
It is straightforward to show that $\Gamma_d^n$ is a Lipschitz transformation (see a proof of Lemma \ref{lem.GammaLip} in Appendix \ref{lem.GammaLip.proof}).
\begin{lemma}\label{lem.GammaLip}
	Assume Assumption \ref{assum.Psi.L} and \ref{assum.ED.lip}. $\Gamma_d^n$ is Lipschitz with a Lipschitz constant $L_{E^n_{\cY}}L_{D^n_{\cX}}L_{\Psi}$.
\end{lemma}

Denote
\begin{align}
	\bepsilon_i=E^n_{\cY}(v_i)-E^n_{\cY}(\Psi(u_i)).
	\label{eq.epsilon}
\end{align}
According to Assumption \ref{assum.ED.lip} and Assumption \ref{assum.noise}(iii)--(iv), we have $$\EE[\bepsilon_i]=\mathbf{0}, \text{ and }\|\bepsilon_i\|_{\infty}<\sigma.$$
We decompose ${\rm T_1}$ as
 \begin{align}
 	{\rm T_1}=&2\EE_{\cS_2}\left[\frac{1}{n}\sum_{i=n+1}^{2n}\left\| \Gamma_{\rm NN} \circ E_{\cX}^n(u_i)- E_{\cY}^n\circ\Psi(u_i)\right\|_{2}^2\right] \nonumber\\
 	=&2\EE_{\cS_2}\left[\frac{1}{n}\sum_{i=n+1}^{2n}\left\| \Gamma_{\rm NN} \circ E_{\cX}^n(u_i)- E_{\cY}^n\circ\Psi(u_i)-\bepsilon_i+\bepsilon_i\right\|_{2}^2\right] \nonumber\\
 	=&2\EE_{\cS_2}\left[\frac{1}{n}\sum_{i=n+1}^{2n}\left\| \Gamma_{\rm NN} \circ E_{\cX}^n(u_i)- E_{\cY}^n\circ\Psi(u_i)-\bepsilon_i\right\|_2^2\right]\nonumber \\
 	&\quad +4\EE_{\cS_2}\left[\frac{1}{n}\sum_{i=n+1}^{2n}\left\langle \Gamma_{\rm NN} \circ E_{\cX}^n(u_i)- E_{\cY}^n\circ\Psi(u_i)-\bepsilon_i, \bepsilon_i\right\rangle\right]+ 2\EE_{\cS_2}\left[\frac{1}{n}\sum_{i=n+1}^{2n}\|\bepsilon_i\|_{2}^2\right] \nonumber\\
 	=&2\EE_{\cS_2}\left[\frac{1}{n}\sum_{i=n+1}^{2n}\left\| \Gamma_{\rm NN} \circ E_{\cX}^n(u_i)- E_{\cY}^n(v_i)\right\|_2^2\right]  +4\EE_{\cS_2}\left[\frac{1}{n}\sum_{i=n+1}^{2n}\left\langle \Gamma_{\rm NN} \circ E_{\cX}^n(u_i), \bepsilon_i\right \rangle \right]- 2\EE_{\cS_2}\left[ \frac{1}{n}\sum_{i=n+1}^{2n} \|\bepsilon_i\|_{2}^2 \right] \nonumber\\
 	= &2\EE_{\cS_2}\left[\inf_{\Gamma\in \cF_{\rm NN}}\frac{1}{n}\sum_{i=n+1}^{2n}\left\| \Gamma \circ E_{\cX}^n(u_i)- E_{\cY}^n(v_i)\right\|_{2}^2\right] \nonumber\\
 	&\quad +4\EE_{\cS_2}\left[\frac{1}{n}\sum_{i=n+1}^{2n}\left\langle \Gamma_{\rm NN} \circ E_{\cX}^n(u_i), \bepsilon_i\right\rangle\right]- 2\EE_{\cS_2}\left[\frac{1}{n}\sum_{i=n+1}^{2n}\|\bepsilon_i\|_{2}^2\right] \nonumber \tag*{by the definition of $\Gamma_{\rm NN}$ in (\ref{eq.PsiNN})}\nonumber\\
 	\leq& 2\inf_{\Gamma\in \cF_{\rm NN}}\EE_{\cS_2}\left[\frac{1}{n}\sum_{i=n+1}^{2n}\left\| \Gamma \circ E_{\cX}^n(u_i)- E_{\cY}^n(v_i)\right\|_{2}^2\right]  +4\EE_{\cS_2}\left[\frac{1}{n}\sum_{i=n+1}^{2n}\left\langle \Gamma_{\rm NN} \circ E_{\cX}^n(u_i), \bepsilon_i\right\rangle \right]\nonumber\\
 	&\quad - 2\EE_{\cS_2}\left[\frac{1}{n}\sum_{i=n+1}^{2n}\|\bepsilon_i\|_{2}^2\right] \nonumber \\
 	=&2\inf_{\Gamma\in \cF_{\rm NN}}\EE_{\cS_2}\left[\frac{1}{n}\sum_{i=n+1}^{2n}\left[\| \Gamma \circ E_{\cX}^n(u_i)- E_{\cY}^n\circ\Psi(u_i)-\bepsilon_i\|_{2}^2-\|\bepsilon_i\|_{2}^2\right] \right]\nonumber\\
 	&\quad +4\EE_{\cS_2}\left[\frac{1}{n}\sum_{i=n+1}^{2n}\left\langle \Gamma_{\rm NN} \circ E_{\cX}^n(u_i), \bepsilon_i\right\rangle \right] \nonumber \\
 	=&2\inf_{\Gamma\in \cF_{\rm NN}}\EE_{u\sim \gamma}\left[ \left\| \Gamma \circ E_{\cX}^n(u)- E_{\cY}^n\circ\Psi(u)\right\|_{2}^2 \right]+4\EE_{\cS_2}\left[\frac{1}{n}\sum_{i=n+1}^{2n}\left\langle \Gamma_{\rm NN} \circ E_{\cX}^n(u_i), \bepsilon_i\right\rangle \right].
 	\label{eq.T1.0}
 \end{align}
 In (\ref{eq.T1.0}), the first term is the neural network approximation error, and the second term is the stochastic error from noise. To derive an upper bound of the first term, we use the following lemma which shows that for any function $f$ in the Sobolev space $W^{k,\infty}$,  when the network architecture is properly set, FNN can approximate $f$ with arbitrary accuracy:
 \begin{lemma}[Theorem 1 of \citep{yarotsky2017error}]
 	\label{lem.approx}
 	Let $k\geq 0$ be a positive integer . There exists an FNN architecture $\cF_{\rm NN}(1,L,p,K,\kappa,M)$ capable of approximating any function in $W^{k,\infty}\left([-B,B]^d\right)$, i.e., for any given $\epsilon\in (0,1)$ and if $f\in W^{k,\infty}\left([-B,B]^d\right)$, the network architecture gives rise to a function $\tilde{f}$ satisfying 
 	$$
 	\left\|\widetilde{f}-f\right\|_{\infty}\leq\varepsilon.
 	$$
 	The hyperparameters in $\cF_{\rm NN}$ are chosen as
 	\begin{align*}
 		&L=O\left(\log \frac{1}{\varepsilon}\right), \ p=O\left(\varepsilon^{-\frac{d}{k}}\right),\ K=O\left(\varepsilon^{-\frac{d}{k}}\log \frac{1}{\varepsilon}\right), \ \kappa=\max\left\{1,B,R \right\},\ M=R.
 	\end{align*}
 	The constant hidden in $O(\cdot)$ depends on $k,\alpha,d,B,R$.
 \end{lemma}
 Since $\Gamma_d^n$ is Lipschitz by Lemma \ref{lem.GammaLip}, according to Lemma \ref{lem.approx} with $k=1$, for any $\varepsilon_1>0$, there is a network architecture $\cF_{\rm NN}(d_{\cY},L,p,K,\kappa,M)$, such that for any  $\Gamma_{d}^n$ defined in (\ref{eq.pid}), there exists a $\widetilde{\Gamma}_d^n\in \cF_{\rm NN}(d_{\cY},L,p,K,\kappa,M)$ with
 \begin{align}
 	\left\|\widetilde{\Gamma}_d^n-\Gamma_d^n\right\|_{\infty}\leq \varepsilon_1.
 \end{align}
 Such a network architecture has
 \begin{equation}
 	\begin{aligned}
 		&L=O(\log \varepsilon_1), \ p=O\left(\varepsilon_1^{-d_{\cX}}\right),\ K=O\left(\varepsilon_1^{-d_{\cX}}\log \varepsilon_1\right), \\
 		& \kappa=\max\left\{1,\sqrt{d_{\cY}}L_{E^n_{\cY}}R_{\cY}, \sqrt{d_{\cX}}L_{E^n_{\cX}}R_{\cX},L_{E^n_{\cY}}L_{D^n_{\cX}}L_{\Psi} \right\},\ M=\sqrt{d_{\cY}}L_{E^n_{\cY}}R_{\cY}.
 	\end{aligned}
 \label{eq.T1.archi}
 \end{equation}

 We bound the first term in (\ref{eq.T1.0}) as
 \begin{align}
 	& \inf_{\Gamma\in \cF_{\rm NN}}\EE_{u\sim \gamma}\left[ \left\| \Gamma \circ E_{\cX}^n(u)- E_{\cY}^n\circ
 	\Psi(u)\right\|_{2}^2 \right]\nonumber\\
 	\leq &\EE_{u\sim \gamma}\left[ \left\| \tilde{\Gamma}_d^n \circ E_{\cX}^n(u)- E_{\cY}^n\circ \Psi(u)\right\|_{2}^2\right] \nonumber\\
 	\leq & 2\EE_{u\sim \gamma} \left[ \left\| \tilde{\Gamma}_d^n \circ E_{\cX}^n(u)-\Gamma_d \circ E_{\cX}^n(u)\right\|_2^2\right]+ 2\EE_{u\sim \gamma}\left[ \left\|\Gamma_d^n \circ E_{\cX}^n(u)- E_{\cY}^n\circ \Psi(u)\right\|_{2}^2\right] \nonumber\\
 	\leq& 2d_{\cY}\varepsilon_1^2+ 2\EE_{u\sim \gamma}\left[ \left\|\Gamma_d^n \circ E_{\cX}^n(u)- E_{\cY}^n\circ \Psi(u)\right\|_{2}^2\right] \nonumber\\
 	= &2d_{\cY}\varepsilon_1^2+ 2\EE_{u\sim \gamma}\left[ \left\|E_{\cY}^n\circ\Psi\circ D_{\cX}^n \circ E_{\cX}^n(u)- E_{\cY}^n\circ \Psi(u)\right\|_{2}^2\right] \tag*{by the definition of $\Gamma_d$ in (\ref{eq.pid})}\\
 	\leq &2d_{\cY}\varepsilon_1^2+ 2L_{E^n_{\cY}}^2L_{\Psi}^2\EE_{u\sim \gamma} \left[ \left\| D_{\cX}^n \circ E_{\cX}^n(u)- u\right\|_{\cX}^2\right] \nonumber\\
 	=& 2d_{\cY}\varepsilon_1^2+ 2L_{E^n_{\cY}}^2L_{\Psi}^2\EE_{u\sim \gamma} \left[ \left\| \Pi_{\cX,d_{\cX}}^n(u)- u\right\|_{\cX}^2\right].
 	\label{eq.T1.1}
 \end{align}
 
 An upper bound of the second term in (\ref{eq.T1.0}) is provided by the following lemma (see a proof in Appendix \ref{lem.T1.2.proof}):
 \begin{lemma}\label{lem.T1.2}
 	Under the conditions of Theorem \ref{thm.general}, for any $\delta\in (0,1)$, we have 
 	\begin{align}
 		&\EE_{\cS_2}\left[\frac{1}{n}\sum_{i=n+1}^{2n}\langle \Gamma_{\rm NN} \circ E_{\cX}^n(u_i), \bepsilon_i\rangle\right] \nonumber\\
 		\leq & 2\sqrt{2d_{\cY}}\sigma\left(\sqrt{\EE_{\cS_2}\|\Gamma_{\rm NN} \circ E_{\cX}^n(u_i)-\Gamma_{d}^n \circ E_{\cX}^n(u_i)\|_n^2}+\sqrt{d_{\cY}}\delta \right)\sqrt{\frac{\log \cN(\delta,\cF_{\rm NN}, \|\cdot\|_{\infty}) +2}{n}} +d_{\cY}\sigma\delta.
 		\label{eq.T1.5}
 	\end{align}
 \end{lemma}
 
 Let $\cF_{\rm NN}$ be the network architecture specified in (\ref{eq.T1.archi}). Substituting (\ref{eq.T1.1}) and (\ref{eq.T1.5}) into (\ref{eq.T1.0}), we have
 \begin{align}
 	{\rm T_1}=&2\EE_{\cS_2}\left[\left\| \Gamma_{\rm NN} \circ E_{\cX}^n(u_i)- E_{\cY}^n\circ\Psi(u_i)\right\|_{n}^2\right]\nonumber\\
 	\leq& 4d_{\cY}\varepsilon_1^2+ 8\sqrt{2d_{\cY}}\sigma\left(\sqrt{\EE_{\cS_2}\left\|\Gamma_{\rm NN} \circ E_{\cX}^n(u_i)-\Gamma_{d}^n \circ E_{\cX}^n(u_i)\right\|_n^2}+\sqrt{d_{\cY}}\delta \right)\sqrt{\frac{\log \cN(\delta,\cF_{\rm NN}, \|\cdot\|_{\infty}) +2}{n}} \nonumber\\
 	& +4d_{\cY}\sigma\delta +4L_{E^n_{\cY}}^2L_{\Psi}^2\EE_{u\sim \gamma}\left[ \left\| \Pi_{\cX,d_{\cX}}^n(u)- u\right\|_{\cX}^2\right] .
 	\label{eq.T1.6}
 \end{align}
 
 Denote 
 \begin{align*}
 	&\rho=\sqrt{\EE_{\cS_2}\left[\left\| \Gamma_{\rm NN} \circ E_{\cX}^n(u_i)- E_{\cY}^n\circ\Psi(u_i)\right\|_{n}^2\right]},\\
 	&a=2d_{\cY}\varepsilon_1^2+ 2d_{\cY}\sigma\delta+ 2L_{E^n_{\cY}}^2L_{\Psi}^2\EE_{u\sim \gamma}\left[ \left\| \Pi_{\cX,d_{\cX}}^n(u)- u\right\|_{\cX}^2\right] + 4\sqrt{2}d_{\cY}\sigma\delta \sqrt{\frac{\log \cN(\delta,\cF_{\rm NN}, \|\cdot\|_{\infty}) +2}{n}}  ,\\
 	&b=2\sqrt{2d_{\cY}}\sigma\sqrt{\frac{\log \cN(\delta,\cF_{\rm NN}, \|\cdot\|_{\infty}) +2}{n}}.
 \end{align*}
 Inequality (\ref{eq.T1.6}) can be rewritten as
 \begin{align*}
 	\rho^2\leq a+2b\rho,
 \end{align*}
 from which we deduce that
 \begin{align*}
 	(\rho-b)^2\leq a+b^2\Rightarrow \rho^2\leq 2a+4b^2.
 \end{align*}
 Therefore, 
 \begin{align}
 	{\rm T_1}=&2\rho^2\leq 8d_{\cY}\varepsilon_1^2+64d_{\cY}\sigma^2\frac{\log \cN(\delta,\cF_{\rm NN}, \|\cdot\|_{\infty}) +2}{n} + 16\sqrt{2}d_{\cY}\sigma\delta \sqrt{\frac{\log \cN(\delta,\cF_{\rm NN}, \|\cdot\|_{\infty}) +2}{n}}\nonumber\\& +8d_{\cY}\sigma\delta +8L_{E^n_{\cY}}^2L_{\Psi}^2\EE_{u\sim \gamma} \left[ \left\| \Pi_{\cX,d_{\cX}}^n(u)- u\right\|_{\cX}^2\right].
 	\label{eq.T1}
 \end{align}

\paragraph{Upper bound of ${\rm T_2}$.}
The term ${\rm T_2}$ is the difference between the population risk and the empirical risk of the network estimator $\Gamma_{\rm NN}$, while there is a factor 2 ahead of the empirical risk. 
Utilizing a covering of $\cF_{\rm NN}(d_{\cY},L,p,K,\kappa,M)$ and Bernstein-type inequalities, we establish a fast convergence of ${\rm T_2}$. The upper bound is presented in the following lemma (see a proof in Appendix \ref{lem.T2.proof}).
\begin{lemma}\label{lem.T2}
	Under the conditions of Theorem \ref{thm.general}, we have
	\begin{align}
		{\rm T_2}\leq \frac{35d_{\cY}L_{E^n_{\cY}}^2R_{E_{\cY}}^2}{n}\log \cN\left(\frac{\delta}{4d_{\cY}L_{E^n_{\cY}}R_{\cY}},\cF_{\rm NN},\|\cdot\|_{\infty}\right)+6\delta.
		\label{eq.T2.3}
	\end{align}
\end{lemma}

Substituting (\ref{eq.T1}) and (\ref{eq.T2.3}) into (\ref{eq.I}) gives rise to
\begin{align}
	{\rm I}\leq&2L_{D^n_{\cY}}^2\EE_{\cS_1}\EE_{u\sim \gamma}\left[\left\| \Gamma_{\rm NN} \circ E_{\cX}^n(u)- E_{\cY}^n\circ\Psi(u)\right\|_{2}^2\right] \nonumber\\
	=&2L_{D^n_{\cY}}^2\EE_{\cS_1}[ {\rm T_1}] + 2L_{D^n_{\cY}}^2\EE_{\cS_1}[ {\rm T_2}] \nonumber\\
	\leq& 16d_{\cY}L_{D^n_{\cY}}^2\varepsilon_1^2+128d_{\cY}\sigma^2L_{D^n_{\cY}}^2\frac{\log \cN(\delta,\cF_{\rm NN}, \|\cdot\|_{\infty}) +2}{n} + 32\sqrt{2}d_{\cY}\sigma L_{D^n_{\cY}}^2\delta \sqrt{\frac{\log \cN(\delta,\cF_{\rm NN}, \|\cdot\|_{\infty}) +2}{n}}\nonumber\\& +16d_{\cY}\sigma L_{D^n_{\cY}}^2 \delta +16L_{D^n_{\cY}}^2 L_{E^n_{\cY}}^2L_{\Psi}^2\EE_{u\sim \gamma}\| \Pi_{\cX,d_{\cX}}^n(u)- u\|_{\cX}^2 \nonumber\\
	& +\frac{70d_{\cY}L_{D^n_{\cY}}^2L_{E^n_{\cY}}^2R_{\cY}^2}{n}\log \cN\left(\frac{\delta}{4d_{\cY}L_{E^n_{\cY}}R_{\cY}},\cF_{\rm NN},\|\cdot\|_{\infty}\right)+12L_{D^n_{\cY}}^2\delta\nonumber \\
	\leq& 16d_{\cY}L_{D^n_{\cY}}^2\varepsilon_1^2+\frac{128d_{\cY}\sigma^2L_{D^n_{\cY}}^2+70d_{\cY}L_{D^n_{\cY}}^2L_{E^n_{\cY}}^2R_{\cY}^2}{n}\log \cN\left(\frac{\delta}{4d_{\cY}L_{E^n_{\cY}}R_{\cY}},\cF_{\rm NN}, \|\cdot\|_{\infty}\right)  \nonumber\\
	& + 64d_{\cY}\sigma L_{D^n_{\cY}}^2\delta \sqrt{\frac{\log \cN(\delta,\cF_{\rm NN}, \|\cdot\|_{\infty}) }{n}} +(16d_{\cY}\sigma+12) L_{D^n_{\cY}}^2 \delta  +16L_{D^n_{\cY}}^2 L_{E^n_{\cY}}^2L_{\Psi}^2\EE_{u\sim \gamma}\| \Pi_{\cX,d_{\cX}}^n(u)- u\|_{\cX}^2,
	\label{eq.I.1}
\end{align}
when $\delta<1$.
The covering number of $\cF_{\rm NN}(d_{\cY},L,p,K,\kappa,M)$ can be bounded in terms of its parameters, which is summarized in the following lemma:
\begin{lemma}[Lemma 6 of \citep{chen2019nonparametric} ]\label{lem.covering}
	Let $\cF_{\rm NN}(d_{\cY},L,p,K,\kappa,M)$ be a class of network: $[-B,B]^{d_{\cX}}\rightarrow [-M,M]^{d_{\cY}}$. For any $\delta>0$, the $\delta$-covering number of $\cF_{\rm NN}(L,p,K,\kappa,M)$ is bounded by
	\begin{align}
		\cN(\delta,\cF_{\rm NN}(d_{\cY},L,p,K,\kappa,M),\|\cdot\|_{\infty})\leq \left( \frac{2L^2(pB+2)\kappa^Lp^{L+1}}{\delta}\right)^{d_{\cY}K}.\label{eq.covering}
	\end{align}
\end{lemma}

Combining (\ref{eq.T1.archi}) and (\ref{eq.covering}) gives
\begin{align}
	\log \cN(\delta,\cF_{\rm NN}(d_{\cY},L,p,K,\kappa,M),\|\cdot\|_{\infty})\leq C_{7}d_{\cY}\left(\varepsilon_1^{-d_{\cX}}\log^3 \varepsilon_1^{-1}+\log \delta + \log d_{\cY}\right), \label{eq.NN.coveringLog}
\end{align}
where $C_{7}$ is a constant depending on $d_{\cX},R_{\cX},R_{\cY},L_{E^n_{\cX}}, L_{E^n_{\cY}},L_{D^n_{\cX}}$ and $L_{\Psi}$. Substituting (\ref{eq.NN.coveringLog}) into (\ref{eq.I.1}) yields
\begin{align}
	{\rm I}\leq &16d_{\cY}L_{D^n_{\cY}}^2\varepsilon_1^2+C_{7}d_{\cY}^2L_{D^n_{\cY}}^2\frac{128\sigma^2+70L_{E^n_{\cY}}^2R_{\cY}^2}{n}\left(\varepsilon_1^{-d_{\cX}}\log^3 \varepsilon_1^{-1}+\log \delta + \log d_{\cY}\right)  \nonumber\\
	&+ 64d_{\cY}\sigma L_{D^n_{\cY}}^2\delta \sqrt{\frac{C_{7}d_{\cY}\left(\varepsilon_1^{-d_{\cX}}\log^3 \varepsilon_1^{-1}+\log \delta + \log d_{\cY}\right) }{n}} \nonumber\\
	&+(16d_{\cY}\sigma+12) L_{D^n_{\cY}}^2 \delta +16L_{D^n_{\cY}}^2 L_{E^n_{\cY}}^2L_{\Psi}^2\EE_{u\sim \gamma}\left[\left\| \Pi_{\cX,d_{\cX}}^n(u)- u\right\|_{\cX}^2\right].
	\label{eq.I.3}
\end{align}

Setting
$$\varepsilon_1=d_{\cY}^{\frac{1}{2+d_{\cX}}}n^{-\frac{1}{2+d_{\cX}}}, \delta=n^{-1},$$
we get an upper bound of ${\rm I}$
\begin{align}
	{\rm I}\leq &C_1(\sigma^2+R_{\cY}^2)d_{\cY}^{\frac{4+d_{\cX}}{2+d_{\cX}}}n^{-\frac{2}{2+d_{\cX}}}\log^3 n +C_2(\sigma^2+R_{\cY}^2)d_{\cY}^2(\log d_{\cY})n^{-1}\nonumber \\
	& + 16L_{D^n_{\cY}}^2L_{E^n_{\cY}}^2L_{\Psi}^2\EE_{\cS_1}\EE_{u\sim \gamma}\left[\left\| \Pi_{\cX,d_{\cX}}^n(u)- u\right\|_{2}^2\right]
	\label{eq.I.4}
\end{align}
for some constants $C_1,C_2$ depending on $d_{\cX},R_{\cX},R_{\cY},L_{E^n_{\cX}}, L_{E^n_{\cY}},L_{D^n_{\cX}},L_{\Psi}$. The constants $C_1,C_2$ are the same ones as in Theorem \ref{thm.general}. The resulting network architecture $\cF(d_{\cY},L,p,K,\kappa,M)$ has
\begin{equation}
	\begin{aligned}
		&L=O(\log n+\log d_{\cY}), \ p=O\left(d_{\cY}^{-\frac{d_{\cX}}{2+d_{\cX}}}n^{\frac{d_{\cX}}{2+d_{\cX}}}\right),\ K=O\left(d_{\cY}^{-\frac{d_{\cX}}{2+d_{\cX}}}n^{\frac{d_{\cX}}{2+d_{\cX}}}\log n\right), \\
		&  \kappa=\max\left\{1,\sqrt{d_{\cY}}L_{E^n_{\cY}}R_{\cY}, \sqrt{d_{\cX}}L_{E^n_{\cX}}R_{\cX},\sqrt{d_{\cX}}L_{E^n_{\cX}}L_{E^n_{\cY}}L_{D^n_{\cX}}L_{\Psi}R_{\cX} \right\},\ M=\sqrt{d_{\cY}}L_{E^n_{\cY}}R_{\cY}.
	\end{aligned}\label{eq.NN.parameter.n.1}
\end{equation}

\paragraph{Combining the bounds of I and II.}
Putting (\ref{eq.I.3}) and (\ref{eq.II})  together gives rise to
\begin{align}
	&\EE_{\cS}\EE_{u\sim \gamma}\left[\left\|D_{\cY}^n\circ \Gamma_{\rm NN} \circ E_{\cX}^n(u)- \Psi(u)\right\|_{\cY}^2\right] \nonumber\\
	\leq & {\rm I}+{\rm II}\nonumber\\
	\leq &C_1(\sigma^2+R_{\cY}^2)d_{\cY}^{\frac{4+d_{\cX}}{2+d_{\cX}}}n^{-\frac{2}{2+d_{\cX}}}\log^3 n +C_2(\sigma^2+R_{\cY}^2)d_{\cY}^2(\log d_{\cY})n^{-1}\nonumber \\
	& + 16L_{D^n_{\cY}}^2L_{E^n_{\cY}}^2L_{\Psi}^2\EE_{\cS_1}\EE_{u\sim \gamma}\left[\left\| \Pi_{\cX,d_{\cX}}^n(u)- u\right\|_{2}^2\right]+2\EE_{\cS_1}\EE_{v^*\sim \Psi_{\#}\gamma}\left[ \left\|\Pi_{\cY,d_{\cY}}^n(v^*)- v^*\right\|_{\cY}^2\right],
	\label{eq.error.general.1}
\end{align}
where $C_1,C_2$ are constants depending on  $d_{\cX},R_{\cX},R_{\cY},L_{E^n_{\cX}},L_{E^n_{\cY}},L_{D^n_{\cX}},L_{\Psi}$.

\end{proof}

\subsection{Proof of Theorem \ref{thm.general.dense}} \label{thm.general.dense.proof}
The main framework of the proof of Theorem \ref{thm.general.dense} is the same as that of Theorem \ref{thm.general}, except special attentions need to be paid on bounding ${\rm T_1}$ and ${\rm T_2}$ in (\ref{eq.I.integral1}):
\begin{itemize}
	\item For ${\rm T_1}$, we establish a new result on the approximation error of deep neural networks with architecture $\cF_{\rm NN}(d_{\cY},L,p,M)$.
	\item For ${\rm T_2}$, we derive an upper bound using the uniform covering numbers. The motivation to use $\cF_{\rm NN}(d_{\cY}, L, p, M)$ is that it removes parameter upper bound, which is appealing to practical training. However, removing parameter upper bound leads to technical issues in bounding ${\rm T_2}$. We address these issues using the uniform covering numbers thanks to the boundedness of network outputs inspired by \citep{Jiao2021DeepNR}.
\end{itemize}   The first part of our proof is the same as that of Theorem \ref{thm.general} up to (\ref{eq.I.integral1}), which is omitted here. In the following, we bound ${\rm T_1}$ and ${\rm T_2}$ in order.

\paragraph{Upper bound of ${\rm T_1}$.} The upper bound of ${\rm T}_1$ can be derived similarly as that in Section \ref{thm.general.proof}, except we make two changes:
\begin{itemize}
	\item Replace Lemma \ref{lem.approx} by the following one
	\begin{lemma}[Theorem 1.1 of \citep{shijun2}]\label{lem.approx.dense}
		Let $0<\alpha\leq 1$ be a real number. There exists a FNN architecture $\cF_{\rm NN}(1,L,p,M)$ with $d_{\cY}=1$ such that for any integers $\widetilde{L},\widetilde{p}>0$ and $f\in \cC^{0,\alpha}([-B,B]^d)$ with $\|f\|_{\cC^{0,\alpha}}\leq R$, such an architecture gives rise to an FNN $\widetilde{f}$ with
		$$
		\left\|\widetilde{f}-f\right\|_{\infty}\leq C\widetilde{L}^{-\frac{2\alpha}{d}}\widetilde{p}^{-\frac{2\alpha}{d}}
		$$
		for some constant $C$ depending on $\alpha,d,B,R$. 
		This architecture has
		\begin{align*}
			L=O(\widetilde{L}\log \widetilde{L}), \ p=O\left(\widetilde{p}\log \widetilde{p}\right),\ M=R.
		\end{align*}
		The constant hidden in $O(\cdot)$ depends on $\alpha,d,B,R$.
	\end{lemma}
	
	According to Lemma \ref{lem.approx.dense} with $\alpha=1$, for any $\varepsilon_1>0$, there is a network architecture $\cF_{\rm NN}(d_{\cY},L,p,M)$, such that for any  $\Gamma_{\rm NN}^n$ defined in (\ref{eq.pid}), there exists a $\widetilde{\Gamma}_d^n\in \cF_{\rm NN}(d_{\cY},L,p,M)$ with
	\begin{align*}
		\left\|\widetilde{\Gamma}_d^n-\Gamma_d^n\right\|_{\infty}\leq \varepsilon_1.
	\end{align*}
	Such a network architecture has
	\begin{equation}
		\begin{aligned}
			&L=O(\widetilde{L}\log \widetilde{L}), \ p=O\left(\widetilde{p}\log \widetilde{p}\right), M=\sqrt{d_{\cY}}L_{E^n_{\cY}}R_{\cY},
		\end{aligned}\label{eq.NN.parameter.dense}
	\end{equation}
	where $\widetilde{L},\widetilde{p}>0$ are integers satisfying $ \widetilde{L}\widetilde{p}=\left\lceil \varepsilon_1^{-d_{\cX}/2}\right\rceil$. The constant hidden in $O(\cdot)$ depends on $d_{\cX},L_{E^n_{\cY}},L_{D^n_{\cX}},L_{\Psi},B$ and $M$.
	
	\item Replace Lemma \ref{lem.T1.2} by
	\begin{lemma}\label{lem.T1.2.dense}
		Under the conditions of Theorem \ref{thm.general.dense}, for any $\delta\in (0,1)$, we have 
		\begin{align}
			&\EE_{\cS_2}\left[\frac{1}{n}\sum_{i=n+1}^{2n}\langle \Gamma_{\rm NN} \circ E_{\cX}^n(u_i), \bepsilon_i\rangle\right] \nonumber\\
			\leq & 2\sqrt{2d_{\cY}}\sigma\left(\sqrt{\EE_{\cS_2}\|\Gamma_{\rm NN} \circ E_{\cX}^n(u_i)-\Gamma_{d}^n \circ E_{\cX}^n(u_i)\|_n^2}+\sqrt{d_{\cY}}\delta \right)\sqrt{\frac{\log \cN(\delta,\cF_{\rm NN}, n) +2}{n}} +d_{\cY}\sigma\delta.
			\label{eq.T1.5.dense}
		\end{align}
	\end{lemma}
	Lemma \ref{lem.T1.2.dense} can be proved similarly as Lemma \ref{lem.T1.2}. We need to replace the $\delta$-cover $\cF^*=\{\Gamma^*_j\}_{j=1}^{\cN(\delta,\cF_{\rm NN},\|\cdot\|_{\infty})}$ by a  $\delta$-cover of $\cF_{\rm NN}$ with respect to $\cS_2$: $\cF^*=\{\Gamma^*_j\}_{j=1}^{\cN(\delta,\cF_{\rm NN},n)}$, where $\cN(\delta,\cF_{\rm NN},n)$ is the uniform covering number. Here the cover $\cF^*$ depends on the samples $\{E^n_{\cX}(u_i)\}_{i=n+1}^{2n}$.  Then there exists $\Gamma^*\in \cF^*$ satisfying $\left\|\Gamma^*\circ E^n_{\cX}(u_i)-\Gamma_{\rm NN}\circ E^n_{\cX}(u_i)\right\|_{\infty}\leq \delta$ for any $n+1\leq i\leq 2n$. The proof is omitted here.
\end{itemize}
Following the rest of the proof for ${\rm T}_1$ in Section \ref{thm.general.proof}, we can derive that 
\begin{align}
	{\rm T_1}\leq& 8d_{\cY}\varepsilon_1^2+64d_{\cY}\sigma^2\frac{\log \cN(\delta,\cF_{\rm NN}, n) +2}{n} + 16\sqrt{2}d_{\cY}\sigma\delta \sqrt{\frac{\log \cN(\delta,\cF_{\rm NN}, n) +2}{n}}\nonumber\\& +8d_{\cY}\sigma\delta +8L_{E^n_{\cY}}^2L_{\Psi}^2\EE_{u\sim \gamma}\left[\left\| \Pi_{\cX,d_{\cX}}^n(u)- u\right\|_{\cX}^2\right].
	\label{eq.T1.dense}
\end{align}
The network architecture of $\cF_{\rm NN}(d_{\cY},L,p,M)$ is specified in (\ref{eq.NN.parameter.dense}).

%

\paragraph{Upper bound of ${\rm T_2}$.}

Using the covering number defined in Definition \ref{def.covernumber.n}, we have the following bound of ${\rm T_2}$.
\begin{lemma}\label{lem.T2.dense}
	Under the conditions of Theorem \ref{thm.general.dense}, we have
	\begin{align}
		{\rm T_2}\leq \frac{35d_{\cY}R_{\cY}^2}{n}\log \cN\left(\frac{\delta}{4d_{\cY}L_{E^n_{\cY}}R_{\cY}},\cF_{\rm NN},2n\right)+6\delta.
		\label{eq.T2.dense}
	\end{align}
\end{lemma}
Lemma \ref{lem.T2.dense} is proved in Appendix \ref{lem.T2.dense.proof} using techniques similar to those in the proof of Lemma \ref{lem.T2}.
Substituting (\ref{eq.T1.dense}) and (\ref{eq.T2.dense}) into (\ref{eq.I}) gives rise to
\begin{align}
	{\rm I}\leq&2L_{D_{\cY}}^2\EE_{\cS_1}\EE_{u\sim \gamma}\left[ \left\| \Gamma_{\rm NN} \circ E_{\cX}^n(u)- E_{\cY}^n\circ\Psi(u)\right\|_{2}^2 \right]\nonumber\\
	=&2L_{D^n_{\cY}}^2\EE_{\cS_1} [{\rm T_1}] + 2L_{D^n_{\cY}}^2\EE_{\cS_1}[ {\rm T_2}] \nonumber\\
	\leq& 16d_{\cY}L_{D^n_{\cY}}^2\varepsilon_1^2+128d_{\cY}\sigma^2L_{D^n_{\cY}}^2\frac{\log \cN(\delta,\cF_{\rm NN}, n) +2}{n} + 32\sqrt{2}d_{\cY}\sigma L_{D^n_{\cY}}^2\delta \sqrt{\frac{\log \cN(\delta,\cF_{\rm NN}, n) +2}{n}}\nonumber\\
	& +16d_{\cY}\sigma L_{D^n_{\cY}}^2\delta +16L_{D^n_{\cY}}^2L_{E^n_{\cY}}^2L_{\Psi}^2\EE_{u\sim \gamma}\left[ \left\| \Pi_{\cX,d_{\cX}}^n(u)- u\right\|_{\cX}^2\right]\\
	& +\frac{70d_{\cY}L_{D^n_{\cY}}^2L_{E^n_{\cY}}^2R_{\cY}^2}{n}\log \cN\left(\frac{\delta}{4d_{\cY}L_{E^n_{\cY}}R_{\cY}},\cF_{\rm NN},2n\right)+12L_{D^n_{\cY}}^2\delta\nonumber \\
	\leq& 16d_{\cY}L_{D^n_{\cY}}^2\varepsilon_1^2+\frac{128d_{\cY}\sigma^2L_{D^n_{\cY}}^2+70d_{\cY}L_{D^n_{\cY}}^2L_{E^n_{\cY}}^2R_{\cY}^2}{3n}\log \cN\left(\frac{\delta}{4d_{\cY}L_{E^n_{\cY}}R_{\cY}},\cF_{\rm NN}, 2n\right)  \nonumber\\
	& + 64d_{\cY}\sigma L_{D^n_{\cY}}^2\delta \sqrt{\frac{\log \cN(\delta,\cF_{\rm NN}, n) }{n}}+(16d_{\cY}\sigma+12) L_{D^n_{\cY}}^2 \delta +16L_{D^n_{\cY}}^2 L_{E^n_{\cY}}^2L_{\Psi}^2\EE_{u\sim \gamma}\left[ \left\| \Pi_{\cX,d_{\cX}}^n(u)- u\right\|_{\cX}^2\right]
	\label{eq.I.dense}
\end{align}

The covering number in (\ref{eq.I.dense}) can be bounded using the pseudo-dimension of the network class:
\begin{lemma}[Theorem 12.2 of \citep{anthony1999neural}]\label{lem.coverN.dense}
	Let $F$ be a class of functions from some domain $\Omega$ to $[-M,M]$. Denote the pseudo-dimension of $F$ by $\Pdim(F)$. For any $\delta>0$, we have
	\begin{align}
		\cN(\delta,F,m)\leq \left(\frac{2eMm}{\delta \Pdim(F)}\right)^{\Pdim(F)}
	\end{align}
	for $m>\Pdim(F)$.
\end{lemma}
The next lemma shows that the pseudo-dimension of $\cF_{\rm NN}(1,L,p,M)$ can be bounded using its parameters:
\begin{lemma}[Theorem 7 of \citep{bartlett2019nearly}]\label{lem.psudoDim.dense}
	For any network architecture $\cF_{\rm NN}$ with $L$ layers and $U$ parameters, there exists a universal constant $C$ such that
	\begin{align}
		\Pdim(\cF_{\rm NN})\leq CLU\log(U).
	\end{align}
\end{lemma}
Now conciser the network architecture $\cF_{\rm NN}(1,L,p,M)$, the number of parameters is bounded by $U=Lp^2$.
Combing Lemma \ref{lem.coverN.dense} and \ref{lem.psudoDim.dense}, we have
\begin{align}
	\log \cN\left(\frac{\delta}{4d_{\cY}L_{E^n_{\cY}}R_{\cY}},\cF_{\rm NN}(d_{\cY},L,p,M),2n\right)\leq C_{8}d_{\cY}p^2L^2\log\left(p^2L\right)\left(\log M+\log \delta^{-1}+\log n\right)
	\label{eq.covering.form.dense}
\end{align}
when $2n>C_{9}p^2L^2\log(p^2L)$ for some universal constant $C_{8},C_{9}$ .
Substituting (\ref{eq.NN.parameter.dense}) into (\ref{eq.covering.form.dense}) gives rise to
\begin{align}
	\log \cN\left(\frac{\delta}{4d_{\cY}L_{E^n_{\cY}}R_{\cY}},\cF_{\rm NN},2n\right)\leq C_{8} d_{\cY}\varepsilon_1^{-d_{\cX}}\log^5\left(\varepsilon_1^{-1}\right)\left(\log \delta^{-1}+\log n\right).
	\label{eq.NN.coveringLog.dense}
\end{align}

Substituting (\ref{eq.NN.coveringLog.dense}) into (\ref{eq.I.dense}) yields
\begin{align}
	{\rm I}\leq &16d_{\cY}L_{D^n_{\cY}}^2\varepsilon_1^2+C_{8}d^2_{\cY}L_{D^n_{\cY}}^2\frac{128\sigma^2+70L_{E^n_{\cY}}^2R_{\cY}^2}{n}\varepsilon_1^{-d_{\cX}}\log^5\left(\varepsilon_1^{-1}\right)\left(\log \delta^{-1}+\log n\right)  \nonumber\\
	&+ 64d_{\cY}\sigma L_{D^n_{\cY}}^2\delta \sqrt{\frac{C_{8} d_{\cY}\varepsilon_1^{-d_{\cX}}\log^5\left(\varepsilon_1^{-1}\right)\left(\log \delta^{-1}+\log n\right) }{n}} \nonumber \\
	& +(16d_{\cY}\sigma+12) L_{D^n_{\cY}}^2 \delta +16L_{D^n_{\cY}}^2 L_{E^n_{\cY}}^2L_{\Psi}^2\EE_{u\sim \gamma}\left[ \left\| \Pi_{\cX,d_{\cX}}^n(u)- u\right\|_{\cX}^2\right].
	\label{eq.I.1.dense}
\end{align}
Setting
$$\varepsilon_1=d_{\cY}^{\frac{1}{2+d_{\cX}}}n^{-\frac{1}{2+d_{\cX}}}, \delta=n^{-1},$$
we have
\begin{align}
	{\rm I}\leq &C_{4}(\sigma^2+R_{\cY}^2)d_{\cY}^{\frac{4+d_{\cX}}{2+d_{\cX}}}n^{-\frac{2}{2+d_{\cX}}}\log^6 n  + 16L_{D^n_{\cY}}^2 L_{E^n_{\cY}}^2L_{\Psi}^2\EE_{u\sim \gamma}\left[ \left\| \Pi_{\cX,d_{\cX}}^n(u)- u\right\|_{\cX}^2\right],
	\label{eq.I.2.dense}
\end{align}
where $C_{4}$ is a constant depending on $d_{\cX},R_{\cX},R_{\cY},L_{E^n_{\cX}}, L_{E^n_{\cY}},L_{D^n_{\cX}},L_{\Psi}$, the same constant in Theorem \ref{thm.general.dense}. The resulting network architecture $\cF(L,p,M)$ has
\begin{equation}
	\begin{aligned}
		&L=O(\widetilde{L}\log \widetilde{L}), \ p=O\left(\widetilde{p} \log \widetilde{p}\right), M=\sqrt{d_{\cY}}L_{E^n_{\cY}}R_{\cY},
	\end{aligned}
	\label{eq.NN.parameter.n.1.dense}
\end{equation}
where $\widetilde{L}\widetilde{p}=d_{\cY}^{-\frac{d_{\cX}}{4+2d_{\cX}}}n^{\frac{d_{\cX}}{4+2d_{\cX}}}$. Now we check the condition in Lemma \ref{lem.coverN.dense}. Under the choice of $L$ and $p$ above, we have
\begin{align*}
	L^2p^2\log(p^2L)=O\left(n^{\frac{2d_{\cX}}{4+2d_{\cX}}}\log^5 n \right)<2n
\end{align*}
when $n$ is large enough. The condition is satisfied.

\paragraph{Combining the bounds of I and II.}
Putting (\ref{eq.I.2.dense}) and (\ref{eq.II})  together gives rise to
\begin{align}
	&\EE_{\cS}\EE_{u\sim \gamma}\left[ \left\|D_{\cY}^n\circ \Gamma_{\rm NN} \circ E_{\cX}^n(u)- \Psi(u)\right\|_{\cY}^2 \right] \nonumber\\
	\leq & {\rm I}+{\rm II}\nonumber\\
	\leq &C_{4}(\sigma^2+R_{\cY}^2)d_{\cY}^{\frac{4+d_{\cX}}{2+d_{\cX}}}n^{-\frac{2}{2+d_{\cX}}}\log^6 n  \nonumber\\
	& +  16L_{D^n_{\cY}}^2 L_{E^n_{\cY}}^2L_{\Psi}^2\EE_{u\sim \gamma}\left[ \left\| \Pi_{\cX,d_{\cX}}^n(u)- u\right\|_{\cX}^2\right]+2\EE_{\cS_1}\EE_{v^*\sim \Psi_{\#}\gamma}\left[ \left\|\Pi_{\cY,d_{\cY}}^n(v^*)- v^*\right\|_{\cY}^2\right],
	\label{eq.error.general.1.dense}
\end{align}
which finishes the proof.

\subsection{Proof of Corollary \ref{coro.leg}}
\label{coro.leg.proof}
\begin{proof}[Proof of Corollary \ref{coro.leg}]
	We only need to derive upper bounds of 
	$$
	\EE_{u\sim \gamma}\left[ \left\| \Pi_{\cX,d_{\cX}}(u)-u\right\|_{2}^2\right] \mbox{ and } \EE_{v\sim \Psi_{\#}\gamma}\Big[ \big\|\Pi_{\cY,d_{\cY}}(v)- v\big\|_{\cY}^2\Big].
	$$ Then Corollary \ref{coro.leg} is a direct result of Corollary \ref{coro.basis.dense}.
	Our proof relies on the following lemma which gives an approximation error of Legendre polynomials for H\"{o}lder functions:
	\begin{lemma}[Theorem 4.5(ii) of \citep{schultz1969multivariate}]\label{lem.leg}
		Let $k\geq0$ be an integer and $\alpha>0$. For any $f\in \cC^{k,\alpha}([0,1]^D)$ with $\|f\|_{\cC^{k,\alpha}}<\infty$, there exists $\widetilde{f}\in \spspan(\Phi^{{\rm L},r})$ such that
		\begin{align}
			\left\|f-\widetilde{f}\right\|_{\infty}\leq \frac{C}{r^{k+\alpha}},
		\end{align}
		where $C$ is a constant depending on $D$ and $\|f\|_{\cC^{k,\alpha}}$.
	\end{lemma}
	We first derive an upper bound of $\EE_{u\sim \gamma}\left[\left\| \Pi_{\cX,d_{\cX}}(u)-u\right\|_{2}^2\right]$. For any $u\in \Omega_{\cX}$, according to Lemma \ref{lem.leg}, there exists $\widetilde{u}\in \spspan(\Phi^{{\rm L},r_{\cX}})$ such that
	$$
	\|u-\widetilde{u}\|_{\infty}\leq C_{10}r_{\cX}^{-s},
	$$
	where $s=k+\alpha$, $C_{10}$ is a constant depending on $D$ and $C_{\cH_P,\cX}$. We deduce that
	\begin{align*}
		\left\| \Pi_{\cX,d_{\cX}}(u)-u\right\|_{\cX}^2=&\min_{\bar{u}\in \spspan(\Phi^{{\rm L},r_{\cX}})} \left\|\bar{u}-u\right\|_{\cX}^2\\
		\leq & \|\widetilde{u}-u\|_{\cX}^2\\
		\leq & \int_{[-1,1]^D} |\widetilde{u}-u|^2d\bx\\
		\leq & 2^DC_{10}r_{\cX}^{-2s}\\
		=&2^DC_{10}d_{\cX}^{-\frac{2s}{D}},
	\end{align*}
	where in the last equality $d_{\cX}=r_{\cX}^D$ is used. Therefore
	\begin{align*}
		\EE_{u\sim \gamma}\left[\left\| \Pi_{\cX,d_{\cX}}(u)-u\right\|_{2}^2\right]\leq C_5d_{\cX}^{-\frac{2s}{D}}.
	\end{align*}
where $C_5$ is a constant depending on $D$ and $C_{\cH_P,\cX}$.
	Similarly, one can show
	\begin{align*}
		\EE_{v\sim \Psi_{\#}\gamma}\left[ \left\|\Pi_{\cY,d_{\cY}}(v)- v\right\|_{\cY}^2\right]\leq C_6d_{\cY}^{-\frac{2s}{D}},
	\end{align*}
	where $C_6$ is a constant depending on $D$ and $C_{\cH,\cY}$.
	The theorem is proved.
\end{proof}

\subsection{Proof of Corollary \ref{coro.trig}}
\label{coro.trig.proof}
\begin{proof}[Proof of Corollary \ref{coro.trig}]
	Our proof relies on the following lemma which gives an approximation error of trigonometric bases for periodic H\"{o}lder functions 
	\begin{lemma}[Theorem 4.3(ii) of \citep{schultz1969multivariate}]\label{lem.trig}
		Let $k\geq 0$ be an integer and $0<\alpha\leq 1$. For any $f\in \cP\cap \cC^{k,\alpha}\left([0,1]^D\right)$ with $\|f\|_{\cC^{k,\alpha}}<\infty$, there exists $\widetilde{f}\in \spspan(\Phi^{{\rm T},r})$ such that
		\begin{align}
			\left\|f-\widetilde{f}\right\|_{\infty}\leq \frac{C}{r^s},
		\end{align}
		where $C$ is a constant depending on $D$ and $\|f\|_{\cC^{k,\alpha}}$.
	\end{lemma}
	Corollary \ref{coro.trig} can be proved by following the proof of Corollary \ref{coro.trig} in which Lemma \ref{lem.leg} is replaced by Lemma \ref{lem.trig}.
\end{proof}

\subsection{Proof of Theorem \ref{thm.pca}}
\label{thm.pca.proof}
\begin{proof}[Proof of Theorem \ref{thm.pca}]
Lemma \ref{lem.LipschitzFG} implies that $E_{\cX}^n,D_{\cX}^n,E_{\cY}^n,D_{\cY}^n$ are Lipschitz with a Lipschitz constant 1. Therefore Corollary \ref{coro.basis.dense} can be applied.
We only need to bound $\EE_{\cS}\EE_{u\sim \gamma}\left[\left\| \Pi_{\cX,d_{\cX}}^n(u)- u\right\|_{\cX}^2\right]$ and $\EE_{\cS}\EE_{w\sim \Psi_{\#}\gamma}\left[\left\|\Pi_{\cY,d_{\cY}}^n(w)- w\right\|_{\cY}^2\right]$ in (\ref{eq.error.general.basis.dense}).
We use the following lemma:
\begin{lemma}[Theorem 3.4 of \citep{bhattacharya2020model}]\label{lem.PCA.empirical}
	Let $\cH$ be a separable Hilbert space and $\rho$ be a probabillity measure defined on it. Define the covariance operator $G_{\rho}=\EE_{u\sim \rho} u\otimes u$ and its empirical estimation from $n$ samples by $G^n_{\rho}=\frac{1}{n} \sum_{i=1}^n u_i\otimes u_i$ where $\{u_i\}_{i=1}^n$ are i.i.d. samples sampled from $\rho$. For some integer $d>0$, let $\Pi_{\cH,d}$ and $ \Pi_{\cH,d}^n$ be the projectors that project any $u\in\cH$ to the space spanned by the eigenfunctions corresponding to the largest $d$ eigenvalues of $G_{\rho}$ and $G_{\rho}^n$, respectively.
	We have
	\begin{align}
		\EE_{\{u_k\}_{k=1}^n\sim \rho}\EE_{u\sim\rho} \left[ \left\|\Pi_{\cH,d}^n(u)-u\right\|_{\cH}^2\right]\leq \sqrt{\frac{Cd}{n}}+\EE_{u\sim\rho}\left[\left \|\Pi_{\cH,d}(u)-u\right\|_{\cH}^2\right]
	\end{align}
	with $C=\EE_{\{u_i\}_{i=1}^n\sim \rho}\left[\left\|G^n-G\right\|_{\rm HS}^2\right]$, where $\|\cdot\|_{\rm HS}$ is the Hilbert-Schmidt norm.
\end{lemma}

We first bound $\EE_{\cS}\EE_{u\sim\gamma} \left[ \left\|\Pi_{\cX,d_{\cX}}^n(u)-u\right\|_{\cX}^2\right]$.
For any $u\sim\gamma$, we have $\|u\|_{\cX}\leq R_{\cX}$. Therefore
$$
\EE_{u\sim \gamma} \left[  \left\|G_{\cX}^n-G_{\cX}\right\|_{\rm HS}^2 \right]\leq 4\EE_{u\sim \gamma} \left[ \|u\|_{\cX}^4\right]\leq 4R_{\cX}^4
$$
and Lemma \ref{lem.PCA.empirical} gives
\begin{align}
	\EE_{\cS}\EE_{u\sim\gamma} \left[ \left\|\Pi_{\cX,d_{\cX}}^n(u)-u\right\|_{\cX}^2 \right]\leq \sqrt{\frac{4R_{\cX}^4d_{\cX}}{n}}+\EE_{u\sim\gamma} \left[ \left\|\Pi_{\cX,d_{\cX}}(u)-u\right\|_{\cX}^2\right].
	\label{eq.PCA.errorX}
\end{align}

An upper bound of $\EE_{\cS}\EE_{w\sim \Psi_{\#}\gamma} \left[ \left\|\Pi_{\cY,d_{\cY}}^n(w)-w\right\|_{\cY}^2\right]$ is given by the following lemma (see a proof in Appendix \ref{lem.pca.PiY.proof}):
\begin{lemma}\label{lem.pca.PiY}
	Under the conditions of Theorem \ref{thm.pca}, we have
	\begin{align}
		\EE_{\cS}\EE_{w\sim \Psi_{\#}\gamma} \left[ \left\|\Pi_{\cY,d_{\cY}}^n(w)-w\right\|_{\cY}^2\right]\leq& 4\sqrt{\frac{(R_{\cY}+\widetilde{\sigma})^4d_{\cY}}{n}}+8\left(\frac{\widetilde{\sigma}}{\lambda_{d_{\cY}}-\lambda_{d_{\cY+1}}}\right)^2\widetilde{\sigma}^2(R_{\cY}+\widetilde{\sigma})^2\nonumber\\
		&+10\widetilde{\sigma}^2+ 8\EE_{w\sim \Psi_{\#}\gamma}\left[ \left\|\Pi_{\cY,d_{\cY}}(w)-w\right\|_{\cY}^2\right].
		\label{eq.pca.PiY}
	\end{align}
\end{lemma}

\end{proof}

\subsection{Proof of Corollary \ref{coro.PCA.noPerturb}}\label{coro.PCA.noPerturb.proof}

\begin{proof}[Proof of Corollary \ref{coro.PCA.noPerturb}]
	We only need to show that the eigenspace spanned by the first $d_{\cY}$ principal eigenfunctions of $G_{\Psi_{\#}\gamma}$ is the same as that of $G_{\zeta}$. Then Corollary \ref{coro.PCA.noPerturb} can be proved by following the proof of Theorem \ref{thm.pca} in which the upper bound of  $\EE_{\cS}\EE_{w\sim \Psi_{\#}\gamma}\left[\left\|\Pi_{\cY,d_{\cY}}^n(w)- w\right\|_{\cY}^2\right]$ can be derived in the same manner as that of  $\EE_{\cS}\EE_{u\sim \gamma}\left[\left\|\Pi_{\cX,d_{\cX}}^n(u)- u\right\|_{\cX}^2\right]$. 
	
	Denote the eigenvalues of $G_{\Psi_{\#}\gamma}$ in non-increasing order by $\{\lambda_{\Psi_{\#}\gamma,k}\}_{k=1}^{\infty}$. Denote the eigenspace spanned by the first $d_{\cY}$ principal eigenfunctions of $G_{\Psi_{\#}\gamma}$ by $\cK$, and its compliment by $\cK^{\top}$. Similarly, we define $\cK_{\zeta}$ and $\cK_{\zeta}^{\top}$ for $G_{\zeta}$. From our assumption, $\cK$ is also the eigenspace spanned by the first $d_{\cY}$ eigenfunctions of $G_{\mu}$. We denote the eigenvalues of $G_{\mu}$ in non-increasing order by $\{\lambda_{\mu,k}\}_{k=1}^{\infty}$. We are going to show that $\cK=\cK_{\zeta}$. From (\ref{eq.GkappaN}), we have  $G_{\zeta}=G_{\Psi_{\#}\gamma}+G_{\mu}$. Note that for any $\phi\in \cK$ and $\widetilde{\phi}\in \cK^{\top}$ with unit length, we have
	\begin{align*}
		\langle G_{\zeta}\phi,\phi \rangle_{\cY}=&\langle G_{\Psi_{\#}\gamma}\phi,\phi\rangle_{\cY}+ \langle G_{\mu}\phi,\phi \rangle_{\cY} \\
		\geq &\lambda_{\Psi_{\#},d_{\cY}}+\lambda_{\mu,d_{\cY}+1}\\
		\geq & \lambda_{\Psi_{\#},d_{\cY}+1}+\lambda_{\mu,d_{\cY}+1}\\
		\geq &\langle G_{\Psi_{\#}\gamma}\widetilde{\phi},\widetilde{\phi}\rangle_{\cY}+ \langle G_{\mu}\widetilde{\phi},\widetilde{\phi} \rangle_{\cY} \\
		=&\langle G_{\zeta}\widetilde{\phi},\widetilde{\phi}\rangle_{\cY}.
	\end{align*} 
	Since both $\cK$ and $\cK_{\zeta}$ have dimension $d_{\cY}$, we have $\cK=\cK_{\zeta}$. The proof is finished.
\end{proof}

\subsection{Proof of Theorem \ref{thm.general.M}}\label{thm.general.M.proof}
\begin{proof}[Proof of Theorem \ref{thm.general.M}]
	Theorem \ref{thm.general.M} can be proved by following the proof of Theorem \ref{thm.general.dense} with the following changes:
	\begin{itemize}
		\item Replace $E_{\cX}^n$ by $E_{\cX}$.
		\item Under Assumption \ref{assum.Psi.L} and \ref{assum.M}, our target function $E_{\cY}\circ\Psi\circ D_{\cX}$ is a Lipschitz function on $\cM$. We replace Lemma \ref{lem.approx.dense} by the following one (see a proof in Appendix \ref{lem.approx.M.dense.proof}):
		\begin{lemma}
			\label{lem.approx.M.dense}
			Suppose Assumption \ref{assum.M} holds. Assume for any $\ab\in\cM$, $\|\ab\|_{\infty}\leq B$ for some $B>0$. There exists a FNN architecture $\cF_{\rm NN}(1,L,p,M)$ such that for any integers $\widetilde{L},\widetilde{p}>0$ and $f\in \cC^{0,1}(\cM)$ with $\|f\|_{\cC^{0,1}}\leq R$, such an architecture gives rise to a FNN $\widetilde{f}$ with
			$$
			\left\|\widetilde{f}-f\right\|_{\infty}\leq C\widetilde{L}^{-\frac{2}{d_0}}\widetilde{p}^{-\frac{2}{d_0}}
			$$
			for some constant $C$ depending on  $d_0,B,R,\tau$ and the surface area of $\cM$.
			This architecture has
			\begin{align}
				L=O\left(\widetilde{L}\log \widetilde{L}\right), \ p=O\left(d_{\cX}\widetilde{p}\log \widetilde{p}\right),\ M=R.
				\label{eq.NN.parameter.n.M.dense.lemma}
			\end{align}
			The constant hidden in $O(\cdot)$ depends on $d_0,B,R,\tau$ and the surface area of $\cM$.

		\end{lemma}
	\end{itemize}
\end{proof}

\subsection{Proof of Theorem \ref{thm.general.lowCom}}
\label{thm.general.lowCom.proof}
\begin{proof}[Proof of Theorem \ref{thm.general.lowCom}]
Theorem \ref{thm.general.lowCom} can be proved similarly as Theorem \ref{thm.general.M} while special attention needs to be paid on bounding $\log \cN(\delta,\cF_{\rm NN}(d_{\cY},L,p,M),n)$. Note that the total number of parameters of $\cF_{\rm NN}(d_{\cY},L,p,M)$ is bounded by $U=Lp+d_{\cX}p$.
Combing Lemma \ref{lem.coverN.dense} and \ref{lem.psudoDim.dense}, we have
\begin{align}
	&\log \cN\left(\frac{\delta}{4d_{\cY}L_{E^n_{\cY}}R_{\cY}},\cF_{\rm NN}(d_{\cY},L,p,M),2n\right)\nonumber\\
	\leq &C_{11}d_{\cY}(p^2L^2+d_{\cX}pL)\log\left(p^2L+d_{\cX}Lp\right)\left(\log M+\log \delta^{-1}+\log n\right),
	\label{eq.covering.form.dense.d1}
\end{align}
where $C_{11}$ is a universal constant.
According to (\ref{eq.NN.parameter.dense}), one has $Lp=O\left( \varepsilon_1^{-d_{\cX}/2}\log^2 \left(\varepsilon^{-1}\right)\right)$. Using this relation and substituting the choice of $L,p$ in (\ref{eq.NN.parameter.dense}) to (\ref{eq.covering.form.dense.d1}) gives rise to
\begin{align}
	&\log \cN\left(\frac{\delta}{4d_{\cY}L_{E^n_{\cY}}R_{\cY}},\cF_{\rm NN}(d_{\cY},L,p_1,p_2,M),2n\right)\nonumber\\
	\leq &C_{11} d_{\cY}\left(\varepsilon_1^{-d_{\cX}}+ d_{\cX}\varepsilon_1^{-d_{\cX}/2}\right)\log^5\left(\varepsilon_1^{-1}\right)\left(\log \delta^{-1}+\log n\right).
	\label{eq.NN.coveringLog.dense.d1}
\end{align}
The proof can be finished by following the rest of the proof of Theorem \ref{thm.general.M}.
\end{proof}

\section{Conclusion}\label{sec.conclusion}
We study the generalization error of a general framework on learning operators between infinite-dimensional spaces by two types of deep neural networks. Our upper bound consists of a network estimation error and a projections error, and holds for general encoders and decoders under mild assumptions. The application of our results on some popular encoders and decoders are discussed, such as those using Legendre polynomials, trigonometric functions, and PCA. We also consider two scenarios where additional low dimensional structures of data can be exploited. The two scenarios are: (1) the input data can be encoded to vectors on a low dimensional manifold; (2) the operator has low complexity. In both scenarios, we show that the generalization error converges at a fast rate depending on the intrinsic dimension.
Our results show that deep neural networks are adaptive to low dimensional structures of data in operator estimation. In general, our results provide  a theoretical justification on the successes of deep neural networks for learning operators between  infinite dimensional spaces. 
\bibliographystyle{abbrv}
\bibliography{ref}

\begin{thebibliography}{10}

\bibitem{li2020neural}
A.~Anandkumar, K.~Azizzadenesheli, K.~Bhattacharya, N.~Kovachki, Z.~Li, B.~Liu,
  and A.~Stuart.
\newblock Neural operator: Graph kernel network for partial differential
  equations.
\newblock In {\em ICLR 2020 Workshop on Integration of Deep Neural Models and
  Differential Equations}, 2020.

\bibitem{anthony1999neural}
M.~Anthony and P.~Bartlett.
\newblock Neural network learning: theoretical foundations, 1999.

\bibitem{barron1993}
A.~R. Barron.
\newblock Universal approximation bounds for superpositions of a sigmoidal
  function.
\newblock {\em IEEE Transactions on Information Theory}, 39(3):930--945, May
  1993.

\bibitem{bartlett2019nearly}
P.~L. Bartlett, N.~Harvey, C.~Liaw, and A.~Mehrabian.
\newblock Nearly-tight vc-dimension and pseudodimension bounds for piecewise
  linear neural networks.
\newblock {\em The Journal of Machine Learning Research}, 20(1):2285--2301,
  2019.

\bibitem{10.1214/18-AOS1747}
B.~Bauer and M.~Kohler.
\newblock {On deep learning as a remedy for the curse of dimensionality in
  nonparametric regression}.
\newblock {\em The Annals of Statistics}, 47(4):2261 -- 2285, 2019.

\bibitem{DBLP:journals/corr/abs-1809-03062}
J.~Berner, P.~Grohs, and A.~Jentzen.
\newblock Analysis of the generalization error: Empirical risk minimization
  over deep artificial neural networks overcomes the curse of dimensionality in
  the numerical approximation of black-scholes partial differential equations.
\newblock {\em CoRR}, abs/1809.03062, 2018.

\bibitem{bhattacharya2020model}
K.~Bhattacharya, B.~Hosseini, N.~B. Kovachki, and A.~M. Stuart.
\newblock Model reduction and neural networks for parametric pdes.
\newblock {\em arXiv preprint arXiv:2005.03180}, 2020.

\bibitem{bungartz2004sparse}
H.-J. Bungartz and M.~Griebel.
\newblock Sparse grids.
\newblock {\em Acta numerica}, 13:147--269, 2004.

\bibitem{CAI2021110296}
S.~Cai, Z.~Wang, L.~Lu, T.~A. Zaki, and G.~E. Karniadakis.
\newblock {DeepM{\&}Mnet: Inferring the electroconvection multiphysics fields
  based on operator approximation by neural networks}.
\newblock {\em Journal of Computational Physics}, 436:110296, 2021.

\bibitem{Yuan1}
Y.~Cao and Q.~Gu.
\newblock Generalization bounds of stochastic gradient descent for wide and
  deep neural networks.
\newblock {\em CoRR}, abs/1905.13210, 2019.

\bibitem{chen1998applications}
L.~Q. Chen and J.~Shen.
\newblock Applications of semi-implicit fourier-spectral method to phase field
  equations.
\newblock {\em Computer Physics Communications}, 108(2-3):147--158, 1998.

\bibitem{chen2019efficient}
M.~Chen, H.~Jiang, W.~Liao, and T.~Zhao.
\newblock Efficient approximation of deep relu networks for functions on low
  dimensional manifolds.
\newblock {\em Advances in neural information processing systems},
  32:8174--8184, 2019.

\bibitem{chen2019nonparametric}
M.~Chen, H.~Jiang, W.~Liao, and T.~Zhao.
\newblock Nonparametric regression on low-dimensional manifolds using deep relu
  networks.
\newblock {\em arXiv preprint arXiv:1908.01842}, 2019.

\bibitem{chen2020doubly}
M.~Chen, H.~Liu, W.~Liao, and T.~Zhao.
\newblock Doubly robust off-policy learning on low-dimensional manifolds by
  deep neural networks.
\newblock {\em arXiv preprint arXiv:2011.01797}, 2020.

\bibitem{chen1995universal}
T.~Chen and H.~Chen.
\newblock Universal approximation to nonlinear operators by neural networks
  with arbitrary activation functions and its application to dynamical systems.
\newblock {\em IEEE Transactions on Neural Networks}, 6(4):911--917, 1995.

\bibitem{chkifa2015breaking}
A.~Chkifa, A.~Cohen, and C.~Schwab.
\newblock Breaking the curse of dimensionality in sparse polynomial
  approximation of parametric pdes.
\newblock {\em Journal de Math{\'e}matiques Pures et Appliqu{\'e}es},
  103(2):400--428, 2015.

\bibitem{cloninger2020relu}
A.~Cloninger and T.~Klock.
\newblock Relu nets adapt to intrinsic dimensionality beyond the target domain.
\newblock {\em arXiv e-prints}, pages arXiv--2008, 2020.

\bibitem{cohen2015approximation}
A.~Cohen and R.~DeVore.
\newblock Approximation of high-dimensional parametric pdes.
\newblock {\em Acta Numerica}, 24:1--159, 2015.

\bibitem{conway2013sphere}
J.~H. Conway and N.~J.~A. Sloane.
\newblock {\em Sphere packings, lattices and groups}, volume 290.
\newblock Springer Science \& Business Media, 2013.

\bibitem{cybenko1989approximation}
G.~Cybenko.
\newblock Approximation by superpositions of a sigmoidal function.
\newblock {\em Mathematics of control, signals and systems}, 2(4):303--314,
  1989.

\bibitem{de2021convergence}
M.~V. de~Hoop, N.~B. Kovachki, N.~H. Nelsen, and A.~M. Stuart.
\newblock Convergence rates for learning linear operators from noisy data.
\newblock {\em arXiv preprint arXiv:2108.12515}, 2021.

\bibitem{phase}
M.~Deng, S.~Li, A.~Goy, I.~Kang, and G.~Barbastathis.
\newblock Learning to synthesize: robust phase retrieval at low photon counts.
\newblock {\em Light: Science \& Applications}, 9(1):36, 2020.

\bibitem{du2021discovery}
Q.~Du, Y.~Gu, H.~Yang, and C.~Zhou.
\newblock The discovery of dynamics via linear multistep methods and deep
  learning: Error estimation.
\newblock {\em arXiv preprint arXiv:2103.11488}, 2021.

\bibitem{duan2021convergence}
C.~Duan, Y.~Jiao, Y.~Lai, X.~Lu, and Z.~Yang.
\newblock Convergence rate analysis for deep ritz method.
\newblock {\em arxiv:2103.13330}, 2021.

\bibitem{Weinan2019}
W.~E, C.~Ma, and L.~Wu.
\newblock A priori estimates of the population risk for two-layer neural
  networks.
\newblock {\em Communications in Mathematical Sciences}, 17(5):1407--1425,
  2019.

\bibitem{flow}
W.~E, C.~Ma, and L.~Wu.
\newblock The barron space and the flow-induced function spaces for neural
  network models.
\newblock {\em Constructive Approximation}, 2021.

\bibitem{ern2004theory}
A.~Ern and J.-L. Guermond.
\newblock {\em Theory and practice of finite elements}, volume 159.
\newblock Springer, 2004.

\bibitem{MNN}
Y.~Fan, J.~Feliu-Fab{\`a}, L.~Lin, L.~Ying, and L.~Zepeda-N{\'u}{\~n}ez.
\newblock A multiscale neural network based on hierarchical nested bases.
\newblock {\em Research in the Mathematical Sciences}, 6(2):21, 2019.

\bibitem{FAN20191}
Y.~Fan, C.~{Orozco Bohorquez}, and L.~Ying.
\newblock Bcr-net: A neural network based on the nonstandard wavelet form.
\newblock {\em Journal of Computational Physics}, 384:1--15, 2019.

\bibitem{Farrell_2021}
M.~H. Farrell, T.~Liang, and S.~Misra.
\newblock Deep neural networks for estimation and inference.
\newblock {\em Econometrica}, 89(1):181–213, 2021.

\bibitem{federer1959curvature}
H.~Federer.
\newblock Curvature measures.
\newblock {\em Transactions of the American Mathematical Society},
  93(3):418--491, 1959.

\bibitem{giulini2017robust}
I.~Giulini.
\newblock Robust pca and pairs of projections in a hilbert space.
\newblock {\em Electronic Journal of Statistics}, 11(2):3903--3926, 2017.

\bibitem{goodfellow2016deep}
I.~Goodfellow, Y.~Bengio, and A.~Courville.
\newblock {\em Deep learning}.
\newblock MIT press, 2016.

\bibitem{graves2013speech}
A.~Graves, A.-r. Mohamed, and G.~Hinton.
\newblock Speech recognition with deep recurrent neural networks.
\newblock In {\em 2013 IEEE international conference on acoustics, speech and
  signal processing}, pages 6645--6649. IEEE, 2013.

\bibitem{gu2021stationary}
Y.~Gu, J.~Harlim, S.~Liang, and H.~Yang.
\newblock Stationary density estimation of itô diffusions using deep learning.
\newblock {\em arxiv:2109.03992}, 2021.

\bibitem{gyorfi2002distribution}
L.~Gy{\"o}rfi, M.~Kohler, A.~Krzy{\.z}ak, and H.~Walk.
\newblock {\em A distribution-free theory of nonparametric regression},
  volume~1.
\newblock Springer, 2002.

\bibitem{haasdonk2017reduced}
B.~Haasdonk.
\newblock Reduced basis methods for parametrized pdes--a tutorial introduction
  for stationary and instationary problems.
\newblock {\em Model reduction and approximation: theory and algorithms},
  15:65, 2017.

\bibitem{hamers2006nonasymptotic}
M.~Hamers and M.~Kohler.
\newblock Nonasymptotic bounds on the l2 error of neural network regression
  estimates.
\newblock {\em Annals of the Institute of Statistical Mathematics},
  58(1):131--151, 2006.

\bibitem{hinton2012deep}
G.~Hinton, L.~Deng, D.~Yu, G.~Dahl, A.-r. Mohamed, N.~Jaitly, A.~Senior,
  V.~Vanhoucke, P.~Nguyen, and B.~Kingsbury.
\newblock Deep neural networks for acoustic modeling in speech recognition.
\newblock {\em IEEE Signal processing magazine}, 29, 2012.

\bibitem{hornik1991approximation}
K.~Hornik.
\newblock Approximation capabilities of multilayer feedforward networks.
\newblock {\em Neural networks}, 4(2):251--257, 1991.

\bibitem{hotelling1933analysis}
H.~Hotelling.
\newblock Analysis of a complex of statistical variables into principal
  components.
\newblock {\em Journal of educational psychology}, 24(6):417, 1933.

\bibitem{hotelling1992relations}
H.~Hotelling.
\newblock Relations between two sets of variates.
\newblock In {\em Breakthroughs in statistics}, pages 162--190. Springer, 1992.

\bibitem{Arthur18}
A.~Jacot, F.~Gabriel, and C.~Hongler.
\newblock Neural tangent kernel: Convergence and generalization in neural
  networks.
\newblock {\em CoRR}, abs/1806.07572, 2018.

\bibitem{Jiao2021DeepNR}
Y.~Jiao, G.~Shen, Y.~Lin, and J.~Huang.
\newblock Deep nonparametric regression on approximately low-dimensional
  manifolds.
\newblock {\em arXiv: Statistics Theory}, 2021.

\bibitem{khoo2021solving}
Y.~Khoo, J.~Lu, and L.~Ying.
\newblock Solving parametric pde problems with artificial neural networks.
\newblock {\em European Journal of Applied Mathematics}, 32(3):421--435, 2021.

\bibitem{switchnet}
Y.~Khoo and L.~Ying.
\newblock Switchnet: A neural network model for forward and inverse scattering
  problems.
\newblock {\em SIAM Journal on Scientific Computing}, 41(5):A3182--A3201, 2019.

\bibitem{kohler2005adaptive}
M.~Kohler and A.~Krzy{\.z}ak.
\newblock Adaptive regression estimation with multilayer feedforward neural
  networks.
\newblock {\em Nonparametric Statistics}, 17(8):891--913, 2005.

\bibitem{kohler2020estimation}
M.~Kohler, A.~Krzyzak, and S.~Langer.
\newblock Estimation of a function of low local dimensionality by deep neural
  networks.
\newblock {\em arxiv:1908.11140}, 2020.

\bibitem{kovachki2021universal}
N.~Kovachki, S.~Lanthaler, and S.~Mishra.
\newblock On universal approximation and error bounds for fourier neural
  operators.
\newblock {\em Journal of Machine Learning Research}, 22(290):1--76, 2021.

\bibitem{krizhevsky2012imagenet}
A.~Krizhevsky, I.~Sutskever, and G.~E. Hinton.
\newblock Imagenet classification with deep convolutional neural networks.
\newblock In {\em Advances in neural information processing systems}, pages
  1097--1105, 2012.

\bibitem{lanthaler2021error}
S.~Lanthaler, S.~Mishra, and G.~E. Karniadakis.
\newblock Error estimates for deeponets: A deep learning framework in infinite
  dimensions.
\newblock {\em arXiv preprint arXiv:2102.09618}, 2021.

\bibitem{lee2006riemannian}
J.~M. Lee.
\newblock {\em Riemannian manifolds: an introduction to curvature}, volume 176.
\newblock Springer Science \& Business Media, 2006.

\bibitem{li2016characterizing}
D.~Li, Z.~Qiao, and T.~Tang.
\newblock Characterizing the stabilization size for semi-implicit
  fourier-spectral method to phase field equations.
\newblock {\em SIAM Journal on Numerical Analysis}, 54(3):1653--1681, 2016.

\bibitem{li2020fourier}
Z.~Li, N.~Kovachki, K.~Azizzadenesheli, B.~Liu, K.~Bhattacharya, A.~Stuart, and
  A.~Anandkumar.
\newblock Fourier neural operator for parametric partial differential
  equations.
\newblock {\em arXiv preprint arXiv:2010.08895}, 2020.

\bibitem{doi:10.1063/5.0041203}
C.~Lin, Z.~Li, L.~Lu, S.~Cai, M.~Maxey, and G.~E. Karniadakis.
\newblock Operator learning for predicting multiscale bubble growth dynamics.
\newblock {\em The Journal of Chemical Physics}, 154(10):104118, 2021.

\bibitem{hao2021icml}
H.~Liu, M.~Chen, T.~Zhao, and W.~Liao.
\newblock Besov function approximation and binary classification on
  low-dimensional manifolds using convolutional residual networks.
\newblock In {\em International Conference on Machine Learning}, 2021.

\bibitem{lu2021priori2}
J.~Lu and Y.~Lu.
\newblock A priori generalization error analysis of two-layer neural networks
  for solving high dimensional schr\"odinger eigenvalue problems.
\newblock {\em arxiv:2105.01228}, 2021.

\bibitem{lu2021priori}
J.~Lu, Y.~Lu, and M.~Wang.
\newblock A priori generalization analysis of the deep ritz method for solving
  high dimensional elliptic equations.
\newblock {\em arxiv:2101.01708}, 2021.

\bibitem{shijun3}
J.~{Lu}, Z.~{Shen}, H.~{Yang}, and S.~{Zhang}.
\newblock Deep network approximation for smooth functions.
\newblock {\em SIAM Journal on Mathematical Analysis}, to appear.

\bibitem{deeponet}
L.~Lu, P.~Jin, G.~Pang, Z.~Zhang, and G.~E. Karniadakis.
\newblock Learning nonlinear operators via deeponet based on the universal
  approximation theorem of operators.
\newblock {\em Nature Machine Intelligence}, 3(3):218--229, 2021.

\bibitem{Luo2020}
T.~Luo and H.~Yang.
\newblock Two-layer neural networks for partial differential equations:
  Optimization and generalization theory.
\newblock {\em ArXiv}, abs/2006.15733, 2020.

\bibitem{miotto2017deep}
R.~Miotto, F.~Wang, S.~Wang, X.~Jiang, and J.~T. Dudley.
\newblock Deep learning for healthcare: review, opportunities and challenges.
\newblock {\em Briefings in bioinformatics}, 19(6):1236--1246, 2017.

\bibitem{mishra2020estimates}
S.~Mishra and R.~Molinaro.
\newblock Estimates on the generalization error of physics informed neural
  networks (pinns) for approximating pdes.
\newblock {\em arxiv:2006.16144}, 2020.

\bibitem{nakada2020adaptive}
R.~Nakada and M.~Imaizumi.
\newblock Adaptive approximation and generalization of deep neural network with
  intrinsic dimensionality.
\newblock {\em J. Mach. Learn. Res.}, 21:174--1, 2020.

\bibitem{JMLR:v21:20-002}
R.~Nakada and M.~Imaizumi.
\newblock Adaptive approximation and generalization of deep neural network with
  intrinsic dimensionality.
\newblock {\em Journal of Machine Learning Research}, 21(174):1--38, 2020.

\bibitem{nelsen2020random}
N.~H. Nelsen and A.~M. Stuart.
\newblock The random feature model for input-output maps between banach spaces.
\newblock {\em arXiv preprint arXiv:2005.10224}, 2020.

\bibitem{niyogi2008finding}
P.~Niyogi, S.~Smale, and S.~Weinberger.
\newblock Finding the homology of submanifolds with high confidence from random
  samples.
\newblock {\em Discrete \& Computational Geometry}, 39(1-3):419--441, 2008.

\bibitem{orszag1971accurate}
S.~A. Orszag.
\newblock Accurate solution of the orr--sommerfeld stability equation.
\newblock {\em Journal of Fluid Mechanics}, 50(4):689--703, 1971.

\bibitem{pearson1901liii}
K.~Pearson.
\newblock Liii. on lines and planes of closest fit to systems of points in
  space.
\newblock {\em The London, Edinburgh, and Dublin philosophical magazine and
  journal of science}, 2(11):559--572, 1901.

\bibitem{PEHERSTORFER2016196}
B.~Peherstorfer and K.~Willcox.
\newblock Data-driven operator inference for nonintrusive projection-based
  model reduction.
\newblock {\em Computer Methods in Applied Mechanics and Engineering},
  306:196--215, 2016.

\bibitem{resolution}
C.~Qiao, D.~Li, Y.~Guo, C.~Liu, T.~Jiang, Q.~Dai, and D.~Li.
\newblock Evaluation and development of deep neural networks for image
  super-resolution in optical microscopy.
\newblock {\em Nature Methods}, 18(2):194--202, 2021.

\bibitem{QIN2021102028}
Z.~Qin, Q.~Zeng, Y.~Zong, and F.~Xu.
\newblock Image inpainting based on deep learning: A review.
\newblock {\em Displays}, 69:102028, 2021.

\bibitem{rozza2014fundamentals}
G.~Rozza.
\newblock Fundamentals of reduced basis method for problems governed by
  parametrized pdes and applications.
\newblock In {\em Separated representations and PGD-based model reduction},
  pages 153--227. Springer, 2014.

\bibitem{schmidt2019deep}
J.~Schmidt-Hieber.
\newblock Deep relu network approximation of functions on a manifold.
\newblock {\em arXiv preprint arXiv:1908.00695}, 2019.

\bibitem{schmidt2020nonparametric}
J.~Schmidt-Hieber.
\newblock Nonparametric regression using deep neural networks with relu
  activation function.
\newblock {\em The Annals of Statistics}, 48(4):1875--1897, 2020.

\bibitem{schultz1969multivariate}
M.~H. Schultz.
\newblock L\^{}?-multivariate approximation theory.
\newblock {\em SIAM Journal on Numerical Analysis}, 6(2):161--183, 1969.

\bibitem{shaham2018provable}
U.~Shaham, A.~Cloninger, and R.~R. Coifman.
\newblock Provable approximation properties for deep neural networks.
\newblock {\em Applied and Computational Harmonic Analysis}, 44(3):537--557,
  2018.

\bibitem{shen2011spectral}
J.~Shen, T.~Tang, and L.-L. Wang.
\newblock {\em Spectral methods: algorithms, analysis and applications},
  volume~41.
\newblock Springer Science \& Business Media, 2011.

\bibitem{shijun2}
Z.~Shen, H.~Yang, and S.~Zhang.
\newblock Deep network approximation characterized by number of neurons.
\newblock {\em Communications in Computational Physics}, 28(5):1768--1811,
  2020.

\bibitem{shijun7}
Z.~Shen, H.~Yang, and S.~Zhang.
\newblock Deep network approximation: Achieving arbitrary accuracy with fixed
  number of neurons.
\newblock {\em arxiv:2107.02397}, 2021.

\bibitem{shijun4}
Z.~Shen, H.~Yang, and S.~Zhang.
\newblock Deep network with approximation error being reciprocal of width to
  power of square root of depth.
\newblock {\em Neural Computation}, 33(4):1005--1036, 03 2021.

\bibitem{shijun5}
Z.~Shen, H.~Yang, and S.~Zhang.
\newblock Neural network approximation: {T}hree hidden layers are enough.
\newblock {\em Neural Networks}, 141:160--173, 2021.

\bibitem{shijun6}
Z.~Shen, H.~Yang, and S.~Zhang.
\newblock Optimal approximation rate of {ReLU} networks in terms of width and
  depth.
\newblock {\em Journal de Math\'ematiques Pures et Appliqu\'ees}, to appear.

\bibitem{shin2020convergence}
Y.~Shin, J.~Darbon, and G.~E. Karniadakis.
\newblock On the convergence of physics informed neural networks for linear
  second-order elliptic and parabolic type pdes.
\newblock {\em arxiv:2004.01806}, 2020.

\bibitem{siegel2021sharp}
J.~W. Siegel and J.~Xu.
\newblock Sharp bounds on the approximation rates, metric entropy, and
  $n$-widths of shallow neural networks.
\newblock {\em arxiv:2101.12365}, 2021.

\bibitem{10.1214/aos/1176345969}
C.~J. Stone.
\newblock {Optimal Global Rates of Convergence for Nonparametric Regression}.
\newblock {\em The Annals of Statistics}, 10(4):1040 -- 1053, 1982.

\bibitem{suzuki2018adaptivity}
T.~Suzuki.
\newblock Adaptivity of deep relu network for learning in besov and mixed
  smooth besov spaces: optimal rate and curse of dimensionality.
\newblock {\em arXiv preprint arXiv:1810.08033}, 2018.

\bibitem{szeg1939orthogonal}
G.~Szeg.
\newblock {\em Orthogonal polynomials}, volume~23.
\newblock American Mathematical Soc., 1939.

\bibitem{Tian_2020}
C.~Tian, L.~Fei, W.~Zheng, Y.~Xu, W.~Zuo, and C.-W. Lin.
\newblock Deep learning on image denoising: An overview.
\newblock {\em Neural Networks}, 131:251–275, Nov 2020.

\bibitem{loring2011introduction}
L.~W. Tu.
\newblock {\em An introduction to manifolds}.
\newblock Springer., 2011.

\bibitem{van1996weak}
A.~W. Van Der~Vaart, A.~W. van~der Vaart, A.~van~der Vaart, and J.~Wellner.
\newblock {\em Weak convergence and empirical processes: with applications to
  statistics}.
\newblock Springer Science \& Business Media, 1996.

\bibitem{wei2019physics}
Z.~Wei and X.~Chen.
\newblock Physics-inspired convolutional neural network for solving full-wave
  inverse scattering problems.
\newblock {\em IEEE Transactions on Antennas and Propagation},
  67(9):6138--6148, 2019.

\bibitem{yarotsky2017error}
D.~Yarotsky.
\newblock Error bounds for approximations with deep relu networks.
\newblock {\em Neural Networks}, 94:103--114, 2017.

\bibitem{yarotsky18a}
D.~Yarotsky.
\newblock Optimal approximation of continuous functions by very deep {ReLU}
  networks.
\newblock In S.~Bubeck, V.~Perchet, and P.~Rigollet, editors, {\em Proceedings
  of the 31st Conference On Learning Theory}, volume~75 of {\em Proceedings of
  Machine Learning Research}, pages 639--649. PMLR, 06--09 Jul 2018.

\bibitem{yarotsky:2021:02}
D.~Yarotsky.
\newblock Elementary superexpressive activations.
\newblock {\em arXiv e-prints}, 2021.

\bibitem{yarotsky:2019:06}
D.~Yarotsky and A.~Zhevnerchuk.
\newblock The phase diagram of approximation rates for deep neural networks.
\newblock In H.~Larochelle, M.~Ranzato, R.~Hadsell, M.~F. Balcan, and H.~Lin,
  editors, {\em Advances in Neural Information Processing Systems}, volume~33,
  pages 13005--13015. Curran Associates, Inc., 2020.

\bibitem{ZHU2018415}
Y.~Zhu and N.~Zabaras.
\newblock Bayesian deep convolutional encoder–decoder networks for surrogate
  modeling and uncertainty quantification.
\newblock {\em Journal of Computational Physics}, 366:415--447, 2018.

\end{thebibliography}

\appendix
\section*{Appendix}
\section{The derivation for the error bound in Corollary \ref{coro.leg} when $d_{\cX}=d_{\cY}=\log^{\frac{1}{2}} n$}
\label{sec.n0.proof}
From Corollary \ref{coro.leg}, we need to balance the two terms $d_{\cY}^{\frac{4+d_{\cX}}{2+d_{\cX}}}n^{-\frac{2}{2+d_{\cX}}}\log^6 n$ and $d_{\cX}^{-\frac{2s}{D}}$. By setting $d_{\cX}=d_{\cY}=\log^{\frac{1}{2}}n$, the first term decays faster than the second term as $n$ increases. We want to find a lower bound of $n$, denoted by $n_0$, so that when $n>n_0$, the error is dominated by the second term. Note that $n_0$ should satisfy
\begin{align}
	d_{\cY}^{\frac{4+d_{\cX}}{2+d_{\cX}}}n^{-\frac{2}{2+d_{\cX}}}\log^6 n \leq  d_{\cX}^{-\frac{2s}{D}}.
	\label{eq.n0}
\end{align}
Since
\begin{align*}
	d_{\cY}^{\frac{4+d_{\cX}}{2+d_{\cX}}}n^{-\frac{2}{2+d_{\cX}}}\log^6 n \leq d_{\cY}^{2}n^{-\frac{2}{2+d_{\cX}}}\log^6 n\leq  d_{\cX}^{-\frac{2s}{D}},
\end{align*}
in the following, we consider solving 
\begin{align*}
	d_{\cY}^{2}n^{-\frac{2}{2+d_{\cX}}}\log^6 n\leq  d_{\cX}^{-\frac{2s}{D}}.
\end{align*}
Substituting the expression of $d_{\cX}$ and $d_{\cY}$, we deduce
\begin{align*}
	n^{-\frac{2}{2+\log^{1/2}n}}\log^7 n \leq \log ^{-\frac{s}{D}}n \Rightarrow -\frac{2}{2+\log^{\frac1/2}n}\log n +7\log\log n\leq -\frac{s}{D} \log\log n.
\end{align*}
Denote $a=\log n$. We have
\begin{align}
	\frac{2}{2+a^{1/2}}a \geq \left(7 +\frac{s}{D}\right) \log a.
	\label{eq.n0.1}
\end{align}
A sufficient condition of (\ref{eq.n0.1}) is 
\begin{align}
	\frac{2}{a^{1/2}}a \geq \left(7 +\frac{s}{D}\right) \log a \Rightarrow a\geq \frac{1}{4} \left(7 +\frac{s}{D}\right)^2 \log^2 a.
	\label{eq.n0.2}
\end{align}
Note that $\log a<a^{1/3}$ for $a>100$. Therefore, it is sufficient to solve
\begin{align*}
	a\geq \frac{1}{4} \left(7 +\frac{s}{D}\right)^2 a^{\frac{2}{3}} \Rightarrow a\geq \left(\frac{7}{2}+\frac{s}{2D}\right)^{6}.
\end{align*}
Substituting $a$ by $\log n$, one has
\begin{align*}
	n\geq \exp\left(\max\left\{100,\left(\frac{7}{2}+\frac{s}{2D}\right)^{6}\right\}\right).
\end{align*}
\section{Proof of Lemma \ref{lem.EncoDecolip}}
\label{lem.EncoDecolip.proof} 
\begin{proof}[Proof of Lemma \ref{lem.EncoDecolip}]
	We first prove (\ref{eq.deter.encoder.lip}):
	\begin{align*}
		\left\|E_{\cH,d}(u)-E_{\cH,d}(\widetilde{u})\right\|_2^2&=\left\|\left[\langle u-\widetilde{u},\phi_1\rangle_{\cH},...,\langle u-\widetilde{u},\phi_d\rangle_{\cH} \right]^{\top}\right\|_{2}^2\\
		&= \sum_{k=1}^d \left| \langle u-\widetilde{u},\phi_k\rangle_{\cH} \right|^2\\
		&\leq \sum_{k=1}^{\infty}  \left| \langle u-\widetilde{u},\phi_k\rangle_{\cH} \right|^2\\
		&= \|u-\widetilde{u}\|_{\cH}^2.
	\end{align*}	
	For (\ref{eq.deter.decoder.lip}), we have
	\begin{align*}
		\left\|D_{\cH,d}(\ab)- D_{\cH,d}(\widetilde{\ab})\right\|_{\cH}^2&=\left\|\sum_{k=1}^d (a_k-\widetilde{a}_k)\phi_k\right\|_{\cH}^2
		=\left\|\ab-\widetilde{\ab}\right\|_2^2,
	\end{align*}
	since $\{\phi_k\}_{k=1}^d$ is an orthonormal set.
\end{proof}

\section{Proof of Lemma \ref{lem.GammaLip}}
\label{lem.GammaLip.proof}
\begin{proof}[Proof of Lemma \ref{lem.GammaLip}]
	Let $\ab,\tilde{\ab}\in \RR^{d_{\cX}}$. We have
	\begin{align}
		\left\|\Gamma_d^n(\ab)-\Gamma_d^n(\tilde{\ab})\right\|_2=&\left\|E_{\cY}^n\circ \Psi\circ D_{\cX}^n(\ab)-E_{\cY}^n\circ \Psi\circ D_{\cX}^n(\tilde{\ab})\right\|_2 \nonumber\\
		\leq &L_{E^n_{\cY}}\left\|\Psi\circ D_{\cX}^n(\ab)- \Psi\circ D_{\cX}^n(\tilde{\ab})\right\|_2 \nonumber\\
		\leq & L_{E^n_{\cY}}L_{\Psi}\left\|D_{\cX}^n(\ab)-  D_{\cX}^n(\tilde{\ab})\right\|_{\cY} \nonumber\\
		\leq& L_{E^n_{\cY}}L_{D^n_{\cX}}L_{\Psi}\|\ab-\tilde{\ab}\|_2.
	\end{align}
\end{proof}

\section{Proof of Lemma \ref{lem.T1.2}}\label{lem.T1.2.proof}
\begin{proof}[Proof of Lemma \ref{lem.T1.2}]
	We prove Lemma \ref{lem.T1.2} using the covering number of $\cF_{\rm NN}$. Let $\cF^*=\left\{\Gamma^*_j\right\}_{j=1}^{\cN\left(\delta,\cF_{\rm NN},\|\cdot\|_{\infty}\right)}$ be a $\delta$-cover of $\cF_{\rm NN}$, where $\cN\left(\delta,\cF_{\rm NN},\|\cdot\|_{\infty}\right)$ is the covering number. Then there exists $\Gamma^*\in \cF^*$ satisfying $\left\|\Gamma^*-\Gamma_{\rm NN}\right\|_{\infty}\leq \delta$, where $\Gamma_{\rm NN}$ is our estimator in (\ref{eq.PsiNN}). Denote $\left\|\Gamma\circ E_{\cX}^n\right\|_n^2=\frac{1}{n}\sum_{i=n+1}^{2n} \left\|\Gamma\circ E_{\cX}^n(u_i)\right\|_2^2$. We have
	\begin{align}
		&\EE_{\cS_2}\left[\frac{1}{n}\sum_{i=n+1}^{2n}\left\langle \Gamma_{\rm NN} \circ E_{\cX}^n(u_i), \bepsilon_i\right\rangle \right] \nonumber\\
		=& \EE_{\cS_2}\left[\frac{1}{n}\sum_{i=n+1}^{2n} \left\langle \Gamma_{\rm NN} \circ E_{\cX}^n(u_i)-\Gamma^* \circ E_{\cX}^n(u_i)+\Gamma^* \circ E_{\cX}^n(u_i)-\Gamma_{d}^n \circ E_{\cX}^n(u_i), \bepsilon_i\right\rangle \right]  \nonumber\\
		\leq &  \EE_{\cS_2}\left[\frac{1}{n}\sum_{i=n+1}^{2n} \left\langle \Gamma^* \circ E_{\cX}^n(u_i)-\Gamma_{d}^n \circ E_{\cX}^n(u_i), \bepsilon_i\right\rangle \right] + \EE_{\cS_2} \left[\frac{1}{n}\sum_{i=n+1}^{2n}\left\|\Gamma_{\rm NN} \circ E_{\cX}^n(u_i)-\Gamma^* \circ E_{\cX}^n(u_i)\|_2\|\bepsilon_i\right\|_2\right] \nonumber\\
		\leq & \EE_{\cS_2} \left[\frac{\left\|\Gamma^* \circ E_{\cX}^n-\Gamma_{d}^n \circ E_{\cX}^n\right\|_n}{\sqrt{n}} \frac{\sum_{i=n+1}^{2n} \left\langle \Gamma^* \circ E_{\cX}^n(u_i)-\Gamma_{d}^n \circ E_{\cX}^n(u_i), \bepsilon_i\right\rangle}{\sqrt{n}\left\|\Gamma^* \circ E_{\cX}^n-\Gamma_{d}^n \circ E_{\cX}^n\right\|_n}\right] +d_{\cY}\sigma\delta \nonumber\\
		\leq & \sqrt{2}\EE_{\cS_2} \left[\frac{\left\|\Gamma_{\rm NN} \circ E_{\cX}^n(u_i)-\Gamma_{d}^n \circ E_{\cX}^n(u_i)\right\|_n+\sqrt{d_{\cY}}\delta}{\sqrt{n}} \left|\frac{\sum_{i=n+1}^{2n} \left\langle \Gamma^* \circ E_{\cX}^n(u_i)-\Gamma_{d}^n \circ E_{\cX}^n(u_i), \bepsilon_i\right\rangle}{\sqrt{n}\left\|\Gamma^* \circ E_{\cX}^n-\Gamma_{d}^n \circ E_{\cX}^n\right\|_n}\right|\right] +d_{\cY}\sigma\delta,
		\label{eq.T1.2}
	\end{align}
	where the first inequality follows from Cauchy--Schwarz inequality, the third inequality holds since
	\begin{align}
		&\left\|\Gamma^* \circ E_{\cX}^n-\Gamma_{d}^n \circ E_{\cX}^n\right\|_n \nonumber\\
		=& \sqrt{\frac{1}{n}\sum_{i=n+1}^{2n}\left\|\Gamma^* \circ E_{\cX}^n(u_i)-\Gamma_{\rm NN} \circ E_{\cX}^n(u_i)+\Gamma_{\rm NN} \circ E_{\cX}^n(u_i)-\Gamma_{d}^n \circ E_{\cX}^n(u_i)\right\|_2^2} \nonumber\\
		\leq &\sqrt{\frac{2}{n}\sum_{i=n+1}^{2n}\left\|\Gamma^* \circ E_{\cX}^n(u_i)-\Gamma_{\rm NN} \circ E_{\cX}^n(u_i)\right\|_2^2+\left\|\Gamma_{\rm NN} \circ E_{\cX}^n(u_i)-\Gamma_{d}^n \circ E_{\cX}^n(u_i)\right\|_2^2} \nonumber\\
		\leq &\sqrt{\frac{2}{n}\sum_{i=n+1}^{2n}d_{\cY}\delta^2+\left\|\Gamma_{\rm NN} \circ E_{\cX}^n(u_i)-\Gamma_{d}^n \circ E_{\cX}^n(u_i)\right\|_2^2} \nonumber\\
		\leq &\sqrt{2}\left\|\Gamma_{\rm NN} \circ E_{\cX}^n(u_i)-\Gamma_{d}^n \circ E_{\cX}^n(u_i)\right\|_n +\sqrt{2d_{\cY}}\delta.
	\end{align}
	
	Denote $z_j=\frac{\sum_{i=n+1}^{2n} \left\langle \Gamma_j^* \circ E_{\cX}^n(u_i)-\Gamma_{d}^n \circ E_{\cX}^n(u_i), \bepsilon_i\right\rangle}{\sqrt{n}\left\|\Gamma_j^* \circ E_{\cX}^n-\Gamma_{d}^n \circ E_{\cX}^n\right\|_n}$. The expectation term in (\ref{eq.T1.2}) can be bounded as
	\begin{align}
		&\EE_{\cS_2} \left[\frac{\left\|\Gamma_{\rm NN} \circ E_{\cX}^n(u_i)-\Gamma_{d}^n \circ E_{\cX}^n(u_i)\right\|_n+\sqrt{d_{\cY}}\delta}{\sqrt{n}} \left|\frac{\sum_{i=n+1}^{2n} \left\langle \Gamma^* \circ E_{\cX}^n(u_i)-\Gamma_{d}^n \circ E_{\cX}^n(u_i), \bepsilon_i\right\rangle}{\sqrt{n}\left\|\Gamma^* \circ E_{\cX}^n-\Gamma_{d}^n \circ E_{\cX}^n\right\|_n}\right|\right] \nonumber\\
		\leq & \EE_{\cS_2} \left[ \frac{\left\|\Gamma_{\rm NN} \circ E_{\cX}^n(u_i)-\Gamma_{d}^n \circ E_{\cX}^n(u_i)\right\|_n+\sqrt{d_{\cY}}\delta}{\sqrt{n}} \max_j |z_j|\right]\nonumber\\
		=& \EE_{\cS_2} \left[\frac{\left\|\Gamma_{\rm NN} \circ E_{\cX}^n(u_i)-\Gamma_{d}^n \circ E_{\cX}^n(u_i)\right\|_n}{\sqrt{n}}\max_j |z_j|+\frac{\sqrt{d_{\cY}}\delta}{\sqrt{n}} \max_j |z_j|\right]\nonumber\\
		\leq& \EE_{\cS_2}  \left[\sqrt{\frac{1}{n}\left\|\Gamma_{\rm NN} \circ E_{\cX}^n(u_i)-\Gamma_{d}^n \circ E_{\cX}^n(u_i)\right\|_n^2}\sqrt{\max_j |z_j|^2}+\frac{\sqrt{d_{\cY}}\delta}{\sqrt{n}} \sqrt{\max_j |z_j|^2}\right]\nonumber\\
		\leq & \sqrt{\frac{1}{n}\EE_{\cS_2}\left[\left\|\Gamma_{\rm NN} \circ E_{\cX}^n(u_i)-\Gamma_{d}^n \circ E_{\cX}^n(u_i)\right\|_n^2\right]}\sqrt{\EE_{\cS_2}\left[\max_j |z_j|^2\right]}+\frac{\sqrt{d_{\cY}}\delta}{\sqrt{n}} \sqrt{\EE_{\cS_2}\left[\max_j |z_j|^2\right]} \nonumber\\
		= & \left(\sqrt{\frac{1}{n}\EE_{\cS_2}\left[\left\|\Gamma_{\rm NN} \circ E_{\cX}^n(u_i)-\Gamma_{d}^n \circ E_{\cX}^n(u_i)\right\|_n^2\right]}+\frac{\sqrt{d_{\cY}}\delta}{\sqrt{n}} \right)\sqrt{\EE_{\cS_2}\left[\max_j |z_j|^2\right]}.
		\label{eq.T1.3}
	\end{align}
	where the second inequality comes from Cauchy--Schwarz inequality, the third inequality comes from Jensen's inequality.
	
	Since $\bepsilon_i\in [-\sigma,\sigma]^{d_{\cY}}$, each component of $\bepsilon_i$ is a sub-Gaussian variable with parameter $\sigma$. Therefore for given $u_{n+1},...,u_{2n}$, each $z_j$ is a sub-gaussian variable with parameter $\sqrt{d_{\cY}}\sigma$. The last term is the maximum of a collection of squared sub-Gaussian variables and is bounded as 
	\begin{align}
		\EE_{\cS_2}\left[\max_j |z_j|^2|u_{n+1},...,u_{2n}\right]=&\frac{1}{t}\log \exp\left(t\EE_{\cS_2}\left[\max_j |z_j|^2| u_{n+1},...,u_{2n}\right] \right)  \nonumber\\
		\leq & \frac{1}{t}\log \EE_{\cS_2}\left[\exp\left(t\max_j |z_j|^2| u_{n+1},...,u_{2n} \right)\right] \nonumber \\
		\leq& \frac{1}{t}\log \EE_{\cS_2}\left[\sum_j\exp\left(t |z_j|^2| u_{n+1},...,u_{2n} \right)\right] \nonumber \\
		\leq & \frac{1}{t}\log \cN\left(\delta,\cF_{\rm NN}, \|\cdot\|_{\infty}\right)+\frac{1}{t}\log\EE_{\cS_2}\left[\exp\left(t |z_1|^2| u_{n+1},...,u_{2n} \right) \right].
	\end{align}
	Since $z_1$ is sub-Gaussian with parameter $\sigma^2$, we have
	\begin{align}
		\EE_{\cS_2}\left[\exp\left(t |z_1|^2| u_{n+1},...,u_{2n}\right)\right]=&1+\sum_{k=1}^{\infty} \frac{t^k\EE_{\cS_2} \left[z_1^{2k}|u_{n+1},...,u_{2n} \right]}{k!} \nonumber\\
		=& 1+\sum_{k=1}^{\infty} \frac{t^k}{k!}\int_0^{\infty} \PP\left(|z_1|\geq \tau^{\frac{1}{2k}}|u_{n+1},...,u_{2n}\right)d\tau \nonumber\\
		\leq& 1+2\sum_{k=1}^{\infty} \frac{t^k}{k!}\int_0^{\infty} \exp\left(-\frac{\tau^{1/k}}{2d_{\cY}\sigma^2}\right)d\tau \nonumber\\
		=& 1+\sum_{k=1}^{\infty} \frac{2k(2td_{\cY}\sigma^2)^k}{k!}\Gamma_{\rm G}(k) \nonumber\\
		=& 1+2\sum_{k=1}^{\infty} (2td_{\cY}\sigma^2)^k,
	\end{align}
	where $\Gamma_{\rm G}$ represents the Gamma function. Setting $t=(4d_{\cY}\sigma^2)^{-1}$ gives rise to
	\begin{align}
		\EE_{\cS_2}\left[\max_j |z_j|^2|u_{n+1},...,u_{2n}\right] & \leq 4d_{\cY}\sigma^2 \log \cN(\delta,\cF_{\rm NN}, \|\cdot\|_{\infty}) +4d_{\cY}\sigma^2\log 3 \nonumber\\
		&\leq  4d_{\cY}\sigma^2 \log \cN(\delta,\cF_{\rm NN}, \|\cdot\|_{\infty}) +6d_{\cY}\sigma^2.
		\label{eq.T1.4}
	\end{align}
	
	Combining (\ref{eq.T1.4}), (\ref{eq.T1.3}), (\ref{eq.T1.2}) finishes the proof.  
	
\end{proof}

\section{Proof of Lemma \ref{lem.T2}}\label{lem.T2.proof}

\begin{proof}[Proof of Lemma \ref{lem.T2}]
	Our proof follows the proof of \citep[Lemma 4.2]{chen2019nonparametric}. Denote $g(u)= \|\Gamma_{\rm NN} \circ E_{\cX}^n(u)- E_{\cY}^n\circ\Psi(u)\|^2_2$. We have $\|g\|_{\infty}\leq 4d_{\cY}L_{E^n_{\cY}}^2R_{\cY}^2$. Then
	\begin{align}
		{\rm T_2}=&\EE_{\cS_2}\left[\EE_{u\sim \gamma}\left[g(u)|\cS_1\right]-\frac{2}{n}\sum_{i=n+1}^{2n} g(u_i)\right]\nonumber\\
		=&2\EE_{\cS_2}\left[\frac{1}{2}\EE_{u\sim \gamma}\left[g(u)|\cS_1\right]-\frac{1}{n}\sum_{i=n+1}^{2n} g(u_i)\right]\nonumber\\
		=&2\EE_{\cS_2}\left[\EE_{u\sim \gamma}[g(u)|\cS_1]-\frac{1}{n}\sum_{i=n+1}^{2n} g(u_i)-\frac{1}{2}\EE_{u\sim \gamma}\left[g(u)|\cS_1\right]\right].
		\label{eq.T2.1}
	\end{align}
	A lower bound of $\frac{1}{2}\EE_{u\sim \gamma}\left[g(u)|\cS_1\right]$ can be derived as
	\begin{align}
		\EE_{u\sim \gamma}\left[g(u)|\cS_1\right]= \EE_{u\sim \gamma}\left[\frac{4d_{\cY}L_{E^n_{\cY}}^2R_{\cY}^2}{4d_{\cY}L_{E^n_{\cY}}^2R_{\cY}^2}g(u)|\cS_1\right]
		\geq  \frac{1}{4d_{\cY}L_{E^n_{\cY}}^2R_{\cY}^2}\EE_{u\sim \gamma}\left[g^2(u)|\cS_1\right]
		\label{eq.g.lower}
	\end{align}
	Substituting (\ref{eq.g.lower}) into (\ref{eq.T2.1}) gives
	\begin{align}
		{\rm T_2}\leq 2\EE_{\cS_2}\left[\EE_{u\sim \gamma}[g(u)|\cS_1]-\frac{1}{n}\sum_{i=n+1}^{2n} g(u_i)-\frac{1}{8d_{\cY}L_{E^n_{\cY}}^2R_{\cY}^2}\EE_{u\sim \gamma}\left[g^2(u)|\cS_1\right]\right].
	\end{align}
	Define the set
	\begin{align}
		\cR=\left\{ g(u)=\|\Gamma \circ E_{\cX}^n(u)- E_{\cY}^n\circ\Psi(u)\|^2_2: \Gamma\in \cF_{\rm NN}\right\}.
	\end{align}
	Denote $\cS_2'=\{u'_i\}_{i=n+1}^{2n}$  as an independent copy of $\cS_2$. We rewrite ${\rm T_2}$ as
	\begin{align}
		{\rm T_2}\leq& 2\EE_{\cS_2}\left[\sup_{g\in \cR}\left(\EE_{\cS_2'}\left[\frac{1}{n}\sum_{i=n+1}^{2n} g(u'_i)\right]\right)-\frac{1}{n}\sum_{i=n+1}^{2n} g(u_i)-\frac{1}{8d_{\cY}L_{E^n_{\cY}}^2R_{\cY}^2}\left(\EE_{\cS_2'}\left[\frac{1}{n}\sum_{i=n+1}^{2n} g^2(u'_i)\right]\right)\right]\nonumber\\
		\leq &2\EE_{\cS_2}\left[\sup_{g\in \cR}\left(\EE_{\cS_2'}\left[\frac{1}{n}\sum_{i=n+1}^{2n} (g(u'_i)-g(u_i))\right]\right)-\frac{1}{16d_{\cY}L_{E^n_{\cY}}^2R_{\cY}^2}\EE_{\cS_2,\cS'_2}\left[\frac{1}{n}\sum_{i=n+1}^{2n} (g^2(u_i)+g^2(u'_i))\right]\right]\nonumber\\
		\leq & 2\EE_{\cS_2,\cS'_2}\left[\sup_{g\in \cR}\left(\frac{1}{n}\sum_{i=n+1}^{2n} \left((g(u_i)-g(\bar{u}_i))-\frac{1}{16d_{\cY}L_{E^n_{\cY}}^2R_{\cY}^2}\left[g^2(u_i)+g^2(u'_i)\right]\right)\right)\right].
		\label{eq.T2.2}
	\end{align}
	
	Let $\cR^*=\{g_i^*\}_{i=1}^{\cN(\delta,\cR,\|\cdot\|_{\infty})}$ be a $\delta$-cover of $\cR$. Then for any $g\in \cR$, there exists $g^*\in \cR^*$ such that $\|g-g^*\|_{\infty}\leq \delta$.
	
	We next bound (\ref{eq.T2.2}) using $g^*$'s. For the first term in (\ref{eq.T2.2}), we have
	\begin{align}
		g(u_i)-g(u'_i)=& g(u_i)-g^*(u_i)+g^*(u_i)-g^*(u'_i)+g^*(u'_i)-g(u'_i) \nonumber\\
		=&\left(g(u_i)-g^*(u_i)\right) +\left(g^*(u_i)-g^*(u'_i)\right)+\left(g^*(u'_i)-g(u'_i)\right)\nonumber\\
		\leq &\left(g^*(u_i)-g^*(u'_i)\right)+2\delta.
		\label{eq.T2.term1}
	\end{align}
	We lower bound $g^2(u_i)+g^2(u'_i)$ as
	\begin{align}
		g^2(u_i)+g^2(u'_i)=&\left(g^2(u_i)-(g^*)^2(u_i)\right)+\left((g^*)^2(u_i)+(g^*)^2(u'_i)\right)-\left((g^*)^2(u'_i)- g^2(u'_i)\right) \nonumber\\
		\geq& (g^*)^2(u_i)+(g^*)^2(u'_i) -\left|g(u_i)-g^*(u_i)\right|\left|g(u_i)+g^*(u_i)\right| -\left|g^*(u'_i)- g(u'_i)\right|\left|g^*(u'_i)+ g(u'_i)\right| \nonumber\\
		\geq &(g^*)^2(u_i)+(g^*)^2(u'_i)-16d_{\cY}L_{E^n_{\cY}}^2R_{\cY}^2\delta.
		\label{eq.T2.term2}
	\end{align}
	Substituting (\ref{eq.T2.term1}) and (\ref{eq.T2.term2}) into (\ref{eq.T2.2}) gives rise to
	\begin{align}
		{\rm T_2}\leq&2\EE_{\cS_2,\cS'_2}\left[\sup_{g^*\in \cR^*}\left(\frac{1}{n}\sum_{i=n+1}^{2n}\left( (g^*(u_i)-g^*(u'_i))-\frac{1}{16d_{\cY}L_{E^n_{\cY}}^2R_{\cY}^2}\left[(g^*)^2(u_i)+(g^*)^2(u'_i)\right]\right)\right)\right] + 6\delta \nonumber\\
		=&2\EE_{\cS_2,\cS'_2}\left[\max_j\left(\frac{1}{n}\sum_{i=n+1}^{2n}\left( (g_j^*(u_i)-g_j^*(u'_i))-\frac{1}{16d_{\cY}L_{E^n_{\cY}}^2R_{\cY}^2}\left[(g_j^*)^2(u)+(g_j^*)^2(u'_i)\right]\right)\right)\right] + 6\delta.
	\end{align}
	
	Denote $h_j=(u_i,u'_i,\xi_i)=(g_j^*(u_i)-g_j^*(u'_i))$. We have
	\begin{align*}
		&\EE_{\cS_2,\cS'_2} [h_j(u_i,u'_i)]=0,\\
		&\Var [h_j(u_i,u'_i)]=\EE \left[ h_j^2(u_i,u'_i)\right]\\
		&\hspace{2.88cm}=\EE_{\cS_2,\cS'_2} \left[(g_j^*(u_i)-g_j^*(u'_i))^2\right]\\
		&\hspace{2.88cm}\leq  2\EE_{\cS_2,\cS'_2} \left[ (g_j^*)^2(u_i)+(g_j^*)^2(u'_i) \right].
	\end{align*}
	Thus ${\rm T_2}$ can be bounded as
	\begin{align*}
		&{\rm T_2}\leq {\rm \widetilde{T}_2}+6\delta\\
		&\mbox{with } {\rm \widetilde{T}_2}=2\EE_{\cS_2,\cS'_2}\left[\max_j\left(\frac{1}{n}\sum_{i=n+1}^{2n} \left( h_j(u_i,u'_i)-\frac{1}{32d_{\cY}L_{E^n_{\cY}}^2R_{\cY}^2}\Var [h_j(u_i,u'_i)]\right)\right)\right].
	\end{align*}
	Note that $\|h_j\|_{\infty}\leq 4d_{\cY}L_{E^n_{\cY}}^2R_{\cY}^2$. We next derive the moment generating function of $h_j$. For any $0<t<\frac{3}{4d_{\cY}L_{E^n_{\cY}}^2R_{\cY}^2}$, we have
	\begin{align}
		\EE_{\cS_2,\cS'_2}\left[\exp(th_j(u_i,u'_i))\right]=& \EE_{\cS_2,\cS'_2}\left[ 1+th_j(u_i,u'_i)+\sum_{k=2}^{\infty} \frac{t^kh_j^k(u_i,u'_i)}{k!}\right] \nonumber\\
		\leq & \EE_{\cS_2,\cS'_2}\left[ 1+th_j(u_i,u'_i)+\sum_{k=2}^{\infty} \frac{(4d_{\cY}L_{E^n_{\cY}}^2R_{\cY}^2)^{k-2}t^kh_j^2(u_i,u'_i)}{2\times 3^{k-2}}\right] \nonumber\\
		=& \EE_{\cS_2,\cS'_2}\left[ 1+th_j(u_i,u'_i)+\frac{t^2h_j^2(u_i,u'_i)}{2}\sum_{k=2}^{\infty} \frac{(4d_{\cY}L_{E^n_{\cY}}^2R_{\cY}^2)^{k-2}t^{k-2}}{3^{k-2}}\right] \nonumber\\
		=& \EE_{\cS_2,\cS'_2}\left[ 1+th_j(u_i,u'_i)+\frac{t^2h_j^2(u_i,u'_i)}{2}\frac{1}{1-4d_{\cY}L_{E^n_{\cY}}^2R_{\cY}^2t/3}\right] \nonumber\\
		=& 1+t^2\Var[h_j(u_i,u'_i)]\frac{1}{2-8d_{\cY}L_{E^n_{\cY}}^2R_{\cY}^2t/3} \nonumber\\
		\leq& \exp\left( \Var[h_j(u_i,u'_i)]\frac{3t^2}{6-8d_{\cY}L_{E^n_{\cY}}^2R_{\cY}^2t} \right),
		\label{eq.h.var}
	\end{align}
	where the last inequality comes from $1+x\leq \exp(x)$ for $x\geq 0$.
	
	Then for $0<t/n<\frac{3}{4d_{\cY}L_{E^n_{\cY}}^2R_{\cY}^2}$, we have
	\begin{align}
		&\exp\left(\frac{t{\rm \tilde{T}_2}}{2}\right) \nonumber\\
		=& \exp\left( t\EE_{\cS_2,\cS'_2}\left[\max_j \left(\frac{1}{n}\sum_{i=n+1}^{2n} h_j(u_i,u'_i)-\frac{1}{32d_{\cY}L_{E^n_{\cY}}^2R_{\cY}^2}\frac{1}{n}\sum_{i=n+1}^{2n}\Var [h_j(u_i,u'_i)]\right)\right] \right) \nonumber\\
		\leq & \EE_{\cS_2,\cS'_2}\left[\exp\left( t\max_j\left( \frac{1}{n}\sum_{i=n+1}^{2n} h_j(u_i,u'_i)-\frac{1}{32d_{\cY}L_{E^n_{\cY}}^2R_{\cY}^2}\frac{1}{n}\sum_{i=n+1}^{2n}\Var [h_j(u_i,u'_i)]\right)\right)\right] \nonumber\\
		\leq &  \EE_{\cS_2,\cS'_2}\left[\sum_{j}\exp\left( \frac{t}{n}\sum_{i=n+1}^{2n} h_j(u_i,u'_i)-\frac{t}{32d_{\cY}L_{E^n_{\cY}}^2R_{\cY}^2}\frac{1}{n}\sum_{i=n+1}^{2n}\Var [h_j(u_i,u'_i)]\right)\right] \nonumber\\
		\leq & \left[\sum_{j}\exp\left( \sum_{i=n+1}^{2n} \Var[h_j(u_i,u'_i)]\frac{3t^2/n^2}{6-8d_{\cY}L_{E^n_{\cY}}^2R_{\cY}^2t/n}-\frac{1}{32d_{\cY}L_{E^n_{\cY}}^2R_{\cY}^2}\frac{t}{n}\Var [h_j(u_i,u'_i)]\right)\right] \nonumber\\
		=& \left[\sum_{j}\exp\left( \sum_{i=n+1}^{2n} \frac{t}{n}\Var[h_j(u_i,u'_i)]\left(\frac{3t/n}{6-8d_{\cY}L_{E^n_{\cY}}^2R_{\cY}^2t/n}-\frac{1}{32d_{\cY}L_{E^n_{\cY}}^2R_{\cY}^2}\right)\right)\right],
		\label{eq.exptT}
	\end{align}
	where the first inequality follows from Jensen's inequality and the third inequality uses (\ref{eq.h.var}). Setting
	$$
	\frac{3t/n}{6-8d_{\cY}L_{E^n_{\cY}}^2R_{\cY}^2t/n}-\frac{1}{32d_{\cY}L_{E^n_{\cY}}^2R_{\cY}^2}=0
	$$
	gives $t=\frac{3n}{52d_{\cY}L_{E^n_{\cY}}^2R_{\cY}^2}<\frac{3n}{4d_{\cY}L_{E^n_{\cY}}^2R_{\cY}^2} $. Substituting our choice of $t$ into (\ref{eq.exptT}) gives
	\begin{align*}
		\frac{t{\rm \tilde{T}_2}}{2}\leq \log \sum_j \exp(0).
	\end{align*}
	Therefore
	\begin{align*}
		{\rm \tilde{T}_2}\leq \frac{2}{t}\log\cN(\delta,\cR,\|\cdot\|_{\infty})=\frac{104d_{\cY}L_{E^n_{\cY}}^2R_{\cY}^2}{3n}\log \cN(\delta,\cR,\|\cdot\|_{\infty})
	\end{align*}
	and 
	$$
	{\rm T_2}\leq \frac{104d_{\cY}L_{E^n_{\cY}}^2R_{\cY}^2}{3n}\log \cN(\delta,\cR,\|\cdot\|_{\infty})+6\delta\leq \frac{35d_{\cY}L_{E^n_{\cY}}^2R_{\cY}^2}{n}\log \cN(\delta,\cR,\|\cdot\|_{\infty})+6\delta.
	$$
	We next derive a relation between the covering number of $\cF_{\rm NN}$ and $\cR$. For any $g,\widetilde{g}\in \cR$, we have
	$$
	g(u)=\left\|\Gamma \circ E_{\cX}^n(u)- E_{\cY}^n\circ\Psi(u)\right\|^2_2, \ \widetilde{g}(u)=\left\|\widetilde{\Gamma} \circ E_{\cX}^n(u)- E_{\cY}^n\circ\Psi(u)\right\|^2_2
	$$
	for some $\Gamma,\widetilde{\Gamma}\in \cF_{\rm NN}$. We have
	\begin{align*}
		\left\|g-\widetilde{g}\right\|_{\infty}=&\sup_u \left|\left\|\Gamma \circ E_{\cX}^n(u)- E_{\cY}^n\circ\Psi(u)\right\|^2_2-\left\|\widetilde{\Gamma} \circ E_{\cX}^n(u)-  E_{\cY}^n\circ\Psi(u)\right\|^2_2\right| \nonumber\\
		=&\sup_u \left|\left\langle \Gamma \circ E_{\cX}^n(u)-\widetilde{\Gamma} \circ E_{\cX}^n(u), \Gamma \circ E_{\cX}^n(u)+\widetilde{\Gamma} \circ E_{\cX}^n(u)-2E_{\cY}^n\circ\Psi(u)\right\rangle\right| \nonumber\\
		\leq & \sup_u \left\| \Gamma \circ E_{\cX}^n(u)-\widetilde{\Gamma} \circ E_{\cX}^n(u)\right\|_2 \left\| \Gamma \circ E_{\cX}^n(u)+\widetilde{\Gamma} \circ E_{\cX}^n(u)-2E_{\cY}^n\circ\Psi(u)\right\|_2 \nonumber\\
		\leq& 4d_{\cY}L_{E^n_{\cY}}R_{\cY}\left\| \Gamma -\widetilde{\Gamma} \right\|_{\infty}.
	\end{align*}
	As a result, we have
	$$
	\cN(\delta,\cR,\|\cdot\|_{\infty})\leq \cN\left(\frac{\delta}{4d_{\cY}L_{E^n_{\cY}}R_{\cY}},\cF_{\rm NN},\|\cdot\|_{\infty}\right).
	$$
	and Lemma \ref{lem.T2} is proved.
\end{proof}

\section{Proof of Lemma \ref{lem.T2.dense}} \label{lem.T2.dense.proof}
Lemma \ref{lem.T2.dense} can be proved similarly to Lemma \ref{lem.T2}. 
Denote $g(u)= \left\|\Gamma_{\rm NN} \circ E_{\cX}^n(u)- E_{\cY}^n\circ\Psi(u)\right\|^2_2$ and let $\cS_2'=\{u_i'\}_{i=n+1}^{2n}$ be an independent copy of $\cS_2$. Following the proof of Lemma \ref{lem.T2} up to (\ref{eq.T2.2}) and replacing $\EE_{u\sim \gamma}[g(u)|\cS_1]$ by $\EE_{\cS_2'}\left[\frac{1}{n}\sum_{i=n+1}^{2n}g(u_i')\right]$, we can derive 
\begin{align}
	{\rm T_2}\leq 2\EE_{\cS_2,\cS_2'}\left[\sup_{g\in \cR}\left(\frac{1}{n}\sum_{i=n+1}^{2n} (g(u_i)-g(u'_i))-\frac{1}{16d_{\cY}L_{E^n_{\cY}}^2R_{\cY}^2}\frac{1}{n}\sum_{i=n+1}^{2n}\left(g^2(u_i)+g^2(u_i')\right)\right)\right],
	\label{eq.T2.2.dense}
\end{align}

Let $\cR^*=\{g_i^*\}_{i=1}^{\cN(\delta,\cR,2n)}$ be a $\delta$-cover of $\cR$ with respect to the data set $\widetilde{\cS}=\{u_i\}_{i=1}^n\cup\{u'_i\}_{i=1}^n$. Then for any $g\in \cR$, there exists $g^*\in \cR^*$ such that $|g(u)-g^*(u)|\leq \delta, \forall u\in \widetilde{\cS}$.

Lemma \ref{lem.T2.dense}  can be proved by following the rest proof of Lemma \ref{lem.T2}.

\section{Proof of Lemma \ref{lem.pca.PiY}}\label{lem.pca.PiY.proof}
The proof of Lemma \ref{lem.pca.PiY} replies on the perturbation theory of operators on separable Hilbert spaces, which is stated in the following lemma:
\begin{lemma}[Proposition 2.1 of \citep{giulini2017robust}] \label{lem.PCAHilbert}
	Let $A,\widetilde{A}$ be two compact self-adjoint nonnegative operators on the separable real Hilbert space $\cH$. Denote the eigenvalues of $A$ and $\widetilde{A}$ in non-increasing order by $\{\lambda_1,\lambda_2,...\}$ and $\{\widetilde{\lambda}_1,\widetilde{\lambda}_2,...\}$, respectively. For some integer $d>0$, let $\Pi_{\cH,d}$ and $\widetilde{\Pi}_{\cH,d}$ be the projectors that project any $u\in\cH$ to the space spanned by the eigenfunctions corresponding to the largest $d$ eigenvalues of $A$ and $\widetilde{A}$, respectively. 
	We have
	\begin{align}
		&\left\|\Pi_{\cH,d}-\widetilde{\Pi}_{\cH,d}\right\|_{\rm HS} \leq \frac{\sqrt{2}\left\|A-\widetilde{A}\right\|_{\rm HS}}{\max\left\{\lambda_d-\lambda_{d+1}, \widetilde{\lambda}_d-\widetilde{\lambda}_{d+1}\right\}}.
	\end{align}
\end{lemma}

\begin{proof}[Proof of Lemma \ref{lem.pca.PiY}]
	Denote $w=\Psi(u)$. Recall that $\zeta$ is the probability measure of $v=\Psi(u)+\widetilde{\epsilon}$. We have 
	\begin{align}
		G_{\zeta}=\EE_{\{v_i\}_{i=1}^n\sim \zeta} \left[G_{\zeta}^n\right]=&\EE_{v\sim \zeta} [v\otimes v]\nonumber\\
		=&\EE_{w \sim \Psi_{\#}\gamma,\widetilde{\epsilon}\sim \mu} \left[(w+\widetilde{\epsilon})\otimes (w+\widetilde{\epsilon})\right] \nonumber\\
		=& \EE_{w\sim \Psi_{\#}\gamma} \left[w\otimes w\right] +\EE_{\widetilde{\epsilon}\sim \mu} \left[\widetilde{\epsilon}\otimes\widetilde{\epsilon}\right]\nonumber\\
		=&G_{\Psi_{\#}\gamma}+G_{\mu},
		\label{eq.GkappaN}
	\end{align}
	where the third equality holds since $w$ and $\widetilde{\epsilon}$ are independent and $\EE \widetilde{\epsilon}=0$. 
	Recall that $\Pi_{\cY,d_{\cY}}$ (resp. $\Pi_{\cY,d_{\cY}}^n$) projects any $w\in \cY$ to the space spanned by the first $d_{\cY}$ principal eigenfunctions of $G_{\Psi_{\#}\gamma}$ (resp. $G_{\zeta}^n$). We denote by $\widetilde{\Pi}_{\cY,d_{\cY}}$ as the projection that projects any $w\in \cY$ to the space spanned by the first $d_{\cY}$ principal eigenfunctions of $G_{\zeta}$. Relation (\ref{eq.GkappaN}) implies that 
	$$\EE_{\{v_i\}_{i=1}^n\sim \zeta} \left[\Pi_{\cY,d_{\cY}}^n\right]=\widetilde{\Pi}_{\cY,d_{\cY}}.$$
	
	We have
	$$
	\EE_{v \sim \zeta} \left[\left \|G_{\zeta}^n-\EE_{\{v_i\}_{i=1}^n\sim \zeta} \left[G_{\zeta}^n\right]\right\|_{\rm HS}^2\right] \leq 4\EE_{v\sim \zeta} \left[\|v\|_{\cY}^4\right]\leq 4(R_{\cY}+\widetilde{\sigma})^4.
	$$
	
	We deduce that
	\begin{align}
		&\EE_{\cS}\EE_{w\sim \Psi_{\#}\gamma} \left[ \left\|\Pi_{\cY,d_{\cY}}^n(w)-w\right\|_{\cY}^2 \right]\nonumber\\
		=&\EE_{\cS}\EE_{\widetilde{\epsilon}\sim \mu}\EE_{w\sim \Psi_{\#}\gamma} \left[\left\|\Pi_{\cY,d_{\cY}}^n(w+\widetilde{\epsilon})-(w+\widetilde{\epsilon})-\left[\Pi_{\cY,d_{\cY}}^n(\widetilde{\epsilon})-\widetilde{\epsilon} \right]\right\|_{\cY}^2\right] \nonumber\\
		\leq & 2\EE_{\cS}\EE_{\widetilde{\epsilon}\sim \mu}\EE_{w\sim \Psi_{\#}\gamma} \left[\left\|\Pi_{\cY,d_{\cY}}^n(w+\widetilde{\epsilon})-(w+\widetilde{\epsilon})\right\|_{\cY}^2\right]+ 2 \EE_{\cS}\EE_{\widetilde{\epsilon}\sim \mu} \left[\left\|\left[\Pi_{\cY,d_{\cY}}^n(\widetilde{\epsilon})-\widetilde{\epsilon} \right]\right\|_{\cY}^2 \right]\nonumber\\
		\leq &2\EE_{\cS}\EE_{v\sim \zeta} \left[\left\|\Pi_{\cY,d_{\cY}}^n(v)-v\right\|_{\cY}^2\right]+2\EE_{\widetilde{\epsilon} \sim \mu}\left[\left\|\widetilde{\epsilon}\right\|_{\cY}^2 \right] \nonumber\\
		\leq &  2\sqrt{\frac{4(R_{\cY}+\widetilde{\sigma})^4d_{\cY}}{n}}+2\EE_{v\sim \zeta} \left[\left\|\widetilde{\Pi}_{\cY,d_{\cY}}(v)-v\right\|_{\cY}^2\right] +2\widetilde{\sigma}^2,
		\label{eq.PCA.errorY}
	\end{align}
	where the last inequality comes from Lemma \ref{lem.PCA.empirical} and $\widetilde{\Pi}_{\cY,d_{\cY}}=\EE_{v\sim \zeta} [v\otimes v]= \EE_{\{v_i\}_{i=1}^n\sim \zeta}\left[ \Pi_{\cY,d_{\cY}}^n\right]$.
	
	We bound the second term on the right-hand side as
	\begin{align}
		&\EE_{v \sim \zeta} \left[ \left\|\widetilde{\Pi}_{\cY,d_{\cY}}(v)-v\right\|_{\cY}^2\right] \nonumber\\
		\leq & 2\EE_{v\sim \zeta} \left[ \left\|\widetilde{\Pi}_{\cY,d_{\cY}}(v)-\Pi_{\cY,d_{\cY}}(v)\right\|_{\cY}^2 \right] +2\EE_{v\sim \zeta}\left[ \left\|\Pi_{\cY,d_{\cY}}(v)-v\right\|_{\cY}^2 \right]\nonumber\\
		\leq & 2\EE_{v\sim \zeta} \left[ \left\|\left(\widetilde{\Pi}_{\cY,d_{\cY}}-\Pi_{\cY,d_{\cY}}\right)(v)\right\|_{\cY}^2 \right] +2\EE_{\widetilde{\epsilon}\sim \mu}\EE_{w \sim \Psi_{\#}\gamma} \left[ \left\|\Pi_{\cY,d_{\cY}}(w+\widetilde{\epsilon})-(w+\widetilde{\epsilon})\right\|_{\cY}^2 \right]\nonumber\\
		\leq & 2\EE_{v\sim \zeta} \left[ \left\|\widetilde{\Pi}_{\cY,d_{\cY}}-\Pi_{\cY,d_{\cY}}\right\|_{\rm op}^2\|v\|_{\cY}^2\right] + 4\EE_{w\sim \Psi_{\#}\gamma} \left[ \left\|\Pi_{\cY,d_{\cY}}(w)-w\right\|_{\cY}^2 \right] + 4 \EE_{\widetilde{\epsilon}\sim \mu} \left[ \left\|\Pi_{\cY,d_{\cY}}(\widetilde{\epsilon})-\widetilde{\epsilon}\right\|_{\cY}^2 \right] \nonumber\\
		\leq & 2\EE_{v\sim \zeta} \left[ \left\|\widetilde{\Pi}_{\cY,d_{\cY}}-\Pi_{\cY,d_{\cY}}\right\|_{\rm HS}^2\|v\|_{\cY}^2 \right] + 4\EE_{w\sim \Psi_{\#}\gamma}\left[ \left\|\Pi_{\cY,d_{\cY}}(w)-w\right\|_{\cY}^2 \right] + 4 \EE_{\widetilde{\epsilon}\sim \mu} \left[ \left\|\Pi_{\cY,d_{\cY}}(\widetilde{\epsilon})-\widetilde{\epsilon}\right\|_{\cY}^2 \right] \nonumber\\
		\leq & 2\left(\frac{\sqrt{2}\|G_{\mu}\|_{\rm HS}}{\lambda_{d_{\cY}}-\lambda_{d_{\cY+1}}}\right)^2(R_{\cY}+\widetilde{\sigma})^2 + 4\EE_{w \sim \Psi_{\#}\gamma} \left[ \left\|\Pi_{\cY,d_{\cY}}(w)-w\right\|_{\cY}^2\right] + 4\widetilde{\sigma}^2 \nonumber\\
		\leq & 4\left(\frac{\widetilde{\sigma}}{\lambda_{d_{\cY}}-\lambda_{d_{\cY+1}}}\right)^2\widetilde{\sigma}^2(R_{\cY}+\widetilde{\sigma})^2+ 4\EE_{w \sim \Psi_{\#}\gamma} \left[ \left\|\Pi_{\cY,d_{\cY}}(w)-w\right\|_{\cY}^2 \right] + 4\widetilde{\sigma}^2,
		\label{eq.PCA.Y}
	\end{align}
	where the fourth inequality follows from Lemma \ref{lem.PCAHilbert}.

	Substituting (\ref{eq.PCA.Y}) into (\ref{eq.PCA.errorY}) gives rise to (\ref{eq.pca.PiY}).

\end{proof}

\section{Proof of Lemma \ref{lem.approx.M.dense}}\label{lem.approx.M.dense.proof}
\begin{proof}[Proof of Lemma \ref{lem.approx.M.dense}]
	Our proof relies on concepts related to functions on manifolds, such as charts, atlas, the partition of unity, and functions on manifolds. We refer the readers to \citep{loring2011introduction,lee2006riemannian,chen2019nonparametric,hao2021icml} for details. Following \citep[Proof of Theorem 1]{chen1995universal}, we first construct an atlas of $\cM$ in which all projections projects any point on $\cM$ to a tangent space of $\cM$. These projections are linear functions that can be realized by a subnetwork. Then the function $f$ is decomposed using a partition of unity subordinates to the atlas we constructed. For each chart $(U,\phi)$, we use a subnetwork to approximate an indicator function that determines whether the input $\xb\in\cM$ belongs to $U$. Another subnetwork is used to approximate $f(\xb)\circ \phi^{-1}$ on its tangent space.  Finally, we multiply both subnetworks together and sum over all chats. The multiplication is approximated by another subnetwork.
	We prove Lemma \ref{lem.approx.M.dense} in four steps. 
	\paragraph{Step 1.} In the first step, we show that there exists an atlas of $\cM$, denoted by $\{U_k,\phi_k\}_{k=1}^{C_{\cM}}$, such that $\phi_k$'s are linear projections. Denote $B_r(\cbb)$ as the Euclidean ball in $\RR^{d_{\cX}}$ centered at $\cbb$ with radius $r$. For any given $r>0$, since $\cM$ is compact, there exists a set of points $\{\cbb_k\}_{k=1}^{ C_{\cM}}$ such that $\cM\in \cup_{k} B_r(\cbb_k)$. 	
	For each $B_r(\cbb_k)$, denote $U_i=\cM\cap B_r(\cbb_k)$. By setting $r<\tau/2$, we have that $U_i$ is diffeomorphic to a ball in $\RR^{d_0}$ \citep{niyogi2008finding}. The minimal number of balls is upper bounded by 
	$$
	C_{\cM}\leq \left\lceil {\rm Area}(\cM)T_d/r^d\right\rceil,
	$$
	where ${\rm Area}(\cM)$ is the area of $\cM$ and $T_d$ is the thickness of $U_k$'s (see Chapter 2 of \citep{conway2013sphere}).
	
	We next define $\phi_k$'s. For each $\cbb_k$, let $\{\vb^k_j\}_{j=1}^{d_0}$ be an orthonormal basis of the tangent space of $\cM$ at $\cbb_k$. Define the matrix $V_k=[\vb^k_1,..., \vb^k_d]$. We set 
	$$
	\phi_k(\xb) = V_k^{\top}(\xb-\cbb_k).
	$$ 
	Note that $\phi_k$ is a linear function which can be realized by a single layer. Then $\{(U_k,\phi_k)\}_{k=1}^{C_{\cM}}$ form an atlas of $\cM$.
	
	\paragraph{Step 2.} In the second step, we design a subnetwork that determines the chart that the input $\xb$ belongs to. To determine whether $\xb\in U_k$, it is equivalent to check whether the squared distance between $\xb$ and $\cbb_k$ is less than $r^2$. It can be done by $\mathds{1}_{[0,r^2]}\circ d_k^2(\xb)$ where $\mathds{1}_{[0,r^2]}(a)$ is an indicator function that outputs $1$ if $a\in[0,r^2]$, and outputs 0 otherwise. Here $d_k^2(\xb)$ is the squared distance function defined as
	\begin{align*}
		d_k^2(\xb)=\|\xb-\cbb_k\|_2^2=\sum_{j=1}^{d_{\cX}} (x_j-c_{k,j})^2,
	\end{align*}
	where the notations $\xb=[x_1,...,x_{d_{\cX}}]^{\top}$ and $\cbb_k=[c_{k,1},...,c_{k,d_{\cX}}]$ are used.

	We next approximate both functions by neural networks. To approximate $d_k^2$, the key issue is to approximate the square function by neural networks, for which we use the following lemma:
	\begin{lemma}[Lemma 4.2 of \citep{shijun3}]\label{lem.multiplicaiton}
		For any $B>0$ and integers $L,p>0$, there exists a network $\widetilde{\times}$ in $\cF_{\rm NN}(1,L,9p+1,B^2)$ with $d_{\cY}=1$ such that for any $x,y\in [-B,B]$, we have
		\begin{align}
			|\widetilde{\times}(x,y)-xy|\leq 24B^2p^{-L}.
		\end{align}
	\end{lemma}
	According to Lemma \ref{lem.multiplicaiton}, we approximate $d_k^2(\xb)$ by
	\begin{align*}
		\widetilde{d}_k^2(\xb)=\sum_{j=1}^{d_{\cX}} \widetilde{\times}(x_j-c_{k,j},x_j-c_{k,j}),
	\end{align*}
	where $\widetilde{\times}\in \cF_{\rm NN}(1,4sL_1,9p_1+1,B^2)$. The approximation error is $\|\widetilde{d}_k-d_k\|_{\infty}\leq 24d_{\cX}B^2p_1^{-4sL_1}$.
	
	For $\mathds{1}_{[0,r]^2}$, we use the following function to approximate it
	\begin{align*}
		\widetilde{\mathds{1}}_{\Delta}(a)=\begin{cases}
			1 & a\leq r^2-\Delta+24d_{\cX}B^2p_1^{-4sL_1},\\
			-\frac{1}{\Delta-48d_{\cX}B^2p_1^{-4sL_1}}a+\frac{r^2-24d_{\cX}B^2p_1^{-4sL_1}}{\Delta-48d_{\cX}B^2p_1^{-4sL_1}} & a\in \left[ r^2-\Delta+24d_{\cX}B^2p_1^{-4sL_1}, r^2-24d_{\cX}B^2p_1^{-4sL_1}\right],\\
			0 & a\geq  r^2-24d_{\cX}B^2p_1^{-4sL_1},
		\end{cases}
	\end{align*}
	where $\Delta\geq  24d_{\cX}B^2p_1^{-4sL_1}$ will be chosen later. We approximate $\widetilde{\mathds{1}}_{\Delta}\circ d_k^2(\xb)$ by $\widetilde{\mathds{1}}_{\Delta}\circ\widetilde{d}_k^2(\xb) $ in which the parameter $\Delta$ is the 'width' of the error region: when $\xb\notin U_k$, we have $d_k^2(\xb)\geq r^2$ and $\widetilde{\mathds{1}}_{\Delta}\circ\widetilde{d}_k^2(\xb)=0$; when $\xb\in U_k$ and $d_k^2(\xb)\leq r^2-\Delta$, we have $\widetilde{\mathds{1}}_{\Delta}\circ\widetilde{d}_k^2(\xb)=1$.
	
	We then realize $\widetilde{\mathds{1}}_{\Delta}(a)$ by a subnetwork. Denoting $m_0=\frac{1}{\Delta-48d_{\cX}B^2p_1^{-4sL_1}}, m_1= r^2-\Delta+24d_{\cX}B^2p_1^{-4sL_1}, m_2=r^2-24d_{\cX}B^2p_1^{-4sL_1}$, we rewrite $\widetilde{\mathds{1}}_{\Delta}(a)$ as
	\begin{align*}
		\widetilde{\mathds{1}}_{\Delta}(a)=-m_0(\min\{\max\{a,m_1\},m_2\})+m_2m_0.
	\end{align*}
	The function above can be realized by a network with one hidden layer:
	\begin{align*}
		\widetilde{\mathds{1}}_{\Delta}(a)=-m_0\left( m_2- \ReLU\left[m_2-\left( \ReLU(a-m_1)+m_1\right) \right]\right)+m_2m_0.
	\end{align*}

	\paragraph{Step 3.} In this step, we decompose $f$ using a partition of unity of $\cM$ and approximate each component by a subnetwork. Let $\{h_k\}_{k=1}^{C_{\cM}}$ be a partition of unity of $\cM$ such that $h_k$ is supported on $U_k$. We decompose $f$ as
	\begin{align*}
		f=\sum_{k=1}^{C_{\cM}} h_kf.
	\end{align*}
	Note that for each $k$, $h_kf$ is a function defined on $\cM$ supported on $U_k$, and $(h_kf)\circ \phi_k^{-1}$ is a function defined in $\RR^{d_0}$ and supported on $[-2B,2B]^{d_0}$. The following lemma shows that $h_kf$ is in the same space as $f$:
	\begin{lemma}\label{lem.fk}
		Suppose Assumption \ref{assum.M} holds. Let $\{U_k,\phi_k\}_{k=1}^{C_{\cM}}$ be defined in {\bf Step 1}. For each $k$, we have $h_kf\in \cC^{0,1}(\cM)$ and $\|h_kf\|_{\cC^{0,1}(\cM)}$ is bounded by a constant depending on $d_0,h_k,f$ and $\phi_k$.
	\end{lemma}
	Lemma \ref{lem.fk} can be proved by following the proof of \citep[Lemma 2]{chen2019nonparametric}. The proof is omitted here.
	According to Lemma \ref{lem.fk} and since $\phi_k$ is a linear projection, we have  $(h_kf)\circ \phi_k^{-1}\in \cC^{0,1}([-2B,2B]^{d_0})$. Lemma \ref{lem.approx.dense} implies that there exists a neural network $\widetilde{f}_k\in \cF_{\rm NN}(1,L_2,p_2,M)$ with
	\begin{align*}
		L_2=O(\widetilde{L}_2\log \widetilde{L}_2), \ p_2=O\left(\widetilde{p}_2\log \widetilde{p}_2\right),\ M=R.
	\end{align*}
	for any $\widetilde{L}_2,\widetilde{p}_2>0$
	such that 
	$$
	\|\widetilde{f}_k-(h_kf)\circ \phi_k^{-1}\|_{\infty}\leq C_1\widetilde{L}_2^{-\frac{2}{d_0}}\widetilde{p}_2^{-\frac{2}{d_0}}
	$$
	for some constant $C_1$ depending on $d_0,B,R$. 
	
	\paragraph{Step 4.} We then assemble all subnetworks constructed in the previous steps and approximate $f$ by
	\begin{align}
		\widetilde{f}=\sum_{k=1}^{C_{\cM}} \widetilde{\times}\left( \widetilde{f}_k\circ \phi_k, \widetilde{\mathds{1}}_k\circ \widetilde{d}_k^2 \right).
		\label{eq.fsum.dens}
	\end{align}
	In (\ref{eq.fsum.dens}), according to Lemma \ref{lem.multiplicaiton}, we set $\widetilde{\times}\in \cF_{\rm NN}(1,4L_3,9p_3+1,M)$ as an approximation of $\times$ with $M=R$ and error $24R^2p_3^{-4L_3}$.
	The following lemma gives an upper bound of the approximation error of $\widetilde{f}$ (see a proof in Appendix \ref{lem.M.f.dense.erro.proof}):
	\begin{lemma}\label{lem.M.f.dense.erro}
		The error of $\widetilde{f}$ can be decomposed as
		\begin{align*}
			\|\widetilde{f}-f\|_{\infty}\leq \sum_{k=1}^{C_{\cM}} A_{k,1}+A_{k,2}+A_{k,3}
		\end{align*}
		with
		\begin{align*}
			&A_{k,1}=\left\|\widetilde{\times}(\widetilde{f}\circ \phi_k^{-1},\widetilde{\mathds{1}}_{\Delta}\circ \widetilde{d}_k^2)-(\widetilde{f}\circ \phi_k^{-1})\times (\widetilde{\mathds{1}}_{\Delta}\circ \widetilde{d}_k^2)\right\|_{\infty}\leq 24R^2p_3^{-2}L_3^{-2},\\
			&A_{k,2}=\left\|(\widetilde{f}\circ \phi_k^{-1})\times (\widetilde{\mathds{1}}_{\Delta}\circ \widetilde{d}_k^2)-[(h_kf)\circ \phi_k^{-1}]\times (\widetilde{\mathds{1}}_{\Delta}\circ \widetilde{d}_k^2)\right\|_{\infty}\leq C_{12}\widetilde{L}_2^{-\frac{2}{d_0}}\widetilde{p}_2^{-\frac{2}{d_0}},\\
			&A_{k,3}=\left\|[(h_kf)\circ \phi_k^{-1}]\times (\widetilde{\mathds{1}}_{\Delta}\circ \widetilde{d}_k^2)-[(h_kf)\circ \phi_k^{-1}]\times \mathds{1}_{[0,r^2]}\right\|_{\infty}\leq \frac{C_{13}(\pi+1)}{r(1-r/\tau)}\Delta,
		\end{align*}
		for some constant $C_{12}$ depending on $d_0,\tau,B,R$, and $C_{13}$ depending on $R$.
	\end{lemma}

	According to Lemma \ref{lem.M.f.dense.erro}, for any $\widetilde{L},\widetilde{p}>0$, we set 
	\begin{itemize}
		\item $\widetilde{f}_k\in \cF_{\rm NN}(1,L_2,p_2,M)$ with $L_2=O\left(\widetilde{L}\log \widetilde{L}\right), p_2=O\left(\widetilde{p}\log \widetilde{p}\right)$.
		\item $\widetilde{\times}\in \cF_{\rm NN}(1,4L_3,9p_3+1,M)$ with $L_3=O\left(\widetilde{L}\right), p_3=O\left( \widetilde{p}\right)$,
		\item $\widetilde{d}_k^2\in \cF_{\rm NN}(1,4L_1,d_{\cX}(9p_1+1),M)$ with $\Delta =\widetilde{L}^{-\frac{2}{d_0}}\widetilde{p}^{-\frac{2}{d_0}}$, $L_1=\widetilde{L}+\log (12d_{\cX}B^2),p_1=\widetilde{p}$ such that
		\begin{align}
			&24d_{\cX}B^2p_1^{-4L_1}=24d_{\cX}B^2\widetilde{p}^{-4\widetilde{L}-\log (48d_{\cX}B^2)}= 24d_{\cX}B^2\widetilde{p}^{-\log (48d_{\cX}B^2)}\widetilde{p}^{-4\widetilde{L}} \nonumber\\
			&= 24d_{\cX}B^2(48d_{\cX}B^2)^{-\log \widetilde{p}}\widetilde{p}^{-4\widetilde{L}}\leq  24d_{\cX}B^2(24d_{\cX}B^2)^{-1}\widetilde{p}^{-4\widetilde{L}}\nonumber\\
			&\leq \widetilde{p}^{-4\widetilde{L}}\leq \widetilde{p}^{-2(\widetilde{L}+1)}\leq \widetilde{p}^{-2}2^{-2\widetilde{L}}\leq \widetilde{p}^{-2}\widetilde{L}^{-2}<\Delta,
		\end{align}
		
		\item $\widetilde{\mathds{1}}_{\Delta}\in \cF_{\rm NN}(1,2,1,1)$.
	\end{itemize}
	The total approximation error is bounded by $C_3\widetilde{L}^{-\frac{2}{d_0}}\widetilde{p}^{-\frac{2}{d_0}}$ for some $C_3$ depending on $d_0,R,B, \tau$ and the surface area of $\cM$. The constant hidden in $O(\cdot)$ depends on $d_0,R,B, \tau$ and the surface area of $\cM$. The resulting network is in $\cF_{\rm NN}(1,L,p,M)$ with $L,p,M$ defined in (\ref{eq.NN.parameter.n.M.dense.lemma}).

\end{proof}

\section{Proof of Lemma \ref{lem.M.f.dense.erro}}\label{lem.M.f.dense.erro.proof}
\begin{proof}
	For $A_{k,1}$, since $\widetilde{\times}\in \cF_{\rm NN}(1,4L_3,p_3,R)$, by Lemma \ref{lem.multiplicaiton}, we have
	\begin{align*}
		A_{k,1}\leq 24R^2p_3^{-4L_3}\leq 24R^2p_3^{-2\left(L_3+1\right)}\leq 24R^2p_3^{-2}2^{-2L_3}\leq 24R^2p_3^{-2}L_3^{-2}.
	\end{align*}
	For $A_{k,2}$, since $\widetilde{\mathds{1}}_k\circ \widetilde{d}_k^2\in [0,1]$, we have
	\begin{align*}
		A_{k,2}\leq \left\|\widetilde{f}\circ \phi_k^{-1}-(h_kf)\circ \phi_k^{-1}\right\|_{\infty}\leq \left\|\widetilde{f}-(h_kf)\right\|_{\infty} \leq C_{12}\widetilde{L}_2^{-\frac{2}{d_0}}\widetilde{p}_2^{-\frac{2}{d_0}}.
	\end{align*}
	The upper bound of $A_{k,3}$ is proved in \citep[Proof of Lemma 3]{chen2019nonparametric}.
\end{proof}
\end{document}